\documentclass[hidelinks,onefignum,onetabnum]{siamart250211}

\usepackage{amsmath}
\usepackage{mathrsfs}
\usepackage{enumerate}
\usepackage{tabularray}
\usepackage{tabularx} 
\usepackage{arydshln}
\usepackage{psfrag}
\usepackage{subfigure}
\usepackage{multirow}
\usepackage{verbatim}
\usepackage{lipsum}
\usepackage{amsfonts}
\usepackage{graphicx}
\usepackage{epstopdf}
\usepackage{algorithmic}
\usepackage{array}
\usepackage{booktabs}

\ifpdf
  \DeclareGraphicsExtensions{.eps,.pdf,.png,.jpg}
\else
  \DeclareGraphicsExtensions{.eps}
\fi

\newcommand{\bgamma}{\boldsymbol{\gamma}}
\newcommand{\bGamma}{\boldsymbol{\Gamma}}
\newcommand{\mh}{\mathcal{H}}
\newtheorem{example}{Example}[section]
\newcommand{\ml}{\mathcal{L}}
\newcommand{\mj}{\mathcal{J}}

\newsiamremark{remark}{Remark}
\newsiamremark{hypothesis}{Hypothesis}
\crefname{hypothesis}{Hypothesis}{Hypotheses}
\newsiamthm{claim}{Claim}
\newsiamremark{fact}{Fact}
\crefname{fact}{Fact}{Facts}

\headers{Sparsity-Promoting Hyperpriors in Empricial Bayes}{Z. Li, Y. Dong, and X. Zeng}

\title{Sparsity via Hyperpriors: A Theoretical and Algorithmic Study under Empirical Bayes Framework\thanks{Submitted to the editors November 6, 2025.\funding{ZL and XZ were supported by the National Key R\&D Plan of China (No. 2024YFC2814403), and the Fundamental Research Funds for the Central Universities (No. 202264006). YD was supported by a Villum Investigator grant (No. 25893) from The Villum Foundation.}}}

\author{Zhitao Li \thanks{School of Mathematical Sciences, Ocean University of China, Qingdao, China (\email{lizhitao3378@stu.ouc.edu.cn})} \and Yiqiu Dong \thanks{Department of Applied Mathematics and Computer Science, Technical University of Denmark, Kgs. Lyngby, Denmark (\email{yido@dtu.uk})} \and Xueying Zeng \thanks{Corresponding author. Laboratory of Marine Mathematics, Ocean University of China, Qingdao, China (\email{zxying@ouc.edu.cn})}}

\usepackage{amsopn}

\begin{document}

\maketitle

% REQUIRED
\begin{abstract}
This paper presents a comprehensive analysis of hyperparameter estimation within the empirical Bayes framework (EBF) for sparse learning. By studying the influence of hyperpriors on the solution of EBF, we establish a theoretical connection between the choice of the hyperprior and the sparsity as well as the local optimality of the resulting solutions. We show that some strictly increasing hyperpriors, such as half-Laplace and half-generalized Gaussian with the power in $(0,1)$, effectively promote sparsity and improve solution stability with respect to measurement noise. Based on this analysis, we adopt a proximal alternating linearized minimization (PALM) algorithm with convergence guaranties for both convex and concave hyperpriors. Extensive numerical tests on two-dimensional image deblurring problems demonstrate that introducing appropriate hyperpriors significantly promotes the sparsity of the solution and enhances restoration accuracy. Furthermore, we illustrate the influence of the noise level and the ill-posedness of inverse problems to EBF solutions.
\end{abstract}

% REQUIRED

% REQUIRED

\section{Introduction}
Inferring signals or parameters from limited and noisy measurements is a central challenge in inverse problems \cite{PCHIP}, which commonly appear in many applications, e.g. medical imaging, geophysics, and engineering. The unknown often admits a sparse representation under an appropriate transform, motivating the search for its sparse representation over a possibly overcomplete dictionary or basis. The prior information on the sparsity can be incorporated through regularization technique. Some commonly used regularizers include the $l_0$-norm \cite{natarajan1995sparse}, the $l_1$-norm that is a convex relaxation of $l_0$ \cite{chen2001atomic,tibshirani1996regression}, and the $l_p$ quasi-norm with $0<p<1$ \cite{grasmair2008sparse,wang2011performance,xie2013rewighted,lorenz2016flexible,ghilli2022nonconvex}, etc. In these methods, the balance between the fitting of the data and regularization is controlled by a regularization parameter, and the selection of this parameter has a significant impact on the quality of the reconstruction results. Choosing an appropriate regularization parameter is always challenging. 

Most of the above mentioned approaches can be interpreted within a Bayesian framework as maximum a posteriori (MAP) estimates with a sparsity-inducing prior. Hierarchical Bayesian models extend this idea by introducing additional layers of priors over hyperparameters, allowing the model to adaptively learn regularization strength and sparsity patterns from data  \cite{calvetti2008hypermodels}. In sparse dictionary learning, Gaussian-inverse Gamma hierarchical priors have been used to jointly infer sparse codes and noise variance \cite{yang2017sparse}. Generalized Gamma hyperpriors further enhance the sparsity and stability of solutions with respect to measurement noise, providing stronger control over the reconstructed solution \cite{calvetti2019hierachical,calvetti2020sparse,calvetti2020sparsity}. Recently, horseshoe shrinkage prior, which encodes a half-Cauchy hyperprior in the conditional Gaussian prior, has been shown to be effective in promoting sparsity \cite{HS1,HS2}.  
In those methods, the unknown and the hyperparameters are usually solved in an alternate iteration scheme. 

In contrast, sparse Bayesian learning (SBL) offers an alternative approach that estimates hyperparameters directly form data without requiring explicit hierarchical hyperpriors. 
SBL originated from the automatic relevance determination (ARD) framework \cite{mackay1992bayesian,tipping2001sparse}, which is a machine learning technique for regression and classification.  
Under SBL, the hyperparameters, which parameterize the variance of the Gaussian prior, are estimated through evidence maximization, also known as Type-II maximum likelihood \cite{bishop2006pattern}. By marginalizing the unknown, the hyperparameters and the data are linked directly. Furthermore, the marginalization process provides a regularization mechanism, which can shrink many hyperparameters to zero and further induce sparsity in the unknown.   
In \cite{tipping2001sparse} Tipping established the theoretical foundation of SBL,
while subsequent analysis explored the geometry of the marginal likelihood and revealed both its sparsity-promoting nature and its nonconvexity \cite{faul2001analysis}. In \cite{wipf2004sparse,wipf2007new} Wipf et al. analyzed the structure of SBL and showed that it is equivalent to an optimization problem with a nonconvex sparsity-inducing penalty, clarifying the underlying mechanism of sparsity formation. Later, they extended this framework to latent-variable Bayesian models, which further enhance the flexibility of sparse priors through hierarchical latent representations \cite{wipf2011latent}. In \cite{yu2024hyperparameter} Yu et al. proposed a unified framework for hyperparameter estimation in SBL, which offers a new perspective to interpret, analyze and compare the algorithms used for hyperparameter estimation, both theoretically and numerically. Building on these insights, several extensions generalizing SBL have been developed to handle structured sparsity and large-scale imaging problems \cite{al2019weighted,glaubitz2023generalized}.

Despite its success across inverse and regression tasks, SBL remains challenged by the nonconvexity of the corresponding optimization problem, the sensitivity to measurement noise and the computational burden of large matrix operations. It motivates continued research on more stable and interpretable empirical Bayes framework (EBF) \cite{mackay1992evidence}. In EBF instead of maximizing the marginal likelihood, it maximizes the marginal posterior by incorporating the hyperpriors. Recent advances extend EBF to several important directions. The work in \cite{tipping2001sparse,Gu2013} discussed the properties of the EBF solution in the case of an inverse Gamma hyperprior. A hyperprior combining a concave function and a convex function was proposed in \cite{rudy2021sparse}. It was also pointed out that, except for certain cases, the generalized Gamma distribution does not impose sparsity, which is the reason why most existing EBF only rely on noninformative or weakly informative hyperpriors. In \cite{wipf2007empirical} Wipf et al. extended EBF to the simultaneous sparse approximation problem, providing a method for jointly reconstructing multiple signals with shared sparsity patterns. 
In addition, EBF has been applied in many applications or as a building block in other approaches, e.g. a unified Bayesian formulation covering EBF tailored for biomedical inverse problems such as MEG/EEG source imaging \cite{wipf2009unified}, adaptive schemes for deep models that integrate layer-wise sparsity \cite{deng2019adaptive}, meta-prior learning that uses EBF to learn higher-level prior distributions \cite{nabi2022bayesian}, etc. Despite empirical successes, EBF inherits challenges of nonconvex marginal-likelihood optimization and sensitivity to hyperprior families, motivating research into robust hyperprior design and sparsity-promoting modeling strategies.

Among the work on EBF, it still lacks a solid study on how hyperpriors promote sparsity meanwhile stablizing the problem. 
In this work, we provide a comprehensive analysis of the hyperparameter estimation under EBF. By analyzing the properties of stationary points in this framework, we establish a theoretical connection between the structure of the hyperprior and both the sparsity and local optimality of the resulting solutions. We show that hyperpriors under certain conditions can effectively promote sparsity and enhance solution stability with respect to measurement noise. Based on these theoretical insights, we focus on the generalized Gamma family as hyperpriors and study their influence on sparsity under different settings. In particular, the hyperpriors following such as the half-Laplace distribution and the half-generalized Gaussian distribution with the power in $(0,1)$ can effectively promote the sparsity of hyperparameters and improve the stability of the solution to EBF. In addition, we introduce the proximal alternating linearized minimization (PALM) algorithm to solve our problem and provide a rigorous convergence analysis. The convergence result shows that regardless of whether the hyperprior is convex or concave, the convergence of PALM to a stationary point is guaranteed. Finally, we use two-dimensional image deblurring probelms to illustrate the effects of different hyperpriors on the restoration quality and the influence of the ill-posedness and the noise level on the EBF results. The results confirm our theoretical findings and show that choosing an appropriate hyperprior, such as half-Laplace or half-generalized Gaussian with power in $(0,1)$, can significantly promote sparsity and improve restoration accuracy.

The remainder of this paper is structured as follows. Section 2 reviews EBF briefly and formulates the hyperparameter estimation problem. Section 3 establishes the theoretical underpinnings of sparsity promotion and solution optimality in EBF. Section 4 presents the proposed numerical algorithm together with its convergence analysis. In Section 5 we test our method on the 2D image deblurring problems. Finally, we conclude our work with some
perspectives for future research in Section 6.

\section{Empirical Bayes framework}
We consider the following discrete inverse problem: 
\begin{equation}
	\label{ip}
	\boldsymbol{y}=\mathbf{F}\boldsymbol{x}+\boldsymbol{\epsilon},
\end{equation}
where $\boldsymbol{y}\in \mathbb{R}^{m}$ denotes the vectorized observed data, $\boldsymbol{x}\in \mathbb{R}^{n}$ is the unknown that we are interested in recovering, $\mathbf{F}\in \mathbb{R}^{m\times n}$ is the forward operator modeling the physical relation between the parameters and the data, and $\boldsymbol{\epsilon}\in \mathbb{R}^{m}$ is the vector of noises. We assume $m\leq n$ and the noise vector $\boldsymbol{\epsilon}$ follows an independent and identically distributed (i.i.d.) normal distribution with zero mean and inverse variance $\sigma$, i.e.,
$$\boldsymbol{\epsilon}\sim\mathcal{N}(0,\sigma^{-1}\mathbf{I}_{m}),$$ 
where $\mathbf{I}_{m}$ is an identity matrix of size $m$, and $\mathcal{N}$ denotes the multivariate Gaussian distribution. Then the likelihood is in the following form:
\begin{equation}
	\pi(\boldsymbol{y}|\boldsymbol{x})=\mathcal{N}(\mathbf{F}\boldsymbol{x},\sigma^{-1}\mathbf{I}_{m}).
\end{equation}
Furthermore, we assume that $\boldsymbol{x}$ follows a Gaussian distribution due to its computational simplicity. However, Gaussian priors are not well suited to promote sparsity. One possible way to retain computational convenience while promoting sparsity is via a conditional Gaussian prior in the form of 
\begin{equation}
	\pi(\boldsymbol{x}|\boldsymbol{\gamma})=\mathcal{N}(0,\bGamma)\propto\prod_i\frac1{\sqrt{\gamma_i}}\exp\left(-\frac{x_i^2}{2\gamma_i}\right).
\end{equation}
where $\boldsymbol{\gamma}\in\mathbb{R}^{n}_{+}:=\left\lbrace \bgamma\in\mathbb{R}^{n}:\gamma_i\geq 0, 1\leq i\leq n\right\rbrace $ collects all the unknown variance $\gamma_i$, which is expected to be small when $x_i$ vanishes, while it should be larger in correspondence of the few nonzero-entries of $\boldsymbol{x}$. In the extreme scenario where $\gamma_i = 0$, $x_i$ essentially reduces to a degenerate distribution, represented as a point mass at the mean, i.e. $0$. In addition, $\bGamma=\text{diag}(\bgamma)$ is the diagonal matrix with diagonal entries given by the hyperparameters $\bgamma$.

In the Bayesian paradigm, the hyperparameters $\bgamma$ are modeled as random variables, and their estimation becomes part of the problem. A common approach to hyperparameter estimations is to use sparse Bayesian learning \cite{faul2001analysis}, also known as Type II maximum likelihood, which computes the marginal likelihood by integrating out the unknown and then maximizes this marginal likelihood with respect to the hyperparameters. 
Since both the likelihood $\pi(\boldsymbol{y}|\boldsymbol{x})$ and the prior $\pi(\boldsymbol{x}|\boldsymbol{\gamma})$ are Gaussian, it follows from the conjugate property of the Gaussian distribution that the marginal likelihood is given by
\begin{equation}
	\pi(\boldsymbol{y}|\boldsymbol{\gamma})=\int \pi(\boldsymbol{y}|\boldsymbol{x})\pi(\boldsymbol{x}|\boldsymbol{\gamma})d\boldsymbol{x}=\mathcal{N}(\boldsymbol{y}|\boldsymbol{0},\mathsf{S}(\boldsymbol{\gamma}))
\end{equation}
with $\mathsf{S}(\boldsymbol{\gamma})=\sigma^{-1}\mathbf{I}_{m}+\mathbf{F}\bGamma\mathbf{F}^\top$, i.e., 
\begin{equation}
	\label{ymidgamma}
	\pi(\boldsymbol{y}|\bgamma)\propto\left| \mathsf{S}(\bgamma) \right|^{-\frac12}\exp\left(-\frac12 \boldsymbol{y}^\top(\mathsf{S}(\bgamma))^{-1}\boldsymbol{y}\right).
\end{equation}

In this paper, we consider the case of limited and noisy data, i.e., $m\leq n$ and $\boldsymbol{\epsilon}\neq\boldsymbol{0}$. Due to the existence of the noise and the under-determined problem, instead of marginal likelihood, we suggest using {\it empirical Bayes framework} \cite{wipf2009unified} to maximize the marginal posterior given as
\[
\pi(\boldsymbol{\gamma}|\boldsymbol{y})\propto\pi(\boldsymbol{y}|\boldsymbol{\gamma})\pi(\boldsymbol{\gamma})
\]
according to Bayes' theorem. Here, $\pi(\boldsymbol{\gamma})$ represents the hyperprior and incorprates the prior information of $\bgamma$. In order to promote sparsity and keep the model simple, we expect that the elements in $\bgamma$ are i.i.d. and that all but few should be close to zero. Thus, the hyperprior is in a product form:
\begin{equation}
	\label{general prior}
	\pi(\bgamma)\propto \prod_{i=1}^n\exp\left(-\mh(\gamma_i)\right),
\end{equation}
where $\mh(\cdot): \mathbb{R}_+\to{\mathbb{R}_+}$ is a continuous function.
Based on (\ref{ymidgamma}) and (\ref{general prior}), the marginal posterior of $\bgamma$ can be expressed as
\begin{equation}
	\label{gamma_post}
	\pi(\boldsymbol{\gamma}|\boldsymbol{y}) \propto\left| \mathsf{S}(\boldsymbol{\gamma}) \right|^{-\frac12}\exp\left(-\frac12 \boldsymbol{y}^\top(\mathsf{S}(\boldsymbol{\gamma}))^{-1}\boldsymbol{y}-\sum_{i=1}^n\mh(\gamma_i)\right).
\end{equation}

A commonly used point estimate for \eqref{gamma_post} is the MAP estimate, where the mode of the posterior is set as the single point representative of the whole density function:
\begin{equation}\label{minJ}
	\boldsymbol{\gamma}^{*}=\arg\max_{{\boldsymbol{\gamma}}\in\mathbb{R}^{n}_{+}}\pi(\boldsymbol{\gamma}|\boldsymbol{y})=\arg\min_{{\boldsymbol{\gamma}}\in\mathbb{R}^{n}_{+}}\mathcal{J}(\boldsymbol{\gamma}),
\end{equation}
where the cost functional $\mathcal{J}$ is defined as
\begin{equation}
	\label{J}
	\mathcal{J}(\boldsymbol{\gamma})=\frac12\boldsymbol{y}^\top(\mathsf{S}(\boldsymbol{\gamma}))^{-1}\boldsymbol{y}+\frac12\log\det\mathsf{S}(\boldsymbol{\gamma})+\sum_{i=1}^n\mh(\gamma_i).
\end{equation}

Once we obtain the estimate $\bgamma^*$ of the hyperparameters, we can derive the conditional posterior distribution $\pi(\boldsymbol{x}|\boldsymbol{y}, \bgamma^*)$ of the unknown $\boldsymbol{x}$ through
\begin{equation}\label{eq:post_x}
	\pi(\boldsymbol{x}|\boldsymbol{y}, \bgamma^*)\propto \pi(\boldsymbol{y}|\boldsymbol{x})\pi(\boldsymbol{x}|\bgamma^*)=\mathcal{N}(\boldsymbol{x}|\mu(\bgamma^*),\Sigma(\bgamma^*)),
\end{equation}
whose mean and covariance are $$\mu(\bgamma^*)=\bGamma^*\mathbf{F}^\top(\mathsf{S}(\bgamma^*))^{-1}\boldsymbol{y} \quad\text{and}\quad \Sigma(\bgamma^*)=\bGamma^*-\bGamma^*\mathbf{F}^\top(\mathsf{S}(\bgamma^*))^{-1}\mathbf{F}\bGamma^*.$$
\section{Sparsity and local optimality}
In this section, we study the influence of the hyperprior on the sparsity and local optimality of the Karush–Kuhn–Tucker (KKT) points of the proposed optimization problem \eqref{minJ}.

The Lagrangian function of \eqref{minJ} is given by
\begin{equation}\label{eq:Lagrangian}
	\mathcal{L}(\bgamma,\boldsymbol{\mu})=\mathcal{J}(\bgamma)-\boldsymbol{\mu}^\top\bgamma,
\end{equation}
with the KKT conditions:
\begin{equation}\label{eq:KKT}
	\begin{cases}
		\nabla \mathcal{J}(\bgamma)-\boldsymbol{\mu}=\boldsymbol{0},\\
		\boldsymbol{\mu}\geq \boldsymbol{0},\bgamma\geq \boldsymbol{0},\\
		\boldsymbol{\mu}^\top\bgamma=0,
	\end{cases}
\end{equation}
where $\boldsymbol{\mu}\in\mathbb{R}^n_+$ is the vector of Lagrange multipliers.

\subsection{Derivatives of $\mathcal{J}(\bgamma)$}

Let $\mathbf{F} = [\mathbf{f}_1,\mathbf{f}_2,\dots,\mathbf{f}_n]$. Then 
\begin{equation}\label{eq:partialS}
	\frac{\partial \mathsf{S}(\boldsymbol{\gamma})}{\partial\gamma_i}=\frac{\partial }{\partial\gamma_i}\left(\sigma^{-1}\mathbf{I}_{m}+\sum^n_{i=1}\gamma_i\mathbf{f}_i\mathbf{f}_i^\top\right)=\mathbf{f}_i\mathbf{f}_i^\top.
\end{equation}

The derivatives of the inverse and log determinant of a matrix $\boldsymbol{A}$ with respect to $t$ are given in \cite{magnus2019matrix} 
\begin{equation}
	\frac{\partial\boldsymbol{A}^{-1}}{\partial t}=-\boldsymbol{A}^{-1}\frac{\partial\boldsymbol{A}}{\partial t}\boldsymbol{A}^{-1},~~\frac{\partial\log\det\boldsymbol{A}}{\partial t}=\mathrm{Tr}\left(\boldsymbol{A}^{-1}\frac{\partial\boldsymbol{A}}{\partial t}\right),
\end{equation}
where $\mathrm{Tr}(\boldsymbol{B})$ denotes the trace of the matrix $\boldsymbol{B}$. Then, we can obtain
\begin{equation}\label{eq:derivativetemp1}
	\frac{\partial \boldsymbol{y}^\top\mathsf{S}(\boldsymbol{\gamma})^{-1}\boldsymbol{y}}{\partial\gamma_i}=-\boldsymbol{y}^\top\mathsf{S}(\boldsymbol{\gamma})^{-1} \mathbf{f}_i\mathbf{f}_i^\top \mathsf{S}(\boldsymbol{\gamma})^{-1}\boldsymbol{y},
\end{equation}
and
$$\frac{\partial \log\det\mathsf{S}(\boldsymbol{\gamma})}{\partial \gamma_i}=\mathrm{Tr}\left(\mathsf{S}(\boldsymbol{\gamma})^{-1}\mathbf{f}_i\mathbf{f}_i^\top \right)=\mathbf{f}_i^\top\mathsf{S}(\boldsymbol{\gamma})^{-1}\mathbf{f}_i.$$
Let $\tilde{p}_i=\mathbf{f}_i^\top\mathsf{S}(\bgamma)^{-1}\boldsymbol{y}$ and $\tilde{q}_{ij}=\mathbf{f}_i^\top\mathsf{S}(\bgamma)^{-1}\mathbf{f}_j$.  The partial derivative of $\mathcal{J}(\bgamma)$ with respect to $\gamma_i$ is
\begin{equation}\label{eq:Jderivative-1}
	\frac{\partial\mathcal{J}(\bgamma)}{\partial \gamma_i}=\frac12\tilde{q}_{ii}-\frac12\tilde{p}_i^2+\mathcal{H}'(\gamma_i).
\end{equation}

Following the similar procedure in \eqref{eq:derivativetemp1}, we have
\begin{equation*}
	\frac{\partial \tilde{q}_{ii}}{\partial\gamma_j}=-\left(\mathbf{f}_i^\top\mathsf{S}(\bgamma)^{-1}\mathbf{f}_j\right)^2=-\tilde{q}_{ij}^2,
\end{equation*}
and
\begin{equation*}
	\frac{\partial \tilde{p}_{i}}{\partial\gamma_j}=-\left(\mathbf{f}_i^\top\mathsf{S}(\bgamma)^{-1}  \mathbf{f}_j\right)\left(\mathbf{f}_j^\top\mathsf{S}(\bgamma)^{-1}  \boldsymbol{y}\right)=-\tilde{q}_{ij}\tilde{p}_{j}.
\end{equation*}
Hence, 
\begin{equation}\label{eq:Jderivative-2order}
	\frac{\partial^2\mathcal{J}(\bgamma)}{\partial \gamma_i\partial\gamma_j}=-\frac12\tilde{q}_{ij}^2+\tilde{p}_{i}\tilde{q}_{ij}\tilde{p}_{j}+\delta_{ij}\mathcal{H}''(\gamma_i),
\end{equation}
where $\delta_{ij}$ is the Kronecker delta function allowing us to separate out the additional (diagnal) term that appears only when $i=j$.

\subsection{Sparsity and local optimality of KKT points} We now study the sparsity of the KKT points of \eqref{minJ}. Since the KKT conditions \eqref{eq:KKT} are separable componentwise, we focus on the $i$-th component.  To this aim, we first rewrite the covariance matrix 
\begin{equation}
	\mathsf{S}(\bgamma)=\mathsf{S}_{-i}+\gamma_i\mathbf{f}_i\mathbf{f}_i^\top
\end{equation}
to isolate its dependence on a single hyperparameter $\gamma_i$, where $\mathsf{S}_{-i}=\sigma^{-1}\mathbf{I}_{m}+\sum_{j\neq i}\gamma_j\mathbf{f}_j\mathbf{f}_j^\top$ is the covariance matrix excluding the contribution of $\mathbf{f}_i$.

Using the Sherman-Morrison formula \cite{hogben2006handbook}, we obtain
\begin{equation}
	\label{invS}
	\mathsf{S}(\bgamma)^{-1}=\mathsf{S}_{-i}^{-1}-\frac{\gamma_i\mathsf{S}_{-i}^{-1}\mathbf{f}_i\mathbf{f}_i^\top\mathsf{S}_{-i}^{-1}}{1+\gamma_i\mathbf{f}_i^\top\mathsf{S}_{-i}^{-1}\mathbf{f}_i},
\end{equation}and
\begin{equation}
	\label{detS}
	\det\mathsf{S}(\bgamma)=(1+\gamma_i\mathbf{f}_i^\top\mathsf{S}_{-i}^{-1}\mathbf{f}_i)\det\mathsf{S}_{-i}.
\end{equation}

Substituting (\ref{invS}) and (\ref{detS}) into the cost function (\ref{J}) yields
\begin{equation}
	\mathcal{J}(\bgamma)=\mathcal{J}_{-i}(\bgamma_{-i})-\frac12\frac{\gamma_i(\mathbf{f}_i^\top\mathsf{S}_{-i}^{-1}\boldsymbol{y})^2}{1+\gamma_i\mathbf{f}_i\mathsf{S}_{-i}^{-1}\mathbf{f}_i^\top}+\frac12\log(1+\gamma_i\mathbf{f}_i^\top\mathsf{S}_{-i}^{-1}\mathbf{f}_i)+\mh(\gamma_i),
\end{equation}
where 
$$\mathcal{J}_{-i}(\bgamma_{-i})=\frac12\boldsymbol{y}^\top\mathsf{S}^{-1}_{-i}\boldsymbol{y}+\frac12\log\det\mathsf{S}_{-i}+\sum_{j\neq i}\mh(\gamma_j),$$
and $\bgamma_{-i}=[\gamma_1,\dots,\gamma_{i-1},\gamma_{i+1},\dots,\gamma_n]^\top$ with $\gamma_i$ removed from $\bgamma$. We have now isolated the terms with respect to $\gamma_i$.

Let $p_i=\mathbf{f}_i^\top\mathsf{S}_{-i}^{-1}\boldsymbol{y}$ and $q_i=\mathbf{f}_i^\top\mathsf{S}_{-i}^{-1}\mathbf{f}_i$. Note that $q_i>0$. Then we can define the scalar function
\begin{equation}
	\ml(\gamma_i)=-\frac12\frac{p_i^2\gamma_i}{1+q_i\gamma_i}+\frac12\log(1+q_i\gamma_i)+\mh(\gamma_i),
\end{equation}
and we have $\mathcal{J}(\bgamma)=\mathcal{J}_{-i}(\bgamma_{-i})+\ml(\gamma_i)$. The derivative of $\mathcal{J}(\bgamma)$ with respect to $\gamma_i$ is 
\begin{equation}\label{eq:Jderivative-1order}
	\frac{\partial\mathcal{J}(\bgamma)}{\partial \gamma_i}=\mathcal{L}'(\gamma_i)=\frac12\frac{q_i^2\gamma_i+(q_i-p_i^2)}{(1+q_i\gamma_i)^2}+\mh'(\gamma_i).
\end{equation}
Note that due to \eqref{invS}, we have
$$\mathbf{f}_i^\top\mathsf{S}(\bgamma)^{-1}\mathbf{f}_i=q_i-\frac{\gamma_iq_i^2}{1+\gamma_iq_i}=\frac{q_i}{1+q_i\gamma_i},$$
and
$$\mathbf{f}_i^\top\mathsf{S}(\bgamma)^{-1}\boldsymbol{y}=p_i-\frac{\gamma_ip_iq_i}{1+\gamma_iq_i}=\frac{p_i}{1+q_i\gamma_i}.$$
Therefore, the two different expressions of $\frac{\partial \mathcal{J}}{\partial \gamma_i}$ in \eqref{eq:Jderivative-1} and \eqref{eq:Jderivative-1order} are indeed consistent.

Let us drop the subscript for clarity, and consider the following univariate optimization problem:
\begin{equation}\label{minL1d}
	\min_{\gamma\geq 0}\mathcal{L}(\gamma)=-\frac12\frac{p^2\gamma}{1+q\gamma}+\frac12\log(1+q\gamma)+\mh(\gamma).
\end{equation}
Since $\ml(\gamma)$ is proper and $\ml(\gamma)\to +\infty$ as $\gamma\rightarrow+\infty$, this problem admits at least one minimizer, satisfying the KKT conditions: 
\begin{equation}\label{eq:KKT1d}
	\begin{cases}
		\frac12\frac{q^2\gamma+(q-p^2)}{(1+q\gamma)^2}+\mh'(\gamma)-\mu=0,\\
		\mu\geq 0,\gamma\geq 0,\\
		\mu\gamma=0.
	\end{cases}
\end{equation}

\begin{theorem}\label{thm:KKT1d}
	Assume $\mh$ is nonnegative in $[0,+\infty)$. Then the following statements hold:
	\begin{itemize}
		\item[(I)] If $\gamma^\star=0$ satisfies the KKT conditions \eqref{eq:KKT1d}, then it must hold that $q-p^2\geq -2\mh'(0^+)$ with $\mh'(0^{+})=\lim_{\gamma\rightarrow 0^+}\mh'(\gamma)$.
		\item[(II)] If $q-p^2> -2\mh'(0^+)$, then $\gamma^\star=0$ satisfies the KKT conditions \eqref{eq:KKT1d}.
		\item[(III)] Any $\gamma^\star>0$ satisfies the KKT conditions \eqref{eq:KKT1d} if and only if it is a positive root of
		\begin{equation}\label{eq:neq1d}
			\frac12\frac{q^2\gamma+(q-p^2)}{(1+q\gamma)^2}+\mh'(\gamma)=0.
		\end{equation}
	\end{itemize}
\end{theorem}
\begin{proof}
	Let $\mu^\star=\mh'(0^+)-\frac{p^2-q}{2}$. For (I), if $\gamma^\star=0$ satisfies the KKT conditions, then $\mu^\star$ satisfies the stationarity of \eqref{eq:KKT1d}. By the dual feasibility,  we must have $\mu^\star \geq 0$, which implies $q - p^2 \geq -2 \mathcal{H}'(0^+)$. For (II), if $q - p^2 > -2 \mathcal{H}'(0^+)$, then $\mu^\star > 0$ and hence $(\gamma^\star, \mu^\star) = (0, \mu^\star)$ satisfies the KKT conditions. 
	
	For (III), if $\gamma^\star > 0$ is a solution to \eqref{eq:neq1d}, then it is straightforward to verify that $(\gamma^\star,0)$ is a KKT point. Conversely, if $\gamma^\star>0$ satisfies the KKT conditions \eqref{eq:KKT1d}, by the complementary slackness, we have $\mu^\star=0$, which together with the stationarity implies that $\gamma^\star$ is a root of \eqref{eq:neq1d}.
\end{proof}

In the sparse Bayesian learning, which does not use any hyperprior, i.e., $\mh'(\gamma)=0$ on $[0,+\infty)$, the unique solution to \eqref{eq:KKT1d} is given by 
\begin{equation}\label{eq:sparsityCondNoPrior}
	\gamma^*=\begin{cases}0,&\text{if } q-p^2\geq0,\\\frac{p^2-q}{q^2},&\text{if } q-p^2<0.\end{cases}
\end{equation}
According to \cref{thm:KKT1d}, it is reasonable to choose $\mh$ as a strictly increasing function to promote the sparsity. This leads to the following corollary. 
\begin{corollary}\label{cor:sparsity}
	Assume $\mh$ is nonnegative and strictly increasing in $[0,+\infty)$. Then any KKT point $\gamma^\star$ of \eqref{eq:KKT1d} satisfies either $\gamma^\star=0$ or $\gamma^\star\in \left(0, \frac{p^2 - q}{q^2} \right)$.
\end{corollary}
\begin{proof}    
	Since $\ml'(\gamma)>0$ for $\gamma\geq\frac{p^2-q}{q^2}$, any positive root of \eqref{eq:neq1d} must lie in $\left(0,\frac{p^2-q}{q^2}\right)$. Based on \cref{thm:KKT1d}, if \( q - p^2 > -2 \mathcal{H}'(0^+) \), then \( \gamma^\star = 0 \) satisfies the KKT conditions. If \( q - p^2 < -2 \mathcal{H}'(0^+) \), we have $\ml'(0^+)<0$. Together with $\ml'(\frac{p^2 - q}{q^2})=\mh'(\frac{p^2 - q}{q^2})>0$ and the continuity of $\ml$, we can conclude that there exists at least one \( \gamma^\star \in \left(0, \frac{p^2 - q}{q^2} \right) \) that is a root of \eqref{eq:neq1d}, hence a KKT point. For \( q - p^2 = -2 \mathcal{H}'(0^+) \), the pair $(\gamma^\star,\mu^\star)=(0,0)$ also satisfies the KKT conditions.
\end{proof}

\begin{remark}
	Compared to the case without any hyperprior, a strictly increasing \( \mathcal{H} \) promotes stronger sparsity in \( \gamma^\star \) by relaxing the threshold from \( q - p^2 \geq 0 \) to \( q - p^2 \geq -2\mathcal{H}'(0^+) \). Furthermore, such a choice of \( \mathcal{H} \) also shrinks the nonzero solution \( \gamma^\star \) from \( \frac{p^2 - q}{q^2} \) toward values closer to zero.
\end{remark}

If $\mh$ is further convex, we have the following theorem.
\begin{theorem}
	\label{thm:localminimizer1d}
	Assume $\mh$ is nonnegative, strictly increasing and convex in $[0,\allowbreak+\infty)$. Then any $\gamma^\star$ satisfying \eqref{eq:KKT1d} is a local minimizer of \eqref{minL1d}.
\end{theorem}
\begin{proof}
	The second derivative of $\ml$ is calculated as
	\begin{equation}\label{eq:2derivativeL}
		\mathcal{L}''(\gamma)
		= \frac{q}{(1 + q\gamma)^2}\left(\frac{p^2}{1 + q\gamma}-\frac{q}{2}\right)  
		+ \mh''(\gamma)
	\end{equation}
	Since $\mh$ is nonnegative and strictly increasing, according to \cref{cor:sparsity} we know that either $\gamma^\star=0$ or $0<\gamma^\star<\frac{p^2-q}{q^2}$.
	
	If $\gamma^\star= 0$, then by item (I) of \cref{thm:KKT1d}, we have
	$$\ml'(0^+)=\frac{q-p^2}{2}+\mh'(0^+)\geq 0,$$
	which guarantees $\gamma^\star=0$ is a local minimizer.
	
	If $\gamma^\star\neq 0$, then 
	$$\frac{p^2}{1 + q\gamma^\star}-\frac{q}{2}\, \geq\, \frac{q}{2} \,> \,0.$$
	Since $\mh''\geq 0$, we get that $\ml''(\gamma^\star)>0$. Therefore, $\gamma^\star$ is a local minimizer. 
\end{proof}

\begin{remark}
	According to \cref{thm:localminimizer1d}, any KKT point $\bgamma^\star$
	of problem \eqref{minJ} is a coordinate-wise minimizer. However, $\bgamma^\star$ is not necessarily a local minimizer, as ascent directions may exist outside the coordinate axes. This is due to the structure of the Hessian matrix $\nabla^2\mathcal{J}(\bgamma^\star)$. As seen in \eqref{eq:Jderivative-2order}, while the diagonal entries are positive at $\bgamma^\star$, the off-diagonal entries usually are nonzero and the positive definiteness of the Hessian at $\bgamma^\star$ cannot be guaranteed. However, if $\mh$ is selected to be strongly convex with sufficiently large curvature such that the Hessian of $\mathcal{J}$ becomes positive semidefinite, then $\bgamma^\star$ is a global minimizer.
\end{remark}

\subsection{General Gamma prior} In this subsection, we consider a specific type of distribution as the hyperprior, namely the generalized Gamma distribution, which encompasses several notable probability density functions (PDFs). Our aim is to analyze the effect of incorporating this prior into the empirical Bayes framework.

The PDF of the generalized Gamma distribution is given by
\begin{equation}
	\pi(\gamma) = \frac{|\zeta|}{\Gamma(\alpha) \beta^{\zeta\alpha}} \gamma^{\zeta\alpha - 1} \exp\left( -\frac{\gamma^\zeta }{\beta^\zeta }\right), \quad \gamma \geq 0,
\end{equation}
where $\Gamma(\cdot)$ denotes the Gamma function. We require $\zeta \neq 0$, and $\alpha$, $\beta > 0$ are the shape and scale parameters, respectively. Here, we restrict the domain to $\gamma \in [0, +\infty)$ and, for notational convenience, define $1/0= \infty$.

Under the assumption of independence, the joint hyperprior over $\boldsymbol{\gamma}$ is
\begin{equation*}
	\label{gammapdf}
	\pi(\boldsymbol{\gamma}) \propto \prod_{i=1}^n \gamma_i^{\zeta\alpha - 1} \exp\left( -\frac{\gamma_i^\zeta }{\beta^\zeta }\right) 
	= (\gamma_1 \cdots \gamma_n)^{\zeta\alpha - 1} \exp\left( -\sum_{i=1}^n \frac{\gamma_i^\zeta }{\beta^\zeta } \right).
\end{equation*}
The corresponding cost function \eqref{J} becomes
\begin{equation}
	\mathcal{J}(\boldsymbol{\gamma}) = \frac{1}{2} \boldsymbol{y}^\top \mathsf{S}(\boldsymbol{\gamma})^{-1} \boldsymbol{y} + \frac{1}{2} \log \det \mathsf{S}(\boldsymbol{\gamma}) 
	- (\zeta\alpha - 1) \sum_{i=1}^n \log \gamma_i + \sum_{i=1}^n \frac{\gamma_i^\zeta }{\beta^\zeta }.
\end{equation}

As in the previous subsection, we isolate a single component and define the univariate objective
\begin{equation}
	\mathcal{L}(\gamma) = -\frac{1}{2} \frac{p^2 \gamma}{1 + q \gamma} + \frac{1}{2} \log(1 + q \gamma) + \mathcal{H}(\gamma),
\end{equation}
where
\begin{equation}
	\mathcal{H}(\gamma) = -(\zeta\alpha - 1)\log \gamma + \frac{\gamma^\zeta }{\beta^\zeta}.
\end{equation}
According to the definition of $\mh$, we have
\begin{equation}
	\mathcal{H}'(\gamma) = \frac{1-\zeta\alpha}{\gamma} + \frac{\zeta}{\beta^\zeta} \gamma^{\zeta - 1}.
\end{equation}

We now analyze the KKT points under different combinations of hyperparameters $\alpha$, $\beta$ and $\zeta$, using \cref{thm:KKT1d} and \cref{cor:sparsity}.
\begin{enumerate}[(I)]
	\item $\zeta\alpha>1$: Since $\alpha>0$, we must have $\zeta>0$, and thus $\mh'(0^+)=-\infty$. Therefore, any KKT point $\gamma^\star$ must be positive. 
	\item $\zeta\alpha<1$: In this case, $\zeta$ can be either positive or negative.
	\begin{enumerate}[(1)]
		\item If $\zeta>0$,  then $\ml(0^+)=-\infty$, and hence $\gamma^\star=0$ is the global minimizer. 
		When $q-p^2<0$, the stationarity equation $\ml'(\gamma)=0$ takes the form 
		\begin{equation}\label{eq:neqnonconvex1}
			\frac12\frac{q^2\gamma+(q-p^2)}{(1+q\gamma)^2}+\frac{1-\zeta\alpha}{\gamma}+\frac{\zeta}{\beta^\zeta} \gamma^{\zeta - 1}=0,
		\end{equation}
		which may admit positive roots in the interval $\left(0,\frac{p^2-q}{q^2}\right)$. Each such root is a KKT point.
		\item If $\zeta<0$, then $\mh'(0^+)=-\infty$, and thus any KKT point $\gamma^\star$ must be positive. 
	\end{enumerate}
	\item $\zeta\alpha=1$: Then $\zeta>0$. 
	\begin{enumerate}[(1)]
		\item If $\zeta>1$, then $\mh'(0^+)=0$. Hence $\gamma^\star=0$ is a KKT point if $q-p^2\geq 0$; otherwise, there exists a KKT point $\gamma^\star\in\left(0,\frac{p^2-q}{q^2}\right)$.
		\item If $\zeta=1$, then $\mh'(0^+)=\frac{1}{\beta}$. Hence $\gamma^\star=0$ is a KKT point if $q-p^2\geq -\frac{2}{\beta}$; otherwise, a KKT point lies in $\left(0,\frac{p^2-q}{q^2}\right)$.
		\item If $0<\zeta<1$, then $\mh'(0^+)=+\infty$, and thus $\gamma^\star=0$ must be a KKT point. When $q-p^2<0$, then $\ml'(\gamma)=0$ becomes 
		\begin{equation}\label{eq:neqnonconvex2}
			\frac12\frac{q^2\gamma+(q-p^2)}{(1+q\gamma)^2}+\frac{\zeta}{\beta^\zeta} \gamma^{\zeta - 1}=0,
		\end{equation}
		which may have positive roots in $\left(0,\frac{p^2-q}{q^2}\right)$, each corresponding to a KKT point.
	\end{enumerate} 
\end{enumerate}

\begin{table}[t]\label{table:GGammacases}
\centering
\caption{Characterization of the KKT points under different hyperparameter regimes.}
%\label{table:GGammacases}
\renewcommand{\arraystretch}{1.25}
\setlength{\tabcolsep}{6pt}
\resizebox{\textwidth}{!}{
\begin{tabular}{c|c|c|>{\centering\arraybackslash}m{7cm}}
	\hline
	$\alpha$-condition & $\zeta$-condition & $\mh'(0^+)$ & KKT points \\ 
	\hline
	$\zeta\alpha>1$ & $\zeta>0$ & $\mh'(0^+)=-\infty$ & $\gamma^\star>0$. \\ 
	\hline
	\multirow{2}{*}{$\zeta\alpha<1$} 
	& $\zeta>0$ & $\mh'(0^+)=+\infty$ & $\gamma^\star=0$ (global minimizer); may exist other KKT points in $\left(0,\frac{p^2-q}{q^2}\right)$ if $q-p^2<0$. \\ 
	\cline{2-4}
	& $\zeta<0$ & $\mh'(0^+)=-\infty$ & $\gamma^\star>0$. \\ 
	\hline
	\multirow{3}{*}{$\zeta\alpha=1$} 
	& $\zeta>1$ & $\mh'(0^+)=0$ & 
	$\begin{cases}
		\gamma^\star=0, & \text{if } q-p^2\geq 0,\\
		\gamma^\star\in\left(0,\frac{p^2-q}{q^2}\right), & \text{if } q-p^2<0.
	\end{cases}$ \\
	\cline{2-4}
	& $\zeta=1$ & $\mh'(0^+)=\frac{1}{\beta}$ & 
	$\begin{cases}
		\gamma^\star=0, & \text{if } q-p^2\geq -\frac{2}{\beta},\\
		\gamma^\star\in\left(0,\frac{p^2-q}{q^2}\right), & \text{if } q-p^2<-\frac{2}{\beta}.
	\end{cases}$ \\
	\cline{2-4}
	& $0<\zeta<1$ & $\mh'(0^+)=+\infty$ & 
	$\gamma^\star=0$; may exist other KKT points in $\left(0,\frac{p^2-q}{q^2}\right)$ if $q-p^2<0$. \\ 
	\hline
\end{tabular}}
\end{table}

We summarize the above results in Table \ref{table:GGammacases} and provide some commonly used examples. 

\begin{example}\textbf{Gamma hyperprior.}\label{example:Gamma} With a Gamma hyperprior, $\mathrm{G}(\alpha,\beta)$, the PDF is
	\begin{equation*}
		\pi(\gamma)=\frac1{\mathsf{\Gamma}(\alpha)\beta^\alpha}\gamma^{\alpha-1}\exp\left(-\frac{\gamma}\beta\right), ~~~ \gamma\geq0.
	\end{equation*}
	It corresponds to the cases with ``$\zeta=1$'' in Table \ref{table:GGammacases}. We arrive at the following conclusion: 
	\begin{itemize}
		\item If $\alpha>1$, then $\gamma^\star>0$ necessarily holds, and thus no sparse solution can be obtained.
		\item If $0<\alpha<1$, we have $\gamma^\star=0$ as the global minimizer, hence $\bgamma^\star=0$ solves \eqref{minJ} and the corresponding hyperprior promotes the spasity.
		\item If $\alpha=1$, the Gamma distribution reduces to the half Laplace case, see \cref{example:halfLaplace}. 
	\end{itemize}
\end{example}

\begin{example}\textbf{Inverse Gamma hyperprior.}\label{example:invGamma}
	With an inverse Gamma hyperprior, $\mathrm{InvG}(\alpha,\beta)$, the PDF is
	\begin{equation*}
		\pi(\gamma)=\frac{\beta^\alpha}{\mathsf{\Gamma}(\alpha)}\gamma^{-(\alpha+1)}\exp\left(-\frac\beta{\gamma}\right), ~~~ \gamma\geq0.
	\end{equation*}
	It corresponds to the case ``$\zeta<0$'' in Table \ref{table:GGammacases} with $\zeta=-1$. Hence, $\gamma^\star$ must be positive, and we can conclude that adding an inverse Gamma hyperprior would reduce the sparsity of $\bgamma$.
\end{example}

\begin{example}\textbf{Half-Gaussian hyperprior.}\label{example:halfGauss} The PDF of half-Gaussian prior $\mathcal{N}^+(0,\theta^2)$ is
	\begin{equation}
		\pi(\gamma)=\frac{\sqrt 2}{\sqrt{\pi}\theta}\exp\left(-\frac{\gamma^2}{2\theta^2}\right),~~~ \gamma\geq0.
	\end{equation}
	It corresponds to the case ``$\zeta\alpha=1$ and $\zeta>1$'' in Table \ref{table:GGammacases} with $\zeta=2$, $\alpha=\frac{1}{2}$, $\beta=\sqrt{2}\theta$. 
	If $q-p^2\geq 0$, then $\gamma^\star=0$ is the global minimizer. If $q-p^2<0$, then the stationarity condition $\ml'(\gamma)=0$ is equivalent to the cubic equation
	\begin{equation}\label{eq:halfGauss}
		4q^2\gamma^{3}+8q\gamma^2+(4+\beta^2q^2)\gamma+\beta^2(q-p^2)=0.
	\end{equation}
	By Vieta’s relations one can show that \eqref{eq:halfGauss} has exactly one positive real root in $\left(0,\frac{p^2-q}{q^2}\right)$, which is the global minimizer in this case. A closed form is available via the Cardano formula. Comparing with the result without the hyperprior given in \eqref{eq:sparsityCondNoPrior}, we can see that half-Gaussian hyperprior won't change the sparsity but can push $\bgamma^*$ closer to zero.
\end{example}

\begin{example}\textbf{Half-Laplace hyperprior.}\label{example:halfLaplace} With a half-Laplace hyperprior, the  PDF is
	\begin{equation*}
		\pi(\gamma)=\frac1{\beta}\exp\left(-\frac{\gamma}\beta\right),~~~ \gamma\geq0,
	\end{equation*}
	which corresponds to the case ``$\zeta\alpha=1$ and $\zeta=1$'' in Table \ref{table:GGammacases} with $\mh(\gamma)=\frac{\gamma}{\beta}$. Since $\beta>0$, the condition promoting $\gamma=0$ is weaker than the no-hyperprior condition $q-p^2\geq 0$. Thus adding a half-Laplace prior can promote sparsity. When $q-p^2<-\frac{2}{\beta}$, the equation $\ml'(\gamma)=0$ reduces to 
	\begin{equation}\label{eq:gammaLaplace}
		2q^2\gamma^2+(4q+\beta q^2)\gamma+2+\beta(q-p^2)=0,
	\end{equation}
	which has a unique positive root that can be explicitly given by the quadratic formula. Therefore, the global minimizer is
	$$\gamma^\star=\begin{cases}0, &\text{if } q-p^2\geq -\frac{2}{\beta},\\ \frac{-(4+\beta q)+\sqrt{\beta^2q^2+8\beta p^2}}{4q},&\text{if } q-p^2<-\frac{2}{\beta}.\end{cases}$$
\end{example}

\begin{example}\textbf{Half-Generalized Gaussian hyperprior with $0<\zeta<1$.} \label{example:halfLp} 
	For the half-Generalized Gaussian hyperprior with $0<\zeta<1$, the PDF is
	\begin{equation}
		\pi(\gamma)=\frac{\zeta}{\mathsf{\Gamma}(\alpha)\beta}\exp\left(-(\frac{\gamma}{\beta })^{\zeta}\right), \quad \gamma \geq 0.
	\end{equation}
	It corresponds to the case ``$\zeta\alpha=1$ and $0<\zeta<1$'' in Table \ref{table:GGammacases} with a concave $\mh(\gamma)=(\frac{\gamma}{\beta})^{\zeta}$.  Since $\mh'(0^+)=+\infty$, one always has $\gamma^\star=0$ as a local minimizer of $\ml(\gamma)$ in $[0,+\infty)$. And any positive root of the equation
	\begin{equation}\label{eq:halfgeneralgau}
		2\zeta q^{2}\gamma^{\zeta+1}+4\zeta q\gamma^{\zeta}+2\zeta \gamma^{\zeta-1}+\beta^{\zeta}q^2\gamma+\beta^{\zeta}(q-p^2)=0,
	\end{equation}
	is a KKT point. In particular, when $\zeta=\frac{1}{2}$, any positive solution $s$ of the quartic equation 
	\begin{equation}\label{eq:exponentialpower}
		q^{2}s^{4}+\sqrt{\beta}\,q^{2}s^{3}+2q\,s^{2}+\sqrt{\beta}\,(q-p^{2})s+1=0,
	\end{equation}
	yields a KKT point with $\gamma=s^2$. Therefore, the global minimizer depends on the specific values of $\beta$, $p$, and $q$, illustrating the added complexity induced by a nonconvex $\mathcal{H}$.
\end{example}

\section{Numerical algorithm}
In this section, referring to the approach in \cite{yu2024hyperparameter}, we present a numerical algorithm for minimizing \eqref{J} within the empirical Bayes framework and analyze its convergence.

\subsection{Proximal alternating linearized minimization}

We begin by reformulating the problem using an auxiliary variable, which enables an alternating minimization scheme between the original variables and the auxiliary variable.
This reformulation allows us to apply a proximal alternating linearized minimization algorithm \cite{bolte2014proximal} to efficiently compute a stationary point, even in the presence of nonconvex terms.

Following the idea in \cite{yu2024hyperparameter}, we construct a strict upper-bounding auxiliary function $\mathcal{F}$ for the first term in $\mathcal{J}(\boldsymbol{\gamma})$:
\begin{equation}	\mathcal{F}(\boldsymbol{x},\boldsymbol{\gamma})=\sigma\|\mathbf{F}\boldsymbol{x}-\boldsymbol{y}\|^2+\boldsymbol{x}^\top\bGamma^\dagger \boldsymbol{x}+\sum_{i\in I_\gamma}\iota_{\{0\}}(x_i),\quad \boldsymbol{x}\in\mathbb{R}^n,\boldsymbol{\gamma}\in\mathbb{R}_+^n,
\end{equation}
where $\bGamma^\dagger$ is the Moore–Penrose inverse of $\bGamma$, i.e.,
\begin{equation}[\bGamma^\dagger]_{ii}=\begin{cases}\gamma_i^{-1},&\text{if } \gamma_i\neq0,\\0,&\text{if } \gamma_i=0,\end{cases}
\end{equation}
$I_\gamma= \{1 \leq i\leq n : \gamma_i = 0\}$, and $\iota_{A}(x)$ is the indicator function defined as
\begin{equation}
	\iota_{A}(x)=\begin{cases}+\infty,&\text{if } x\notin A,\\0,&\text{if } x\in A.\end{cases}
\end{equation}

For any $\bgamma\in\mathbb{R}^{n}_+$ and $\boldsymbol{y}\in\mathbb{R}^{m}$, the optimization problem $$\min\limits_{\boldsymbol{x}\in\mathbb{R}^n}\mathcal{F}(\boldsymbol{x},\boldsymbol{\gamma})$$ admits a unique minimizer $\boldsymbol{x}^{*}(\boldsymbol{\gamma})$ given by
\begin{equation} \label{eq:xsub}
	\boldsymbol{x}^*(\boldsymbol{\gamma})=\bGamma\mathbf{F}^\top(\mathsf{S}({\boldsymbol{\gamma}}))^{-1}\boldsymbol{y} \quad \mbox{and} \quad
	\mathcal{F}(\boldsymbol{x}^*(\bgamma),\boldsymbol{\gamma})=\boldsymbol{y}^\top(\mathsf{S}(\boldsymbol{\gamma}))^{-1}\boldsymbol{y}. 
\end{equation}
Therefore, $\mathcal{J}$ in \eqref{J} can be equivalently expressed as
\begin{equation}\label{eq:eqxgamma}	
	\mathcal{J}(\bgamma)=\min_{\boldsymbol{x}\in\mathbb{R}^n}\frac12\mathcal{F}(\boldsymbol{x},\boldsymbol{\gamma})+g(\boldsymbol{\gamma}),
\end{equation} 
where 
\begin{equation}\label{eq:def_g}
	g(\bgamma)=\frac12\log\det\mathsf{S}(\boldsymbol{\gamma})+\sum\limits_{i=1}^n\mh(\gamma_i).
\end{equation}
Minimizing $\mathcal{J}$ on $\boldsymbol{\gamma}\in\mathbb{R}^n_+$ leads to the joint minimization problem
\begin{equation}\label{model:F} 
	\min_{\boldsymbol{\gamma}\in\mathbb{R}^n_+, \boldsymbol{x}\in\mathbb{R}^n}\frac{1}{2}\mathcal{F}(\boldsymbol{x},\boldsymbol{\gamma})+g(\boldsymbol{\gamma}).
\end{equation}

A natural approach for such two-block minimization problems is the alternating minimization algorithm (AMA). Specifically, starting with an initial guess $\boldsymbol{\gamma}^{(0)}$, AMA for the above problem alternates between the following updates
\begin{align}
	\boldsymbol{x}&-\text{subproblem}: & \boldsymbol{x}^{(k+1)}=\arg\min_{\boldsymbol{x}
		\in\mathbb{R}^n} \ & \mathcal{F}(\boldsymbol{x},\bgamma^{(k)}), \label{x-pro}\\
	\boldsymbol{\gamma}&-\text{subproblem}: & \boldsymbol{\bgamma}^{(k+1)}=\arg\min_{\boldsymbol{\gamma}\in\mathbb{R}^n_+} \  &\frac{1}{2} \mathcal{F}(\boldsymbol{x}^{(k+1)},\boldsymbol{\gamma})+g(\bgamma). \label{gamma-pro}
\end{align}

According to \eqref{eq:xsub}, the $\boldsymbol{x}$-subproblem has a closed-form solution
\begin{equation}\label{eq:xsubproblem}
	\boldsymbol{x}^{(k+1)}=\bGamma^{(k)}\mathbf{F}^\top(\mathsf{S}({\boldsymbol{\gamma}}^{(k)}))^{-1}\boldsymbol{y}.   
\end{equation}
However, the $\bgamma$-subproblem is substantially more challenging: the $\log$-determinant term is nonconvex, and $\mathcal{H}$ may also be nonconvex, which prevents a closed-form solution and making the global minimizer computationally demanding.

To address this difficulty, we replace $g(\bgamma)$ with a surrogate function $\tilde{g}(\bgamma,\bgamma^{(k)})$, leading to the PALM update
\begin{align}
	\boldsymbol{x}&-\text{subproblem}: & \boldsymbol{x}^{(k+1)}=\arg\min_{\boldsymbol{x}
		\in\mathbb{R}^n} \ & \mathcal{F}(\boldsymbol{x},\bgamma^{(k)}), \label{x-PAMApro}\\
	\boldsymbol{\gamma}&-\text{subproblem}:  &\boldsymbol{\bgamma}^{(k+1)}=\arg\min_{\boldsymbol{\gamma}\in\mathbb{R}^n_+} \   &\frac{1}{2}\mathcal{F}(\boldsymbol{x}^{(k+1)},\boldsymbol{\gamma})+\tilde{g}(\bgamma,\bgamma^{(k)})+\frac{\tau}{2}\|\bgamma-\bgamma^{(k)}\|^2, \label{gamma-PAMApro}
\end{align}
where $\tau>0$ is the proximal parameter. The choice of the surrogate function $\tilde{g}$ is critical for both computational efficiency and convergence. Based on the property of $\mh$, we present two cases:
\begin{itemize}
	\item for a convex $\mh$: 
	\begin{equation}
		\label{def:convexHsurrogate}
		\tilde{g}(\bgamma,\bgamma^{(k)})=\frac12\log\det\mathsf{S}(\boldsymbol{\gamma}^{(k)})+\frac{1}{2}\langle\nabla \log\det\mathsf{S}(\boldsymbol{\gamma}^{(k)}),\bgamma-\bgamma^{(k)}\rangle+\sum\limits_{i=1}^n\mh(\gamma_i),
	\end{equation}
	\item for a nonconvex $\mh$:
	\begin{equation}
		\label{def:nonconvexHsurrogate}
		\tilde{g}(\bgamma,\bgamma^{(k)})=g(\bgamma^{(k)})+\langle\nabla g(\boldsymbol{\gamma}^{(k)}),\bgamma-\bgamma^{(k)}\rangle.
	\end{equation}
\end{itemize}

\subsection{Solving the $\bgamma$-subproblem}
Before discussing the solver of the $\bgamma$-sub-problem, we can show the sequence $\{(\boldsymbol{x}^{(k)},\bgamma^{(k)})\}$ generated by PALM satisfies the following proposition.

\begin{proposition}\label{prop:gamma}
	Let $\mathcal{H}$ be nonnegative, and let $\{(\boldsymbol{x}^{(k)},\bgamma^{(k)})\}$ be generated by PALM \eqref{x-PAMApro}-\eqref{gamma-PAMApro} with an initial $\bgamma^{(0)}$. If we have $\gamma_i^{(K)}=0$ for some $K>0$ and $1\leq i\leq n$. Then $x_i^{(k)}=\gamma_i^{(k)}=0$ for all $k>K$.
\end{proposition}
\begin{proof}
	
	It suffices to prove that $\gamma_i^{(K)}=0$ implies $x_i^{(K+1)}=0$ and $\gamma_i^{(K+1)}=0$. From \eqref{eq:xsubproblem}, $\gamma_i^{(K)}=0$ directly yields $x_i^{(K+1)}=0$.
	
	For the $\boldsymbol{\gamma}$-subproblem given in \eqref{gamma-PAMApro}, note that when $\gamma_i^{(K)}=0$ and $x_i^{(K+1)}=0$, replacing a positive $\gamma_i$ by zero decreases the terms $\left(\boldsymbol{x}^{(K+1)}\right)^\top\bGamma^\dagger \boldsymbol{x}^{(K+1)}$, $\log\det\mathsf{S}(\boldsymbol{\gamma})$ and $\|\boldsymbol{\gamma}-\boldsymbol{\gamma}^{(K)}\|^2$, while leaving $\sum_{i\in I_\gamma}\iota_{\{0\}}(x_i^{(K+1)})$ unchanged. Hence, $\gamma_i^{(K+1)}=0$ for such $i$.
\end{proof}

\cref{prop:gamma} shows that once $\gamma_i^{(k)}=0$ for some $i$, both $x_i$ and $\gamma_i$ will remain zero in all subsequent iterations. Consequently, these components can be removed from the optimization, and only the remaining indices need to be updated. Moreover, if $x_i^{(k)}\neq 0$, then necessarily $\gamma_i^{(k+1)}>0$; otherwise, the term $\sum_{i\in I_\gamma}\iota_{\{0\}}(x_i^{(k)})$ would become infinite. Therefore, in the following we only focus on those $\gamma_i$'s with $\gamma_i^{(k)}>0$ and $x^{(k)}_i\neq 0$.

Due to the separability of the objective in \eqref{gamma-PAMApro} with respect to $\gamma_i$, the update of $\boldsymbol{\gamma}^{(k+1)}$ reduces to solving a set of independent univariate problems:
\begin{equation}\label{1dgamma}
	\gamma^{(k+1)}_i = \arg\min_{\gamma_i>0} \ \frac{\big(x^{(k+1)}_i\big)^2}{2\gamma_i} + \mathcal{L}(\gamma_i) + \frac{\tau}{2}\big(\gamma_i - \gamma_i^{(k)}\big)^2,
\end{equation}
where
\[
\mathcal{L}(\gamma_i) =
\begin{cases}
	\frac{1}{2}\tilde{q}_{ii}^{(k)}\gamma_i + \mathcal{H}(\gamma_i), & \text{if $\mathcal{H}$ is convex}, \\[4pt]
	\left(\frac{1}{2}\tilde{q}_{ii}^{(k)}+\mh'(\gamma_i^{(k)})\right)\gamma_i, & \text{if $\mathcal{H}$ is nonconvex},
\end{cases}
\]
and $$\tilde{q}_{ii}^{(k)}=\mathbf{f}_i^\top\mathsf{S}(\boldsymbol{\gamma}^{(k)})^{-1}\mathbf{f}_i.$$

For the hyperpriors considered in \cref{example:Gamma}--\ref{example:halfLp}, the update formula \eqref{1dgamma} reduces to finding the positive root of a cubic polynomial. In this case, $\gamma_i^{(k+1)}$ can be computed exactly via Cardano’s formula, enabling an efficient closed-form update of $\boldsymbol{\gamma}$ without requiring any inner iterations.

\subsection{Convergence analysis}

\begin{algorithm}[t]
	\renewcommand{\algorithmicrequire}{\textbf{Input:}}
	\renewcommand{\algorithmicensure}{\textbf{Output:}}
	\caption{The PALM algorithm for minimizing $\mj(\bgamma)$ defined in \eqref{J}.}
	\label{alg}
	\begin{algorithmic}[1] 
		\REQUIRE $\boldsymbol{y}$, $\mathbf{F}$, $\sigma$, $\tau$ and $\omega$;
		\ENSURE $\boldsymbol{x}^{(k+1)}$ and $\bgamma^{(k+1)}$; 
		\STATE Initialize $\bgamma^{(0)}$ and set $k=0$
		\WHILE {$\boldsymbol{x}^{(k)}$ not converged}
		
		\STATE $\boldsymbol{x}$-update: $\boldsymbol{x}^{(k+1)}=\bGamma^{(k)}\mathbf{F}^\top(\mathsf{S}({\boldsymbol{\gamma}}^{(k)}))^{-1}\boldsymbol{y}$
		\STATE $\boldsymbol{\gamma}$-update:
		\FOR{$i = 1, \cdots, n$}
		\IF{$\gamma^{(k)}_i = 0$}
		\STATE set $\gamma^{(k+1)}_i = 0$
		\ELSE
		\STATE
		$\gamma^{(k+1)}_i =
		\arg\min\limits_{\gamma_i \in \mathbb{R}_{++}}
		\frac{(x^{(k+1)}_i)^2}{2\gamma_i}
		+ [\tilde{g}(\boldsymbol{\gamma},\boldsymbol{\gamma}^{(k)})]_i
		+ \frac{\tau}{2}(\gamma_i - \gamma_i^{(k)})^2$
		\IF{$\gamma^{(k+1)}_i < \omega$}
		\STATE set $\gamma^{(k+1)}_i = 0$
		\ENDIF
		\ENDIF
		\ENDFOR
		\STATE $k=k+1$
		\ENDWHILE
		\STATE \textbf{return} $\boldsymbol{x}^{(k+1)}$ and $\bgamma^{(k+1)}$
	\end{algorithmic}
\end{algorithm}

In this subsection, we study the convergence of the proposed PALM algorithm. First, we summarize the PALM algorithm to minimize $\mj(\bgamma)$ in \cref{alg}. For the convergence analysis we restrict $\mh$ to be either convex or concave. This condition is satisfied by the hyperpriors in \cref{example:Gamma}, and \cref{example:halfGauss}-\ref{example:halfLp}, which are well-suited for promoting sparsity. The convergence results of PALM are given in the following theorem.

\begin{theorem}\label{thm:convergence}
	Let $\tau>0$, $\mathcal{H}$ be nonnegative and either convex or concave, and $\{\left(\boldsymbol{x}^{(k)},\bgamma^{(k)}\right)\}$ be generated by PALM given in \cref{alg} with an initial positive $\bgamma^{(0)}$, i.e., $\bgamma^{(0)}\in\mathbb{R}^n_{++}$. Then the following convergence results hold:
	\begin{enumerate}[(I)]
		\item The sequence $\{\mathcal{J}(\bgamma^{(k)})\}$ is strictly decreasing and convergent.
		\item There exists a subsequence $\{s_k\}\subset\mathbb{N
		}$ such that $\{\boldsymbol{x}^{(s_k)}\}$ and $\{\bgamma^{(s_k)}\}$ converge.
		\item The limit $\bgamma^\star$ of $\{\bgamma^{(s_k)}\}$ is a KKT point of \eqref{minJ}, that is, for each $1\leq i\leq n$, either $\gamma_i^\star=0$ or $\frac{\partial \mathcal{J}(\bgamma^\star)}{\partial\gamma_i}=0$ holds.
	\end{enumerate}
\end{theorem}
\begin{proof} 
	Proof of (I). From \eqref{eq:eqxgamma} and the $\boldsymbol{x}$-subproblem in \eqref{x-PAMApro}, we obtain 
	\begin{equation}\label{eq:descentJ1}
		\mathcal{J}(\bgamma^{(k)})=\frac{1}{2}\mathcal{F}(\boldsymbol{x}^{(k+1)},\bgamma^{(k)})+g(\bgamma^{(k)})    
	\end{equation}
	and
	\begin{equation}\label{eq:descentJ2}
		\mathcal{J}(\bgamma^{(k+1)})=\frac{1}{2}\mathcal{F}(\boldsymbol{x}^{(k+2)},\bgamma^{(k+1)})+g(\bgamma^{(k+1)})\leq \frac{1}{2}\mathcal{F}(\boldsymbol{x}^{(k+1)},\bgamma^{(k+1)})+g(\bgamma^{(k+1)}).    
	\end{equation}
	According to \eqref{eq:def_g} and the concavity of $\log\det\mathsf{S}(\boldsymbol{\gamma})$, we have
	\begin{equation}
		\label{eq:upperboundg}
		g(\bgamma^{(k+1)})\leq \frac12\log\det\mathsf{S}(\boldsymbol{\gamma}^{(k)})+\frac{1}{2}\langle\nabla \log\det\mathsf{S}(\boldsymbol{\gamma}^{(k)}),\bgamma^{(k+1)}-\bgamma^{(k)}\rangle+\sum\limits_{i=1}^n\mh(\gamma_i^{(k+1)}).
	\end{equation}
	According to the definition of $\tilde{g}(\bgamma,\bgamma^{(k)})$ in the case of a convex $\mh$, we obtain
	\begin{equation}\label{eq:cong}
		g(\bgamma^{(k+1)})\leq \tilde{g}(\bgamma^{(k+1)},\bgamma^{(k)}).
	\end{equation}
	For a concave $\mh$, based on \eqref{def:nonconvexHsurrogate} and the concavicity of $g(\bgamma)$ we can obtain \eqref{eq:cong} in the similar way.

	From the $\bgamma$-subproblem in \eqref{gamma-PAMApro}, it follows that
	\begin{eqnarray*}\label{eq:descentJ3}
		&&\frac{1}{2}\mathcal{F}(\boldsymbol{x}^{(k+1)},\bgamma^{(k+1)})+\tilde{g}(\bgamma^{(k+1)},\bgamma^{(k)})
		+\frac{\tau}{2}\|\bgamma^{(k+1)}-\bgamma^{(k)}\|^2\nonumber\\
		&\leq& \frac{1}{2}\mathcal{F}(\boldsymbol{x}^{(k+1)},\bgamma^{(k)})+\tilde{g}(\bgamma^{(k)},\bgamma^{(k)})\\
		&=&    \frac{1}{2}\mathcal{F}(\boldsymbol{x}^{(k+1)},\bgamma^{(k)})+
		\frac12\log\det\mathsf{S}(\boldsymbol{\gamma}^{(k)})+\sum\limits_{i=1}^n\mh(\gamma_i^{(k)})\nonumber\\
		&=&\frac{1}{2}\mathcal{F}(\boldsymbol{x}^{(k+1)},\bgamma^{(k)})+g(\bgamma^{(k)}).
	\end{eqnarray*}
	Combining the above inequalities yields
	\begin{equation}\label{eq:sufficientdescentJ}
		\mathcal{J}(\bgamma^{(k+1)})\leq \mathcal{J}(\bgamma^{(k)})-\frac{\tau}{2}\|\bgamma^{(k+1)}-\bgamma^{(k)}\|^2.
	\end{equation}
	Hence, $\{\mathcal{J}(\bgamma^{(k)})\}$ is strictly decreasing. Moreover, by the definition of $\mathcal{J}(\bgamma)$ in \eqref{J}, we know it is always bounded below by $-\frac12\log\sigma$. Therefore, the sequence $\{\mathcal{J}(\bgamma^{(k)})\}$ converges.
	
	Proof of (II). We now show that $\{\bgamma^{(k)}\}$ is bounded. Observe that 
	$$\log\det\mathsf{S}(\boldsymbol{\gamma})=\log\det\left(\sigma^{-1}\mathbf{I}_{m}+\sum^n_{i=1}\gamma_i\mathbf{f}_i\mathbf{f}_i^\top\right)$$
	is coercive, i.e., $\lim\limits_{\|\bgamma\|\rightarrow+\infty}\log\det\mathsf{S}(\boldsymbol{\gamma})=+\infty$.
	Moreover,
	$$\frac12\log\det\mathsf{S}(\boldsymbol{\gamma}^{(k)})\leq\mathcal{J}(\bgamma^{(k)})\leq \mathcal{J}(\bgamma^{(0)}),~\forall k\geq 0.$$
	which implies boundedness of $\{\bgamma^{(k)}\}$. Thus, there exists a subsequence $s_k$ such that $\{\bgamma^{(s_k)}\}$ converges. Due to the closed-form expression of $\boldsymbol{x}^{(k)}$ in \eqref{eq:xsubproblem}, we conclude that $\{\boldsymbol{x}^{(s_k)}\}$ also converges.
	
	Proof of (III). Taking the limit in the sufficiently descent inequality \eqref{eq:sufficientdescentJ}, we find
	$$\lim\limits_{k\rightarrow+\infty}\|\bgamma^{(k+1)}-\bgamma^{(k)}\|^2=0.$$
	Due to the convergence of $\{\bgamma^{(s_k)}\}$, we define the limit point $\bgamma^\star$:
	$$\lim\limits_{s_k\to+\infty}\bgamma^{(s_k)}=\bgamma^\star.$$ 
	For any $1\leq i\leq n$, if $\gamma_i^{(s_k)}=0$ for some $s_k\in\mathbb{N}$, based on \cref{prop:gamma} we have $\gamma_i^\star=0$. Otherwise, the stationarity condition for \eqref{1dgamma} gives
	\begin{align}
		-\frac12\left(\frac{x_i^{(s_k)}}{\gamma_i^{(s_k+1)}}\right)^2+\frac{1}{2}\tilde{q}_{ii}^{(s_k)}+\mh'\left(\gamma_i^{(s_k+1)}\right)+\tau\left(\gamma_i^{(s_k+1)}-\gamma_i^{(s_k)}\right)&=0, ~ \text{if $\mathcal{H}$ is convex},\label{eq:gammastationality1} \\
		-\frac12\left(\frac{x_i^{(s_k)}}{\gamma_i^{(s_k+1)}}\right)^2+\frac{1}{2}\tilde{q}_{ii}^{(s_k)}+\mh'\left(\gamma_i^{(s_k)}\right)+\tau\left(\gamma_i^{(s_k+1)}-\gamma_i^{(s_k)}\right)&=0, ~ \text{if $\mathcal{H}$ is concave}.\label{eq:gammastationality2}
	\end{align}
	Define $\tilde{p}_i^{(k)}=\mathbf{f}_i^\top\mathsf{S}(\bgamma^{(k)})^{-1}\boldsymbol{y}$. By \eqref{eq:xsubproblem}, we have
	$x_i^{(s_k)}=\gamma_i^{(s_k)}\tilde{p}_i^{(s_k)}.$ Substituting it into \eqref{eq:gammastationality1} and \eqref{eq:gammastationality2}, after taking the limit on both sides we obtain
	$$\frac12\tilde{q}_{ii}^\star-\frac12\left(\tilde{p}_{i}^\star\right)^2+\mh'(\gamma_i^\star)=0.$$
	From the expression of $\frac{\partial \mathcal{J}(\bgamma)}{\partial\gamma_i}$ in \eqref{eq:Jderivative-1}, we conclude that $\frac{\partial \mathcal{J}(\bgamma^\star)}{\partial\gamma_i} = 0$.
\end{proof}
\section{Numerical experiments}

In this section, we present some numerical results that illustrate how the choice of the hyperprior from the generalized Gamma family affects the sparsity of the solution to \eqref{minJ}.

We consider a 2D image deblurring problem with a Gaussian blurring operator $\mathbf{K}$ as the test problem. We assume that the inverse variance $\sigma$ of the noise and the standard deviation $\sigma_{ker}$ of the Gaussian blurring kernel are known. 
In practice, images are often not sparse in the spatial domain but can exhibit sparsity under suitable bases or dictionaries. In our tests, we assume that the test images have a sparse representation under the Discrete Cosine Transform (DCT), i.e., we have
\begin{equation}
	\label{R}
	\boldsymbol{x} = \mathbf{R} \boldsymbol{z},
\end{equation}
where $\boldsymbol{z}$ denotes the test image, $\boldsymbol{x}$ is the sparse DCT coefficients, and $\mathbf{R}$ represents DCT that satisfies
\[
\mathbf{R}^\top\mathbf{R}=\mathbf{R}\mathbf{R}^\top=\mathbf{I}.
\]
Then, we rewrite our forward problem according to \eqref{ip} as
\begin{equation}
	\boldsymbol{y}=\mathbf{K} \mathbf{R}^\top\boldsymbol{x}+\boldsymbol{\epsilon}
\end{equation}
with $\mathbf{F}=\mathbf{K} \mathbf{R}^\top$. 

In \cref{alg}, most of the computational cost lies in the calculation of the main diagonal of $\mathbf{F}^\top\mathsf{S}({\boldsymbol{\gamma}})^{-1}\mathbf{F}$ in the $\bgamma$-update. Applying the Woodbury matrix identity \cite{higham2002accuracy}, we have
\begin{align}
	\mathbf{F}^\top\mathsf{S}({\boldsymbol{\gamma}})^{-1}\mathbf{F}&=\mathbf{F}^{\top}(\sigma^{-1}\mathbf{I}+\mathbf{F}\bGamma\mathbf{F}^\top)^{-1}\mathbf{F} \notag\\
	&=\sigma\mathbf{F}^\top\mathbf{F}-\sigma\mathbf{F}^\top\mathbf{F}(\bGamma^{-1}+\sigma\mathbf{F}^\top\mathbf{F})^{-1}\sigma\mathbf{F}^\top\mathbf{F}. \label{eq:diag}
\end{align} 
Under the symmetric boundary condition, the blurring matrix $\mathbf{K}$ can be diagonalized by DCT \cite{ng1999fast}, i.e., $\mathbf{R}\mathbf{K}\mathbf{R}^\top=\mathbf{\Lambda}$, where $\mathbf{\Lambda}$ is a diagonal matrix. Further, we obtain
\[
\mathbf{F}^{\top}\mathbf{F} = \mathbf{R} (\mathbf{R}^{\top} \mathbf{\Lambda}^2\mathbf{R})\mathbf{R}^{\top} = \mathbf{\Lambda}^2.
\]
Then, according to \eqref{eq:diag} $\mathbf{F}^\top\mathsf{S}({\boldsymbol{\gamma}})^{-1}\mathbf{F}$ becomes a diagonal matrix and can be easily calculated.

In order to prevent unnecessary calculations caused by too small $\gamma_i^{(k)}$, we introduce a threshold $\omega$ in \cref{alg}. When $\gamma_{i}^{(k)}<\omega$, we set $\gamma_{i}^{(k)}=0$. According to \cref{prop:gamma}, we can neglect these zero components from the optimization problem. In all tests, we set $\omega=10^{-16}$. In addition, the initialization of $\bgamma$ is set as the absolute value of the DCT coefficients of the degraded image and the stopping criteria of \cref{alg} are 
$$\frac{\Vert\boldsymbol{x}^{(k+1)}-\boldsymbol{x}^{(k)}\Vert}{\Vert\boldsymbol{x}^{(k)}\Vert}<10^{-8}\quad \mbox{ and }\quad k\leq200.$$
To quantitatively evaluate the performance of different hyperpriors, we use the relative error and the sparsity rate as the main evaluation metrics. The relative error between the image $\hat{\boldsymbol{z}}$ and the ground truth $\boldsymbol{z}$ is defined as $\|\hat{\boldsymbol{z}}-\boldsymbol{z}\|/\|\boldsymbol{z}\|$, and the sparsity rate of $\boldsymbol{x}$ is defined as the ratio of zero elements in $\boldsymbol{x}$ to its total number of elements.

\subsection{Impact of different hyperpriors}\label{sec:num1}
In this subsection, we test the performance and characteristics of EBF with different hyperpriors compared with SBL (without any hyperprior) to illustrate the influence of the hyperpriors on the sparsity. We use the 256-by-256 gray image \textit{Cameraman} as the test image, and the intensity range is $[0,1]$. Since \textit{Cameraman} is not sufficiently sparse under DCT, we truncate its DCT coefficients at 0.025, which gives the sparsity rate as $54.94\%$, and the relative error to the original image is 0.0172. In addition, we set $\sigma_{ker}=1$ and the noise level as $10\%$ for all tests. Figure \ref{test.camera} shows the original image, the compressed image that is used as the ground truth, the degraded image and their DCT coefficient maps. In each map, frequencies increase from top to bottom and left to right.
The absolute value of each DCT coefficient indicates the contribution of the corresponding frequency component.

\begin{figure}[t]
	\label{test.camera}
	\centering
	\subfigure
	{
		\begin{minipage}[b]{.3\linewidth}
			\centering
			\includegraphics[width=\linewidth,trim=280 0 280 0,clip]{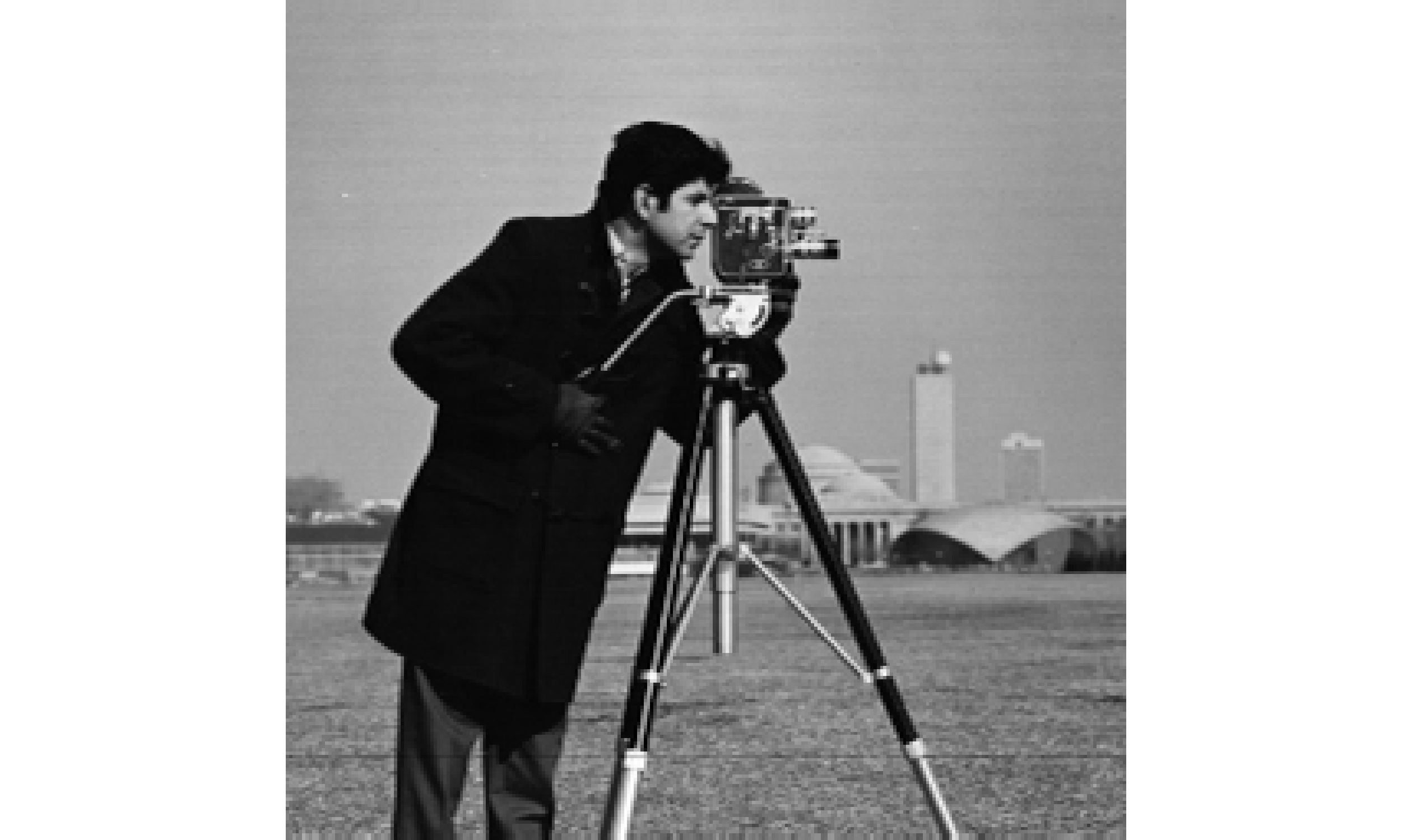}
		\end{minipage}
		\begin{minipage}[b]{.3\linewidth}
			\centering
			\includegraphics[width=\linewidth,trim=280 0 280 0,clip]{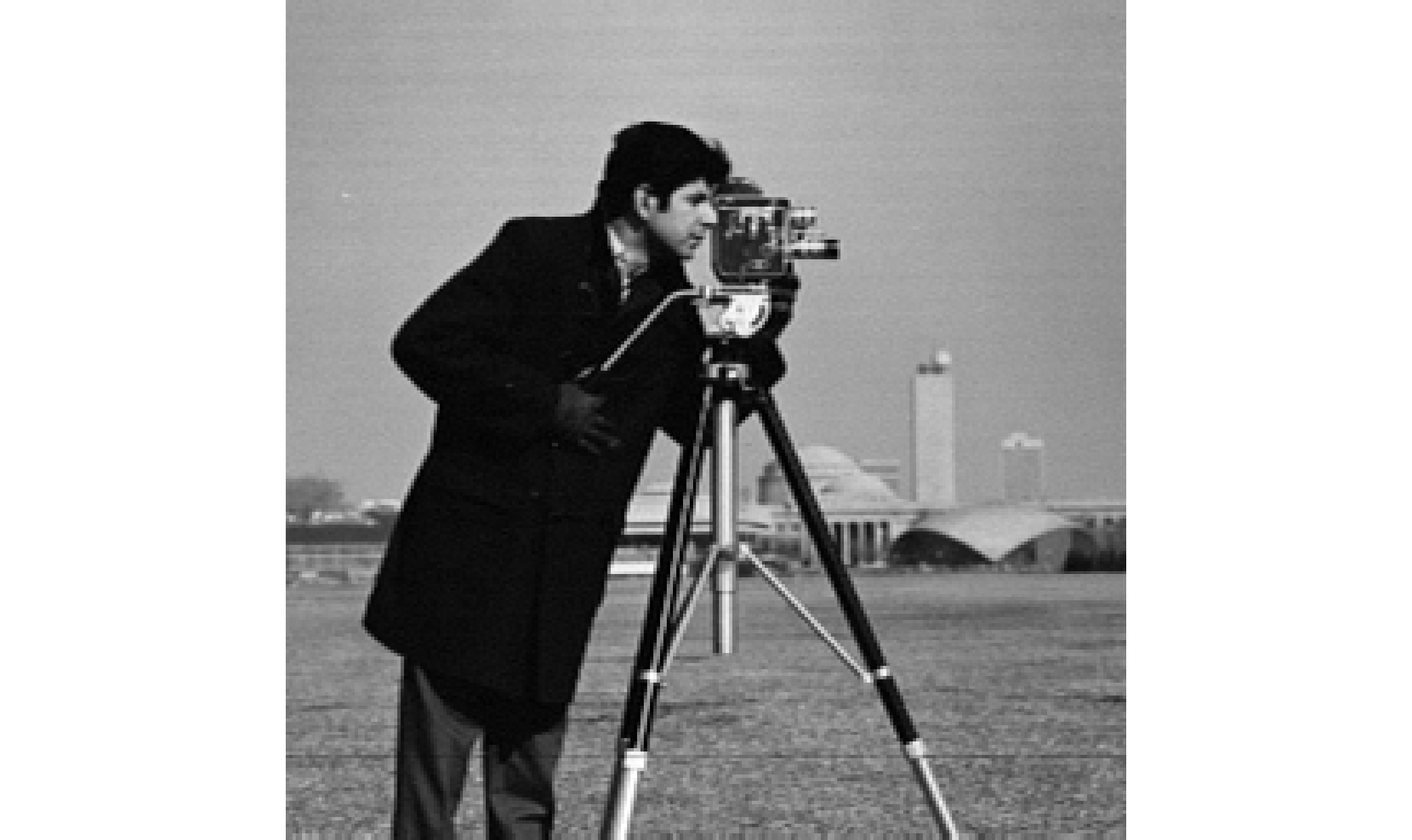}
		\end{minipage}
		\begin{minipage}[b]{.3\linewidth}
			\centering
			\includegraphics[width=\linewidth,trim=280 0 280 0,clip]{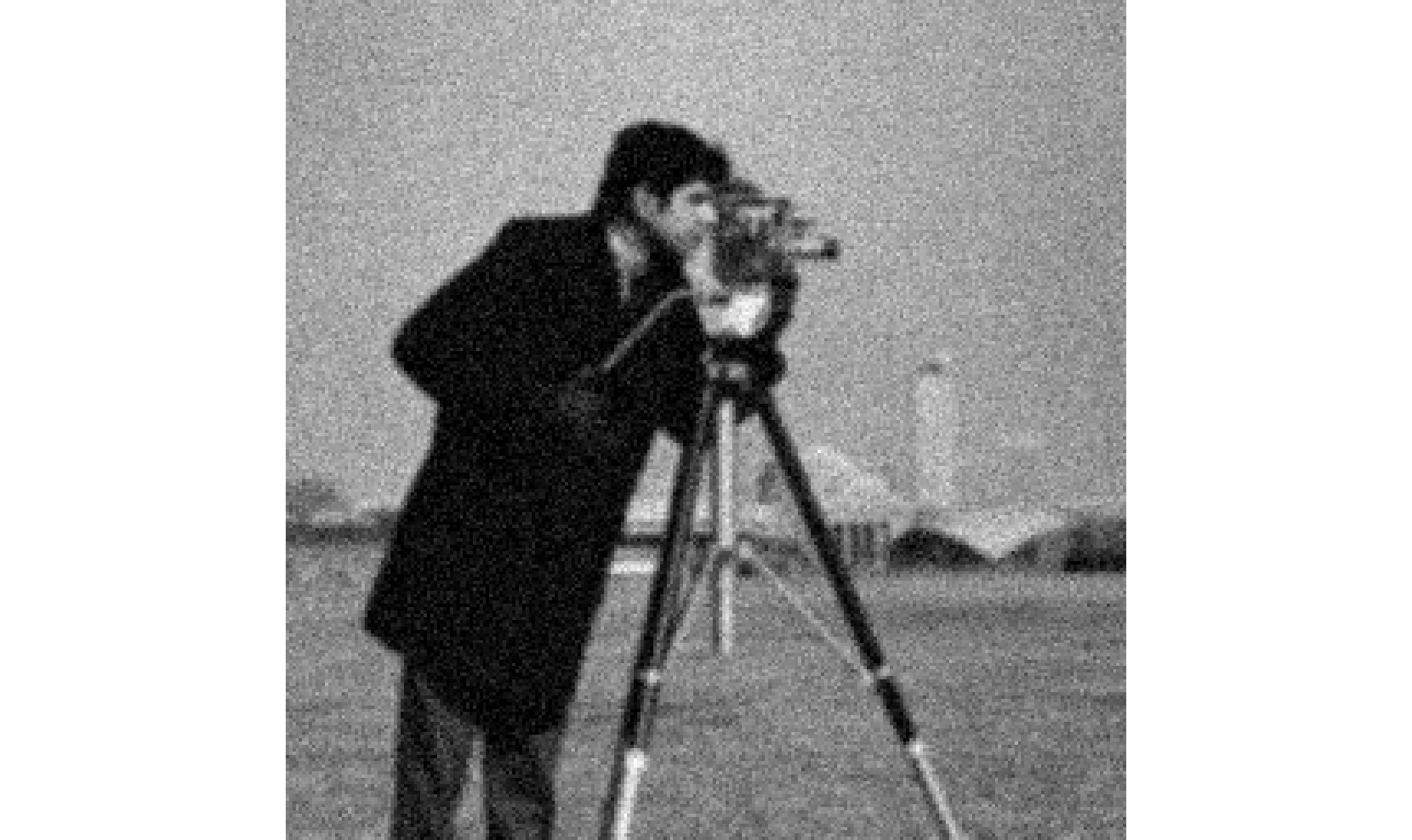}
		\end{minipage}
		\hspace{0.01\linewidth}
	}
	\subfigure
	{
		\begin{minipage}[b]{.3\linewidth}
			\centering
			\includegraphics[width=\linewidth,trim=280 0 280 0,clip]{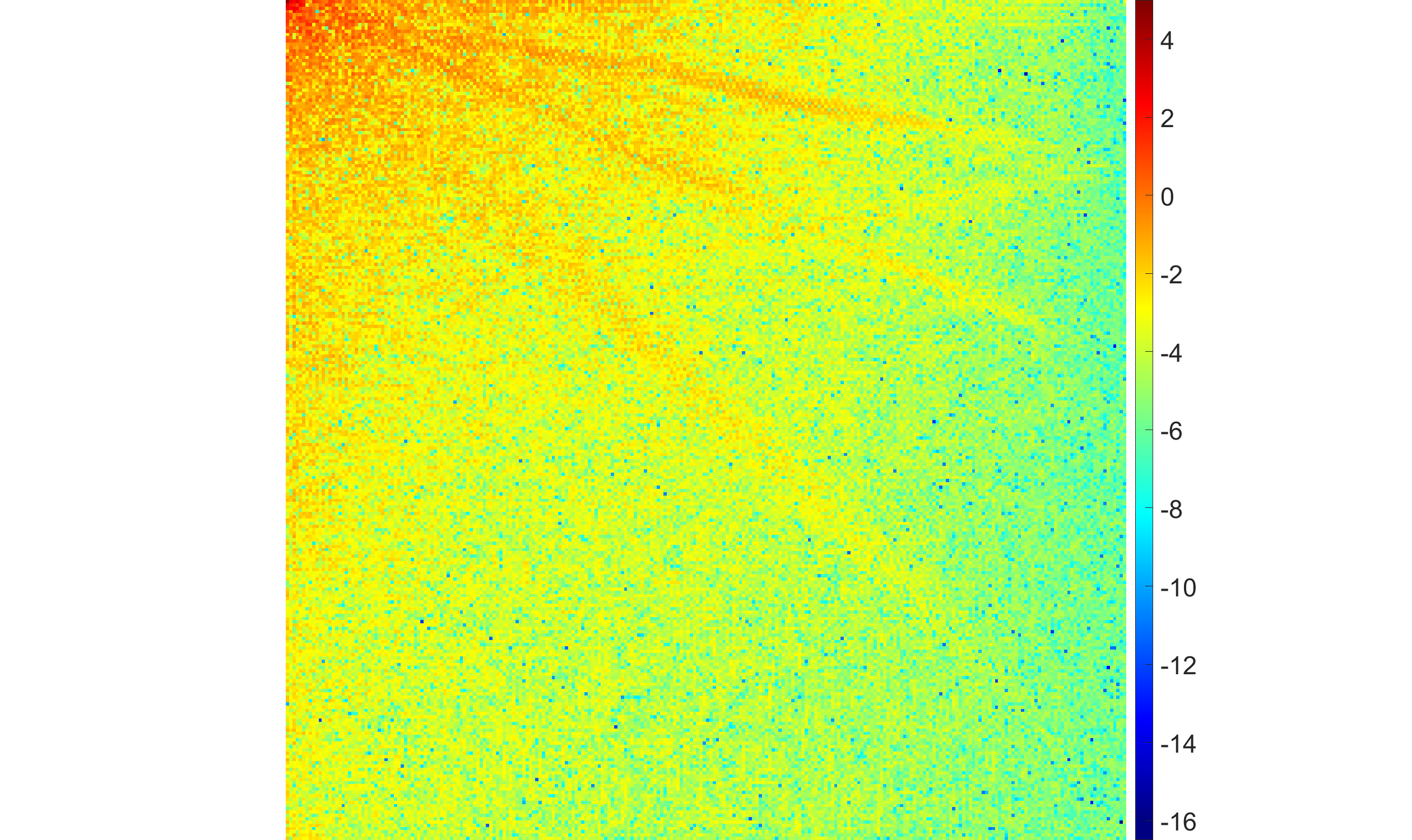}\\
			Original image
		\end{minipage}
		\begin{minipage}[b]{.3\linewidth}
			\centering
			\includegraphics[width=\linewidth,trim=280 0 280 0,clip]{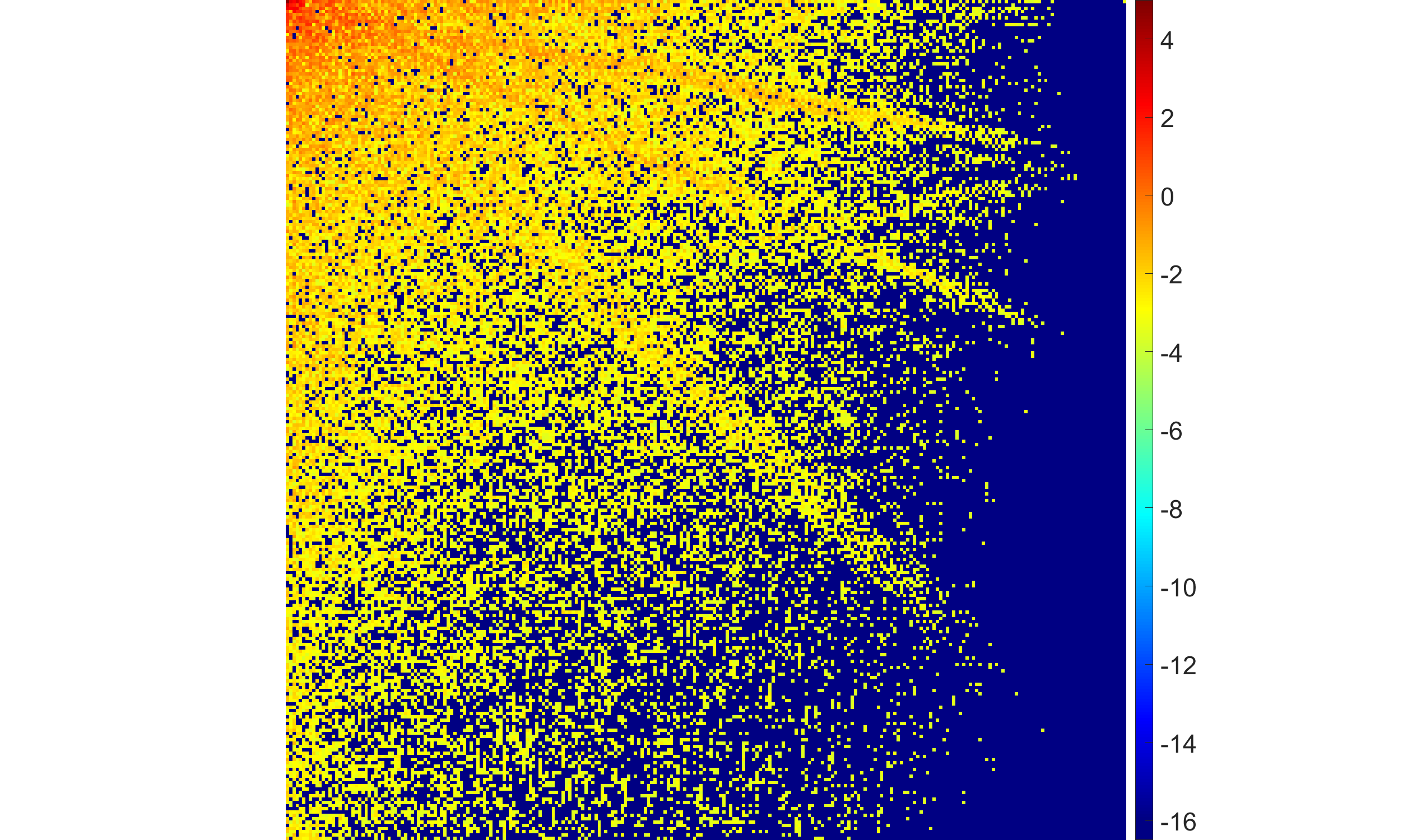}\\
			Compressed image
		\end{minipage}
		\begin{minipage}[b]{.3\linewidth}
			\centering
			\includegraphics[width=\linewidth,trim=280 0 280 0,clip]{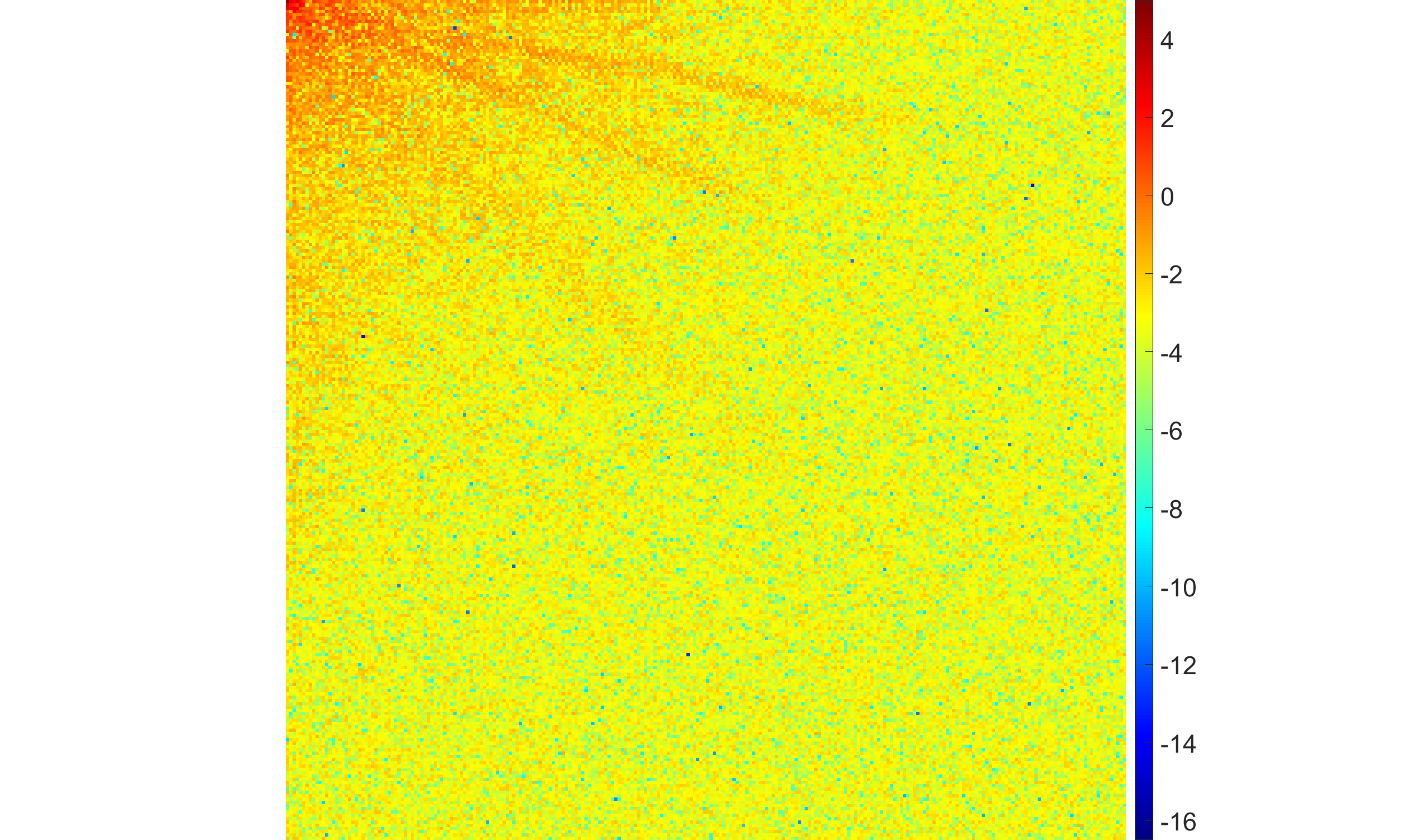}\\
			Degraded image
		\end{minipage}
		\hspace{0.01\linewidth}
	}
	\caption{The \textit{Cameraman} image. Top row: test images; Bottom row: absolute values of DCT coefficients.} 
\end{figure}

\subsubsection{Comparison of the half-Laplace and half-Gaussian hyperpriors}

In this numerical test, we illustrate the influence of using the half-Laplace and half-Gaussian hyperpriors on the solutions of \eqref{model:F}, and further compare the effect of these two hyperpriors with the one without hyperprior, i.e., SBL. To ensure that the strengths of two hyperpriors are comparable, we set $\beta = \theta=0.1$.

\begin{figure}[t]
	\label{fig1.gau_gam}
	\centering
	\subfigure
	{
		\begin{minipage}[b]{.3\linewidth}
			\centering
			\includegraphics[width=\linewidth,trim=280 0 280 0,clip]{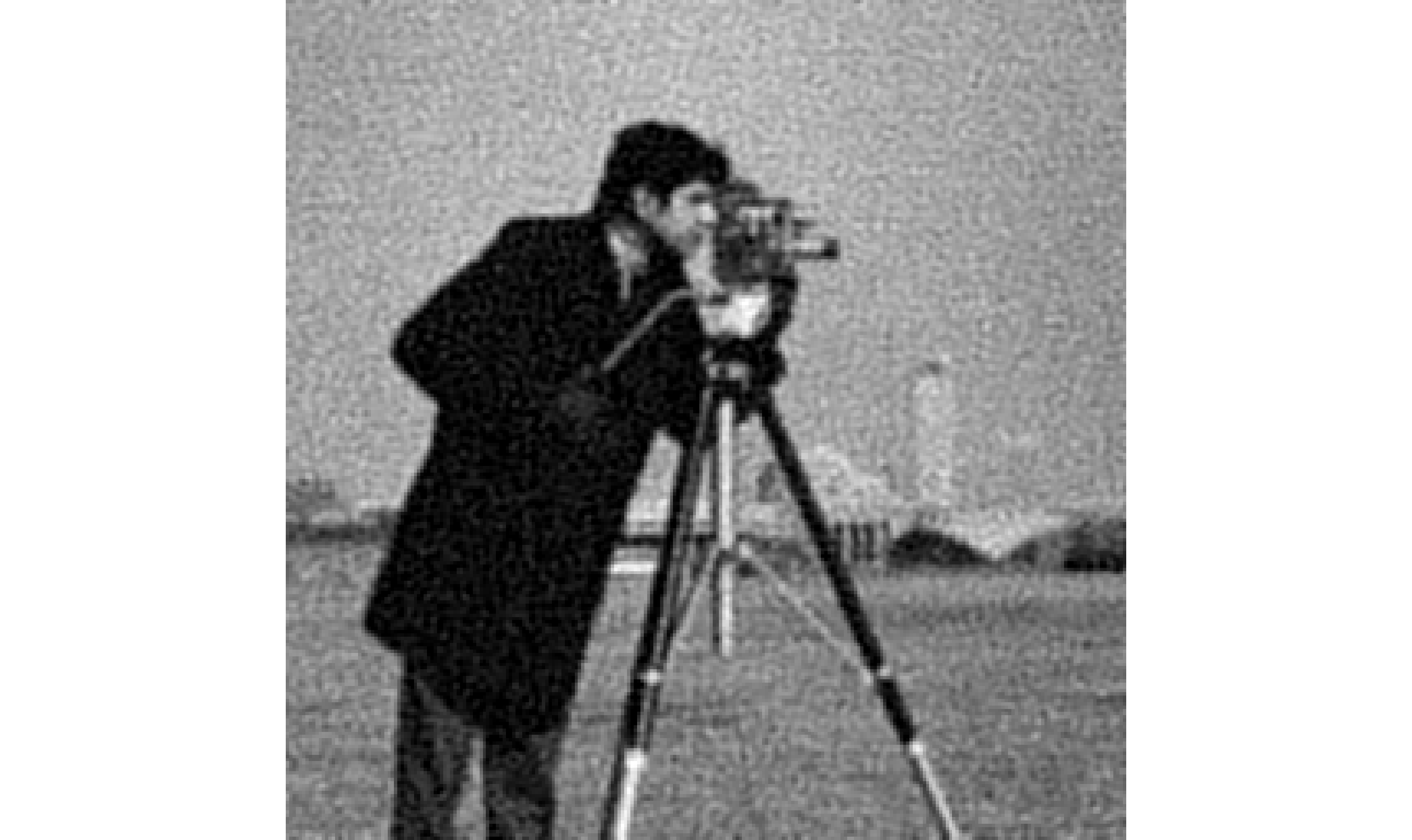}
		\end{minipage}
		\begin{minipage}[b]{.3\linewidth}
			\centering
			\includegraphics[width=\linewidth,trim=280 0 280 0,clip]{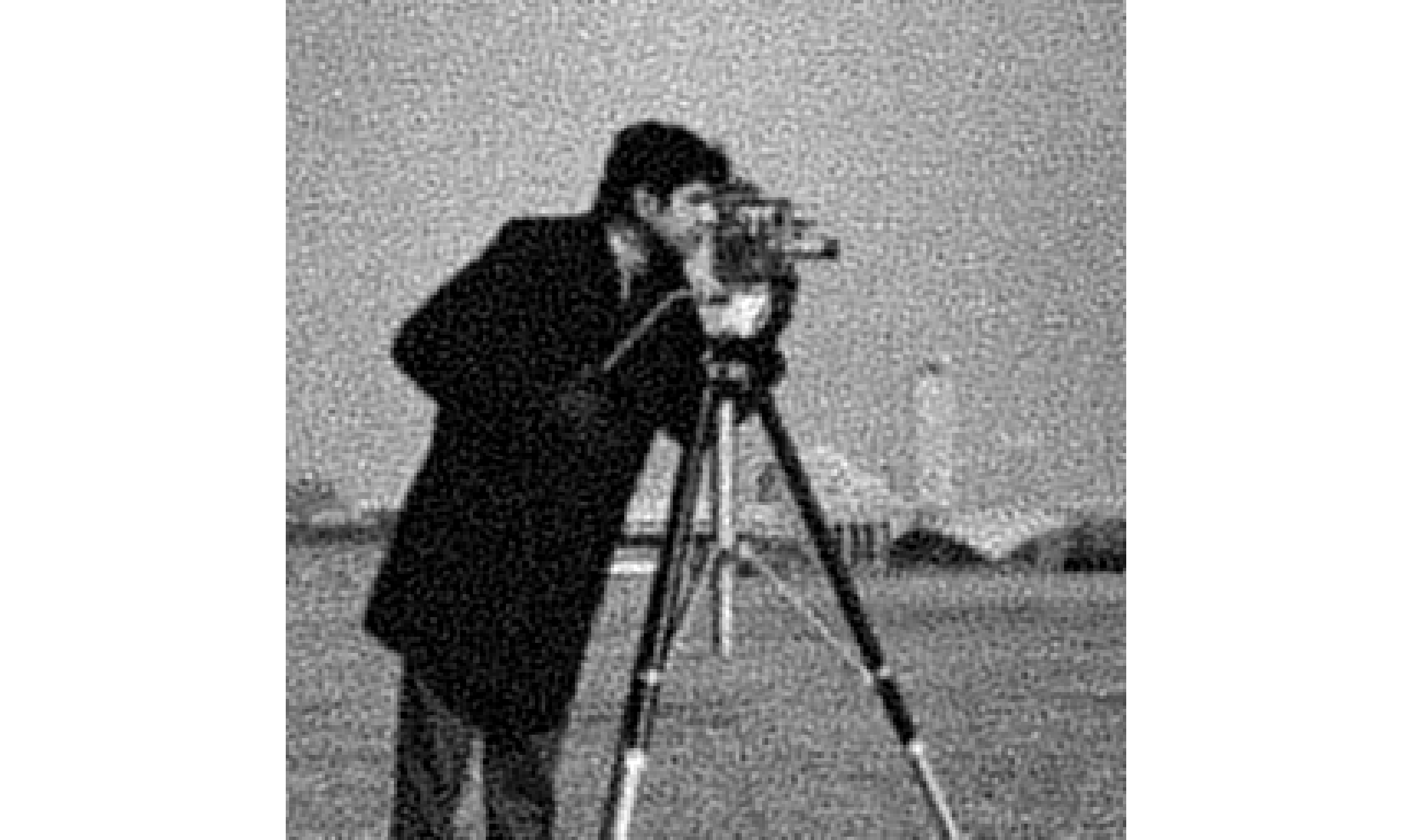}
		\end{minipage}
		\begin{minipage}[b]{.3\linewidth}
			\centering
			\includegraphics[width=\linewidth,trim=280 0 280 0,clip]{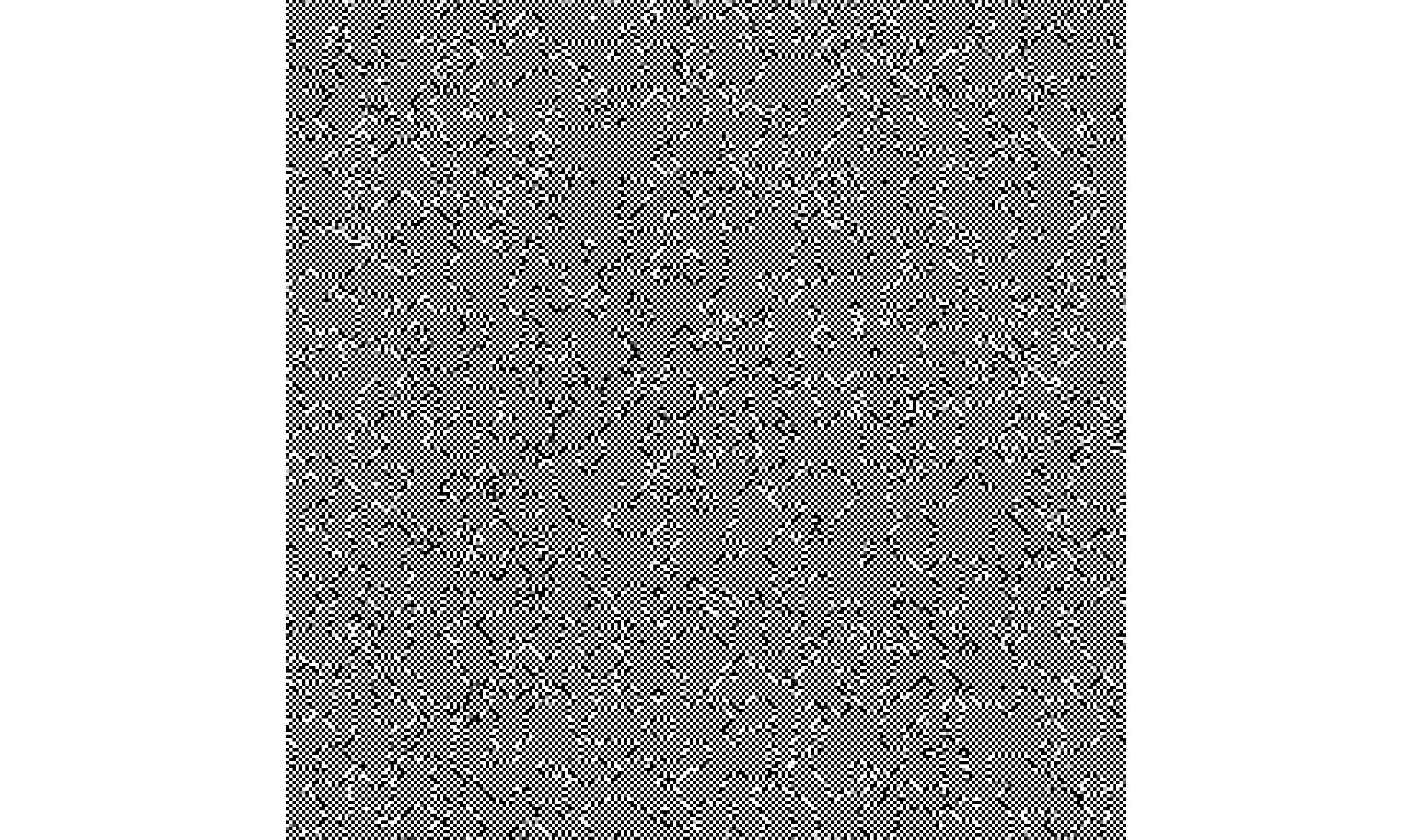}
		\end{minipage}
		\hspace{0.01\linewidth}
	}
	\subfigure
	{
		\begin{minipage}[b]{.3\linewidth}
			\centering
			\includegraphics[width=\linewidth,trim=280 0 280 0,clip]{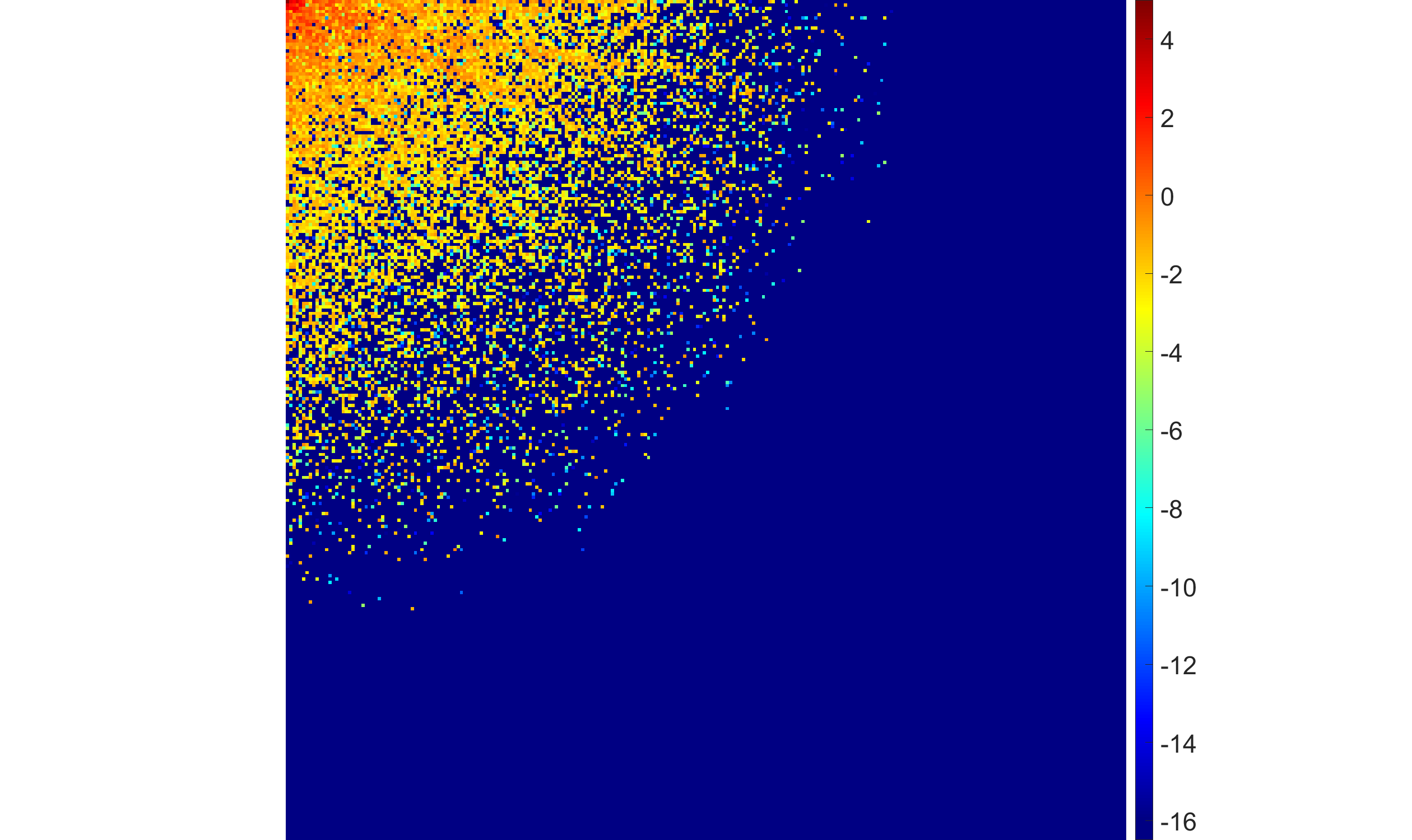}\\
			Half-Laplace
		\end{minipage}
		
		\begin{minipage}[b]{.3\linewidth}
			\centering
			\includegraphics[width=\linewidth,trim=280 0 280 0,clip]{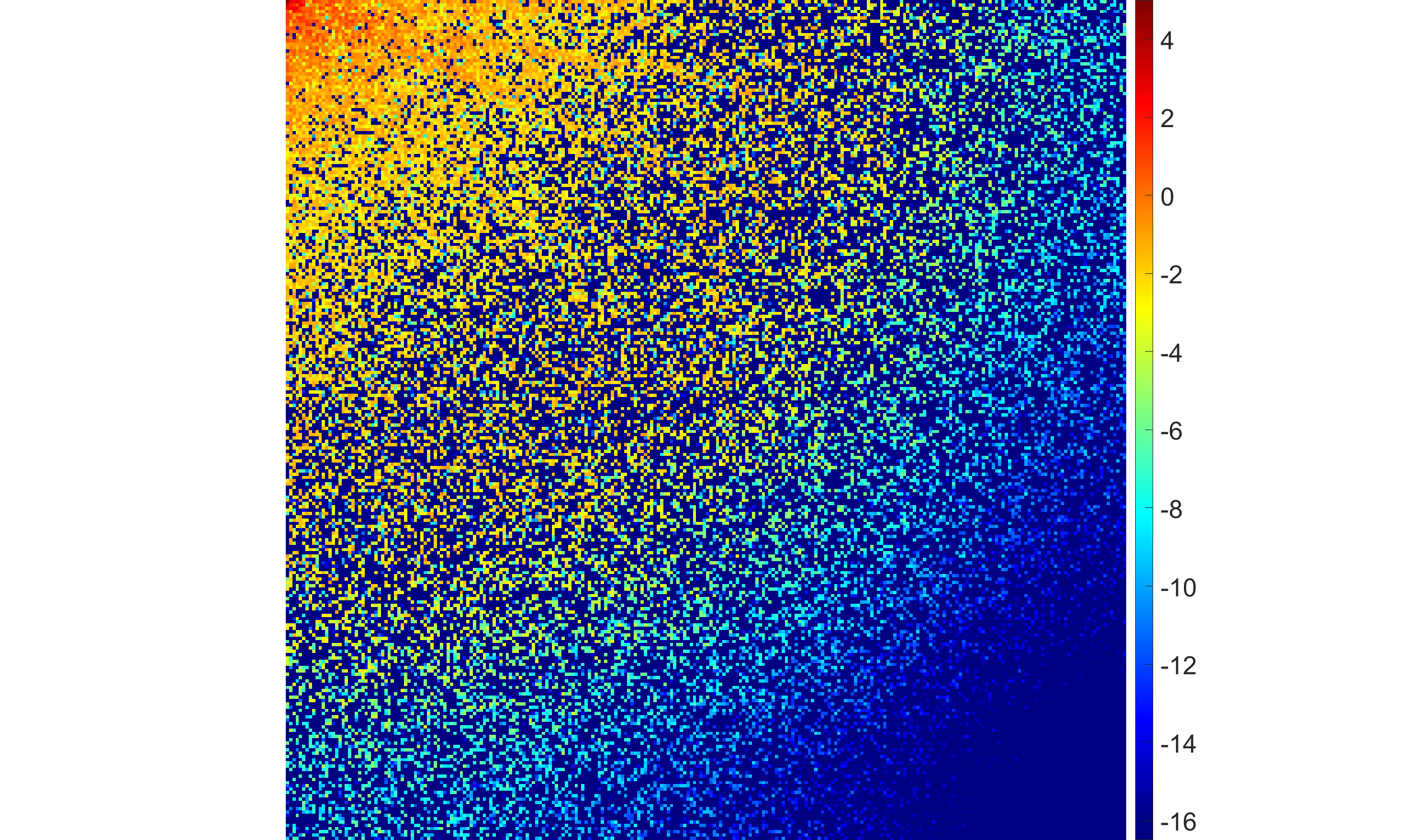}\\
			Half-Gaussian
		\end{minipage}
		
		\begin{minipage}[b]{.3\linewidth}
			\centering
			\includegraphics[width=\linewidth,trim=280 0 280 0,clip]{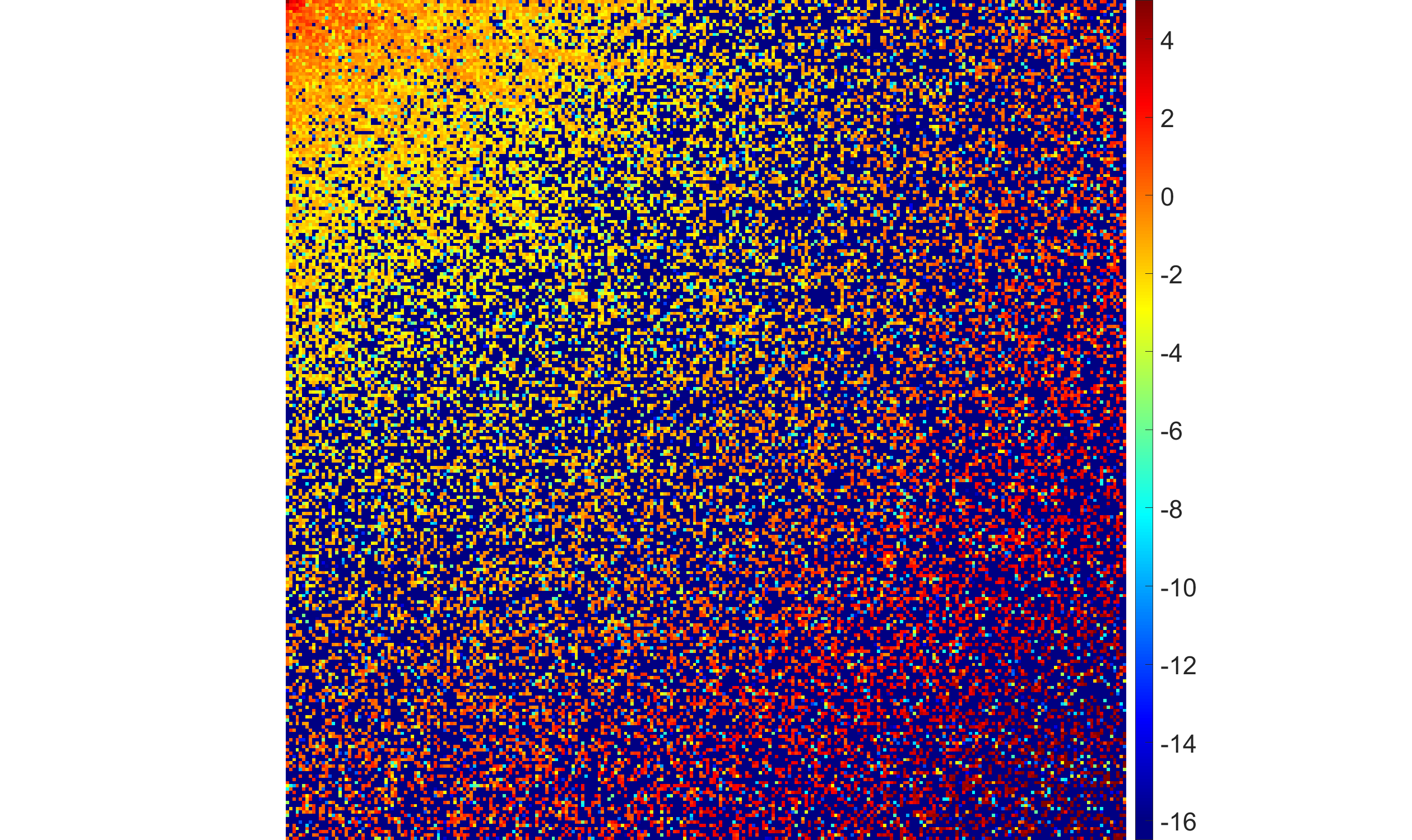}\\
			SBL
		\end{minipage}
		\hspace{0.01\linewidth}
	}
	\caption{Comparisons of the results using the half-Laplace, half-Gaussian and no hyperprior.
		Top row: restored images; Bottom row: absolute values of DCT coefficients.}
\end{figure}

In Figure \ref{fig1.gau_gam}, it is obvious that the half-Laplace hyperprior significantly promotes the sparsity of the DCT coefficients, making the sector-shaped region of non-zero coefficients more compact, which leads to fewer artifacts in the restoration. The half-Gaussian hyperprior produces a smoother decay of the coefficients from low frequency to high frequency, but keeps many small non-zero components, resulting in more artifacts in the restoration compared with the half-Laplace result. In contrast, the SBL method without any hyperprior exhibits numerous scattered abnormal peaks in the coefficient map, leading to high-frequency noise and resulting in obvious speckles and artifacts in the restoration. According to the relative errors listed in \cref{tab1}, we can see that the half-Laplace hyperprior provides the best results quantitively. Comparing the sparsity level, i.e. the percentage of the zero elements in $\boldsymbol{x}$, in \cref{tab1} we conclude that the half-Laplace hyperprior can promote the sparsity. The half-Gaussian hyperprior can stabilize the problem by reducing the contributions of high-frequency components, but cannot promote sparsity.  

\subsubsection{Influence of the half-generalized Gaussian hyperpriors}

To investigate the influence of the half-generalized Gaussian hyperprior on the sparsity of the DCT coefficients, we test it with different choice of $p$, which controls the sharpness and tail thickness of the hyperprior. When $p = 1$, the hyperprior reduces to the half-Laplace hyperprior. When $p = 2$, it becomes the half-Gaussian hyperprior. To obtain a prior that strongly promotes sparsity, we focus on the case $0 < p \leq 1$. In this test, we set $\beta=0.1$. Figure \ref{fig2.lp} shows the restored images and the DCT coefficients with $p=0.5$, $0.75$ and $1$. 

\begin{figure}[t]\label{fig2.lp}
	\centering
	\subfigure
	{
		\begin{minipage}[b]{.3\linewidth}
			\centering
			\includegraphics[width=\linewidth,trim=280 0 280 0,clip]{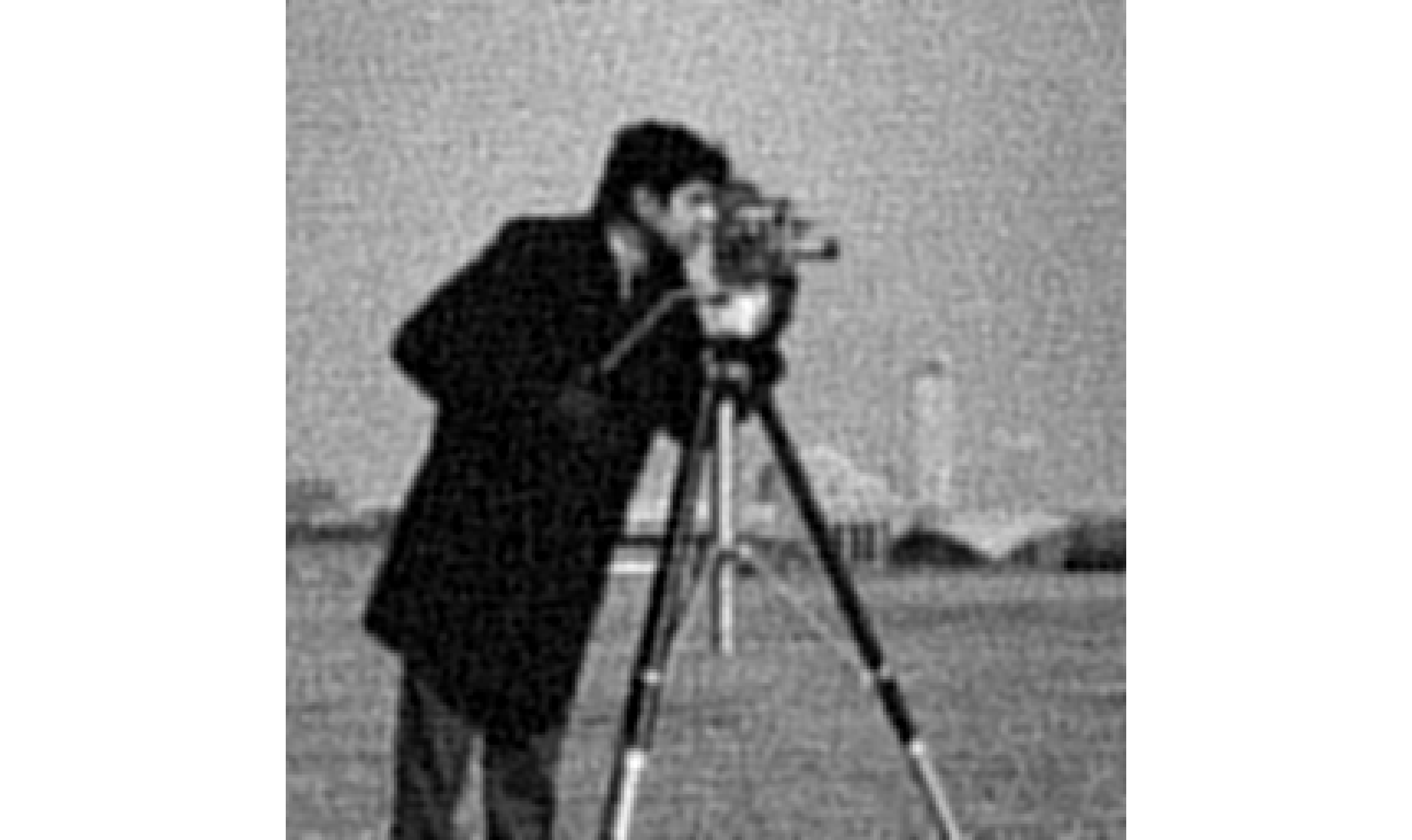}
		\end{minipage}
		\begin{minipage}[b]{.3\linewidth}
			\centering
			\includegraphics[width=\linewidth,trim=280 0 280 0,clip]{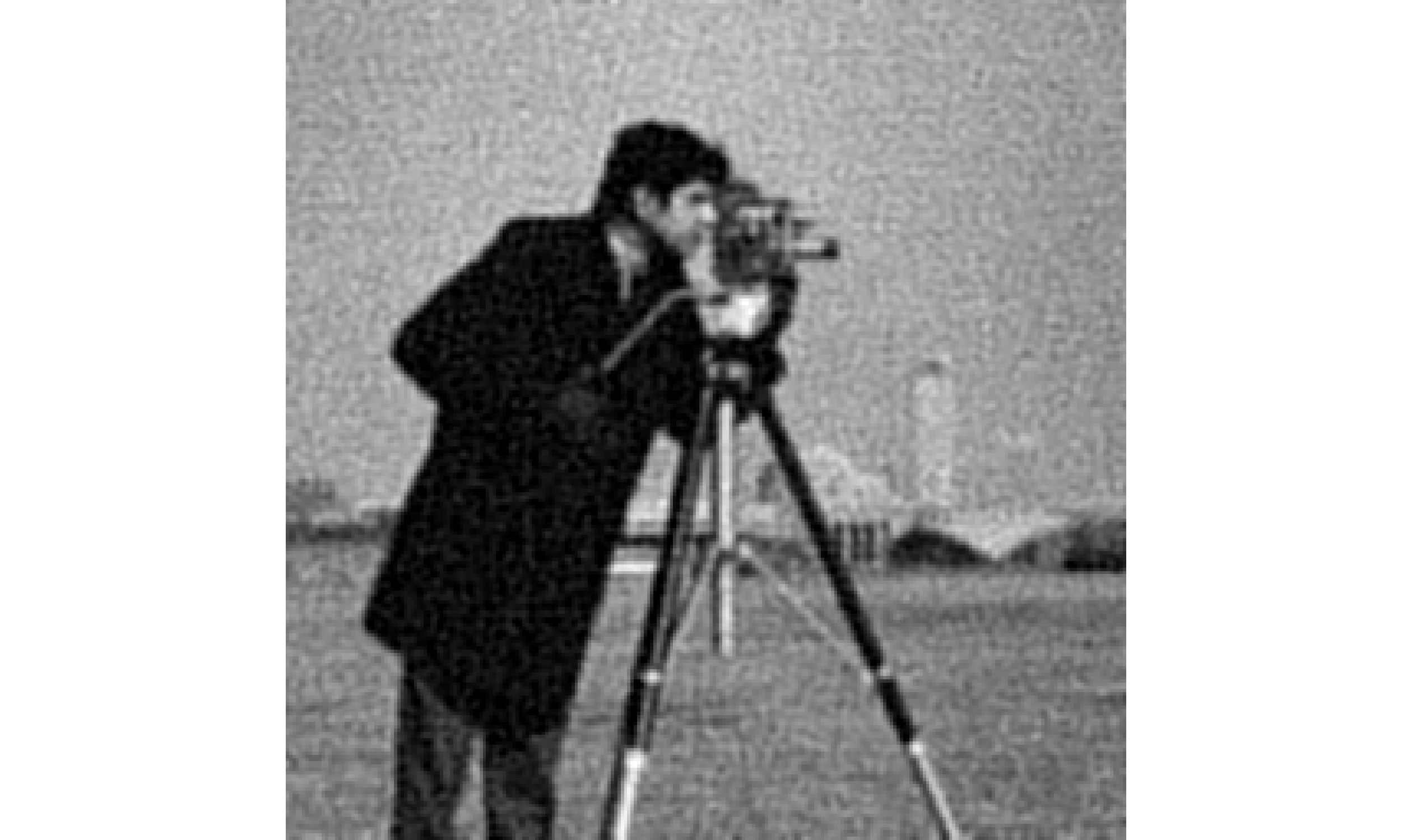}
		\end{minipage}
		\begin{minipage}[b]{.3\linewidth}
			\centering
			\includegraphics[width=\linewidth,trim=280 0 280 0,clip]{l1x.png}
		\end{minipage}
		\hspace{0.01\linewidth}
	}
	\subfigure
	{
		\begin{minipage}[b]{.3\linewidth}
			\centering
			\includegraphics[width=\linewidth,trim=280 0 280 0,clip]{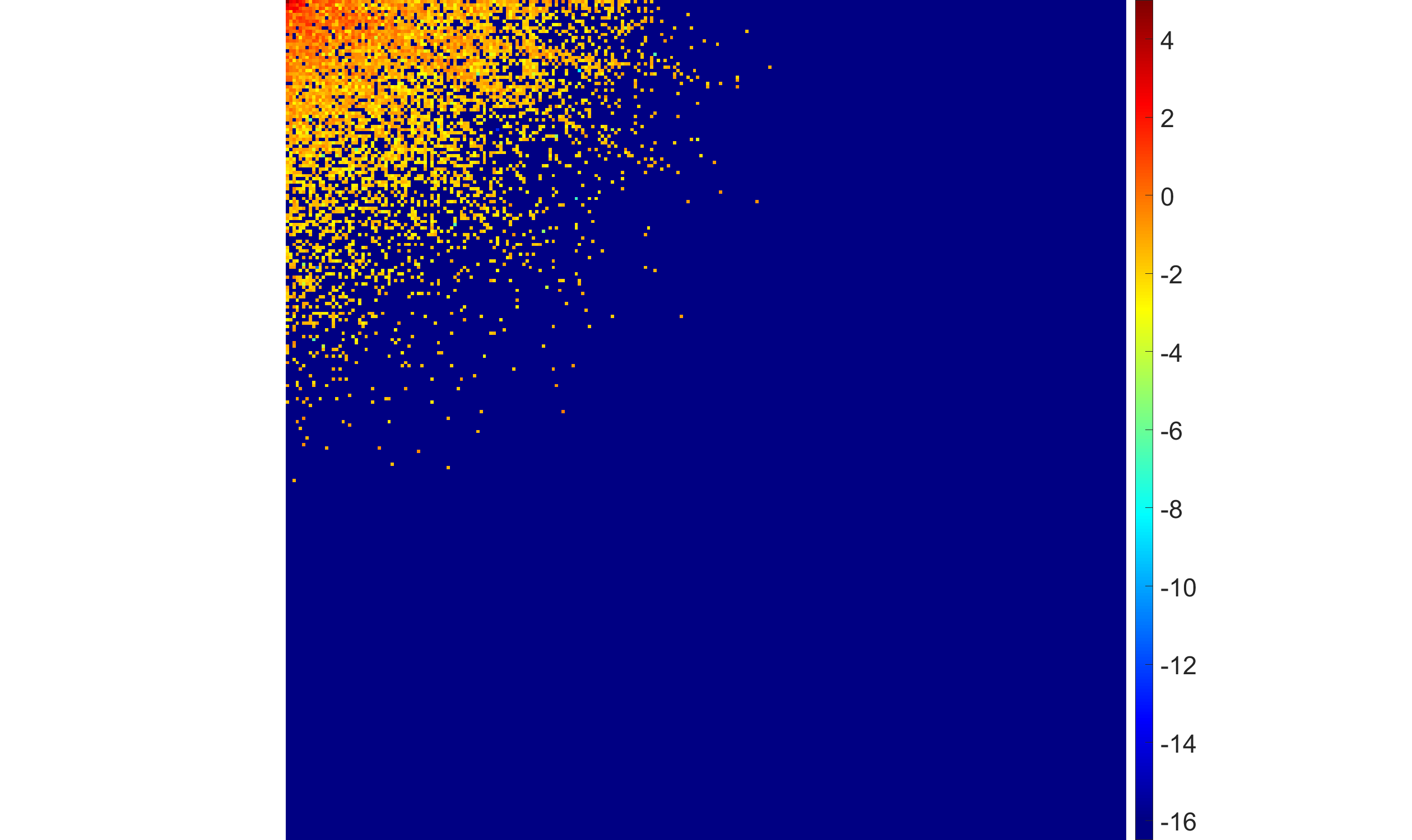}\\
			$p=0.5$
		\end{minipage}
		%\hspace{0.005\linewidth}
		\begin{minipage}[b]{.3\linewidth}
			\centering
			\includegraphics[width=\linewidth,trim=280 0 280 0,clip]{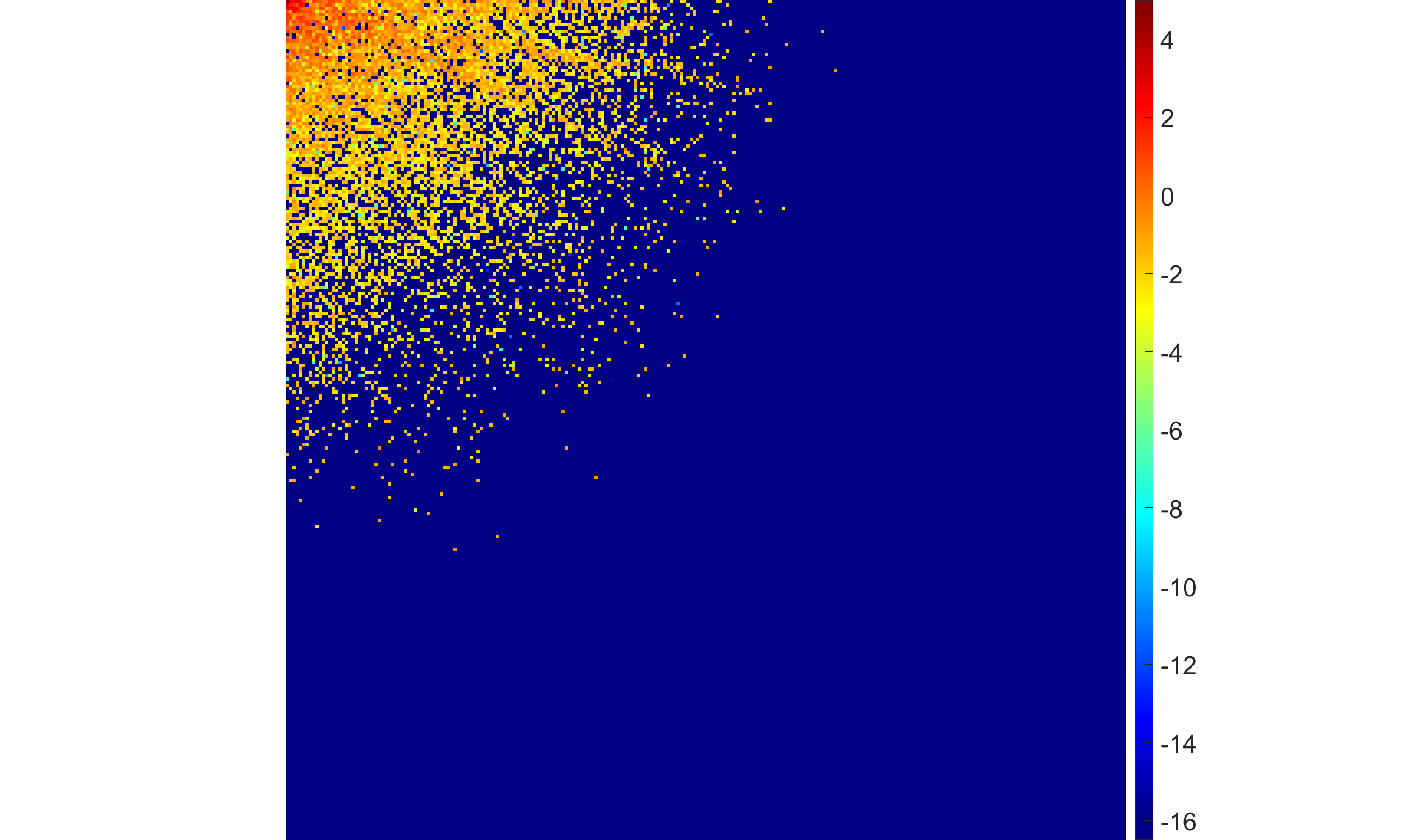}\\
			$p=0.75$
		\end{minipage}
		%\hspace{0.005\linewidth}
		\begin{minipage}[b]{.3\linewidth}
			\centering
			\includegraphics[width=\linewidth,trim=280 0 280 0,clip]{l1c.png}\\
			$p=1$ (half-Laplace)
		\end{minipage}
		\hspace{0.01\linewidth}
	}
	\caption{Comparisons of the results using the half-generalized Gaussian hyperpriors with different $p$.
		Top row: restored images; Bottom row: absolute values of DCT coefficients.}
\end{figure}

In the half-generalized Gaussian distribution, the smaller the $p$-value, the sharper the peak of the distribution. Therefore, the half-generalized Gaussian hyperprior with a smaller $p$ has a stronger capability to promote sparsity. This property can be easily observed by comparing the DCT coefficients shown in Figure \ref{fig2.lp}, where more and more high-frequency coefficients are driven to zero as $p$ decreasing. The sparsity levels reported in \cref{tab1} also
confirm this trend. However, setting \( p \) too small (e.g., \( p = 0.5 \)) can lead to oversmoothing and loss of image detail. As shown in the figure, the restoration for \( p = 0.5 \) appears overly smooth compared to the results with larger \( p \).

Although the half-generalized Gaussian hyperpriors with $0<p<1$ can offer excellent sparse solutions, they correspond to a nonconvex $\mh$, making analysis and computation more challenging. In contrast, with the half-Laplace hyperprior, i.e., $p=1$, we can obtain the similar restoration quality and sparsity level, and the corresponding $\mh$ is convex. Therefore, in \cref{sec:num2} we choose the half-Laplace hyperprior to illustrate the performance of EBF with respect to the ill-posedness and noise. 

\begin{table}[t]
	\centering
	\caption{Restoration performance of half-generalized Gaussian hyperpriors with different $p$.}
	\label{tab1}
	\resizebox{\textwidth}{!}{
	\begin{tabular}{lccccc}
		\toprule
		\multirow{2}{*}{~} 
		& \multicolumn{5}{c}{$\ell_p$ Hyperprior} \\ 
		\cmidrule(lr){2-6}
		& None (SBL) & $p=2$ (Gaussian) & $p=1$ (Laplace) & $p=0.75$ & $p=0.5$\\ 
		\midrule
		Relative error & 0.5246 & 0.1279 & 0.1055 & 0.1017 & 0.1041 \\
		Sparsity (\%)  & 55.95 & 56.36 & 84.73 & 90.32 & 93.25 \\ 
		\bottomrule
	\end{tabular}}
\end{table}

\subsubsection{Influence of the Gamma hyperprior}
Before discussing the performance of EBF equipped with half-Laplace, we test EBFs with the Gamma hyperprior. According to \cref{example:Gamma}, we know that the relationship between $\alpha$ and $1$ determines the sparsity-promoting effect of the Gamma hyperprior. In this test, we fix $\beta=0.1$ and show how the hyperparameter $\alpha$ affects the sparsity of the solution. Figure \ref{fig3.gamma} shows the restored images and the corresponding DCT coefficient maps for $\alpha=0.5$, $1$ and $1.5$, respectively.

\begin{figure}[t]\label{fig3.gamma}
	\centering
	\subfigure
	{
		\begin{minipage}[b]{.3\linewidth}
			\centering
			\includegraphics[width=\linewidth,trim=280 0 280 0,clip]{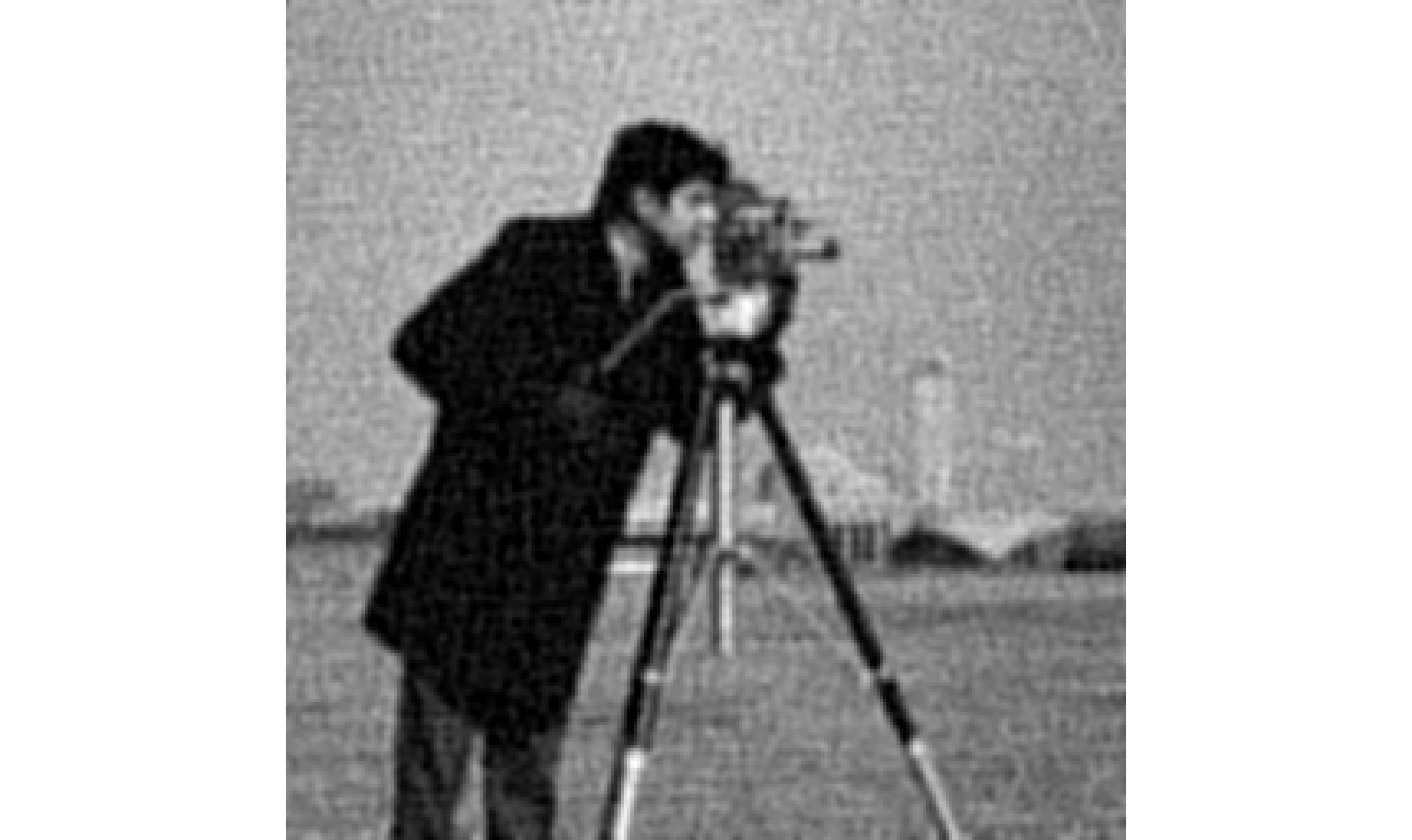}
		\end{minipage}
		\begin{minipage}[b]{.3\linewidth}
			\centering
			\includegraphics[width=\linewidth,trim=280 0 280 0,clip]{l1x.png}
		\end{minipage}
		\begin{minipage}[b]{.3\linewidth}
			\centering
			\includegraphics[width=\linewidth,trim=280 0 280 0,clip]{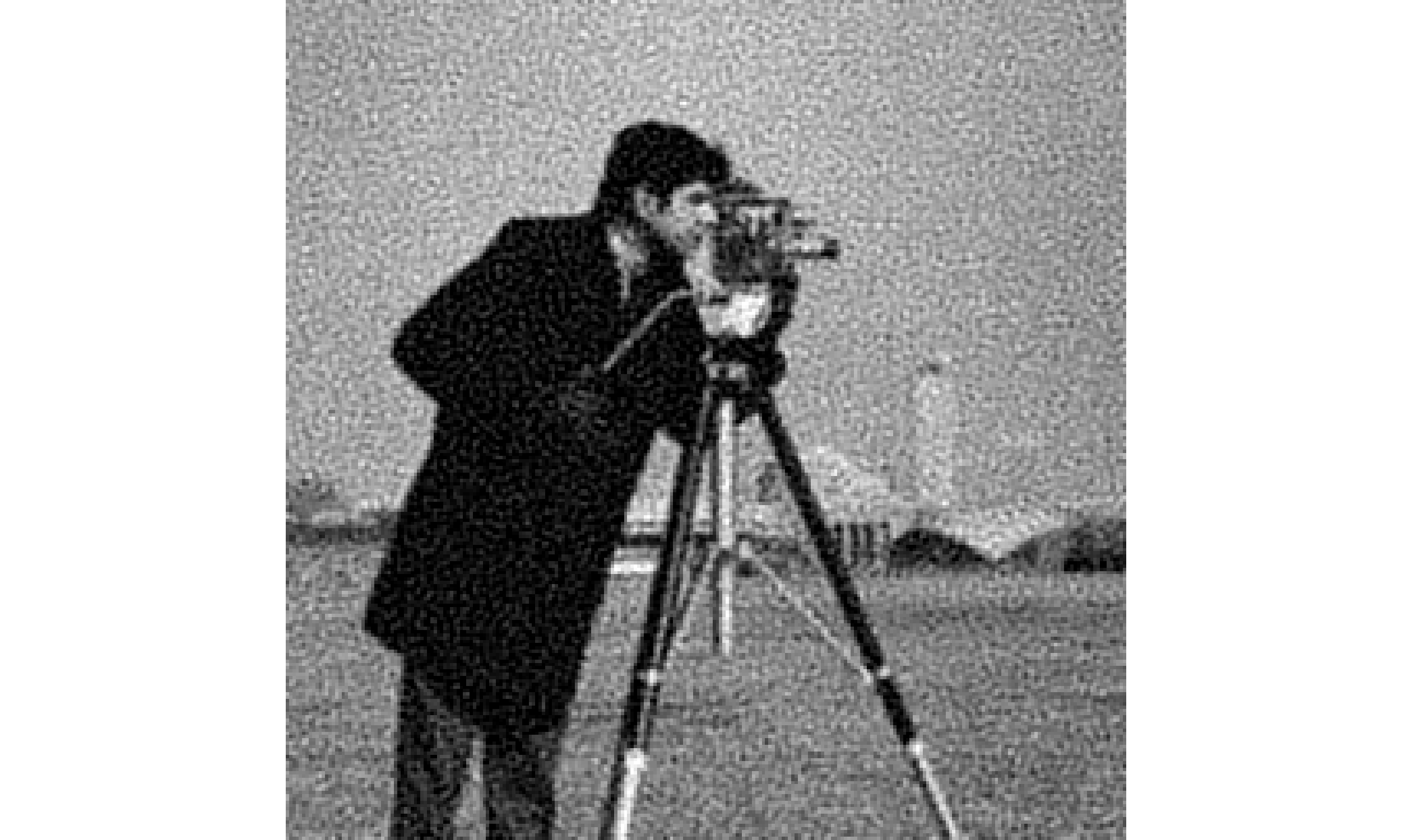}
		\end{minipage}
		\hspace{0.01\linewidth}
	}
	\subfigure
	{
		\begin{minipage}[b]{.3\linewidth}
			\centering
			\includegraphics[width=\linewidth,trim=280 0 280 0,clip]{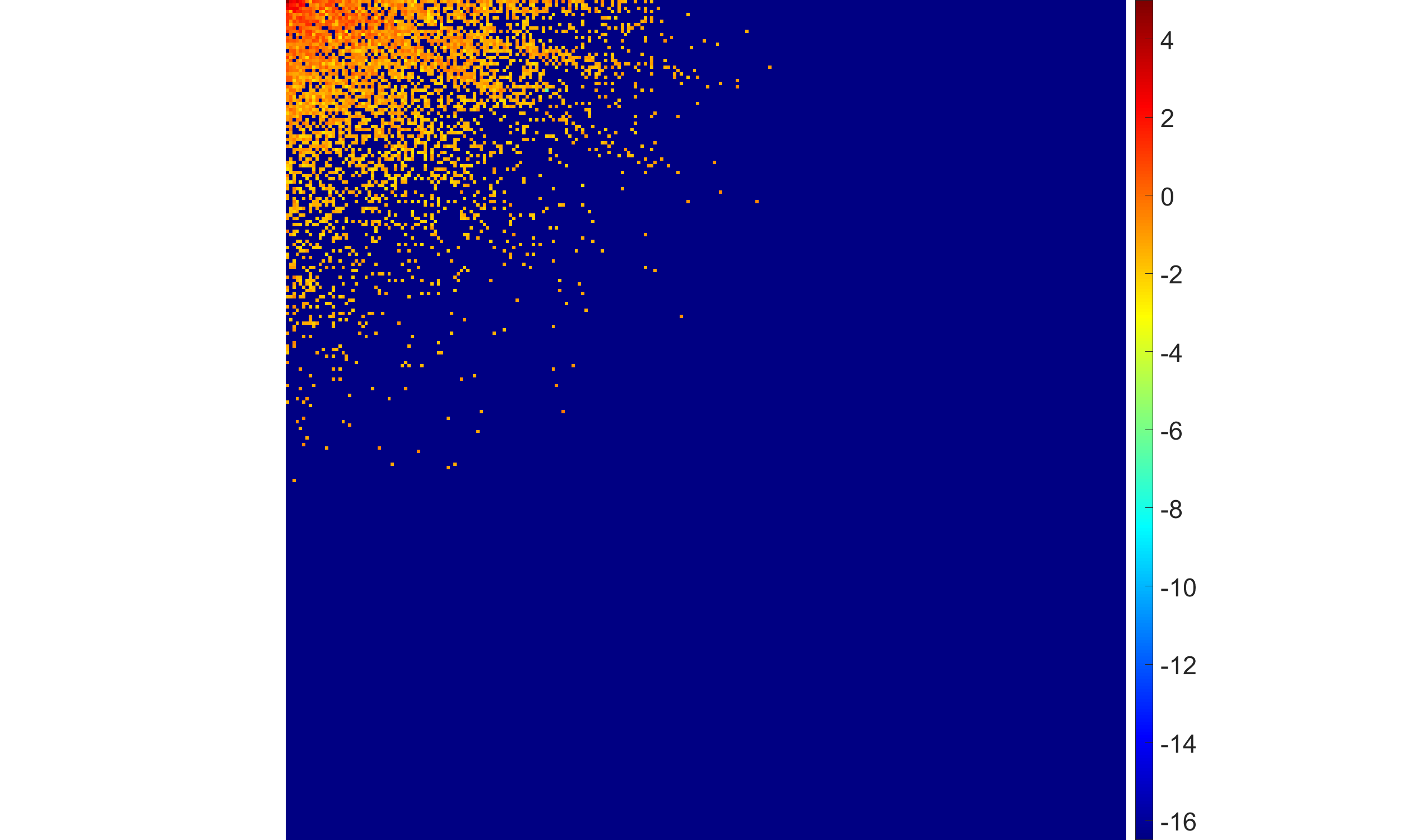}\\
			$\alpha=0.5$
		\end{minipage}
		%\hspace{0.005\linewidth}
		\begin{minipage}[b]{.3\linewidth}
			\centering
			\includegraphics[width=\linewidth,trim=280 0 280 0,clip]{l1c.png}\\
			$\alpha=1$ (half-Laplace)
		\end{minipage}
		%\hspace{0.005\linewidth}
		\begin{minipage}[b]{.3\linewidth}
			\centering
			\includegraphics[width=\linewidth,trim=280 0 280 0,clip]{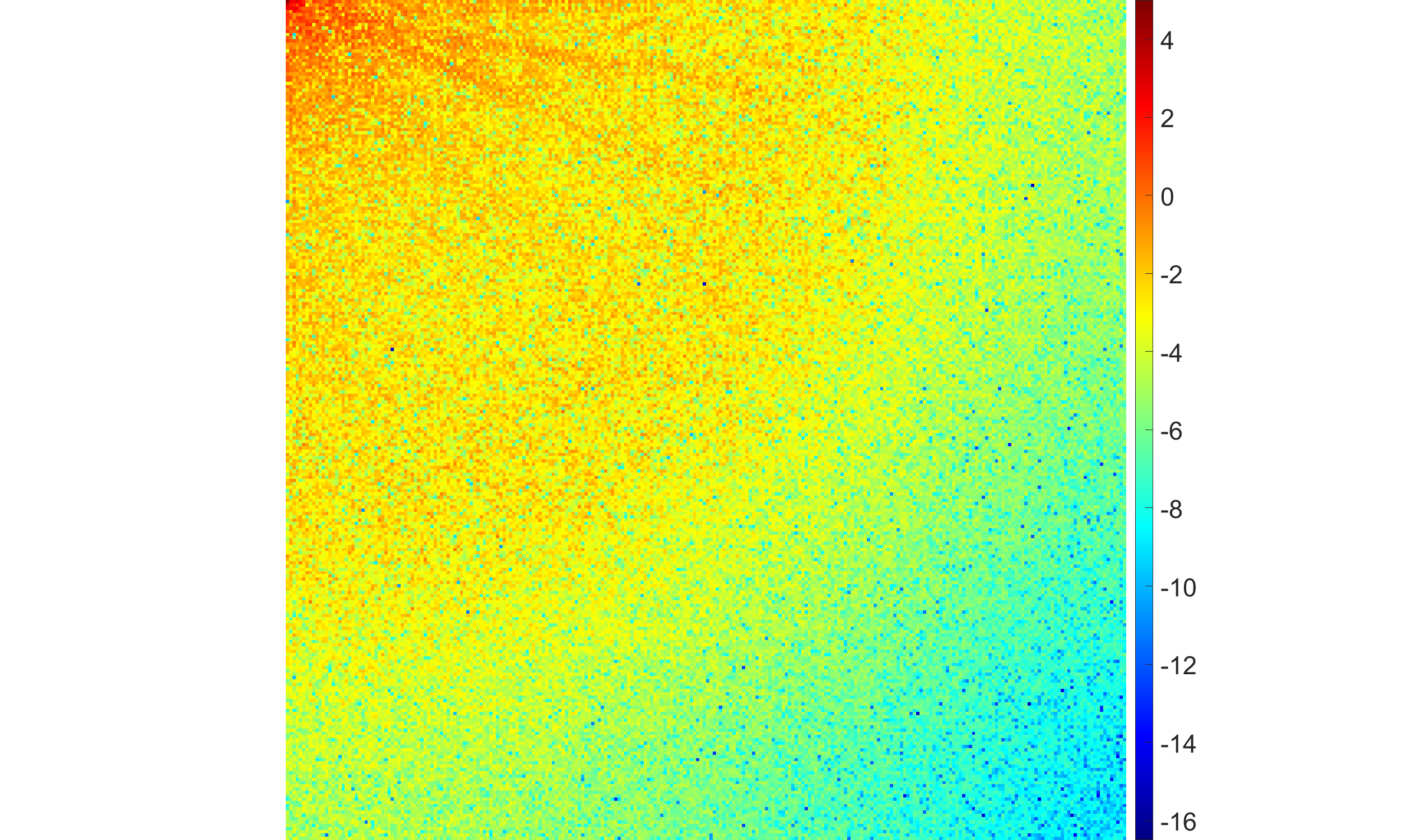}\\
			$\alpha=1.5$
		\end{minipage}
		\hspace{0.01\linewidth}
	}
	\caption{Comparisons of the results using the Gamma hyperprior with different $\alpha$.
		Top row: restored images with relative errors of $0.1084$, $0.1055$ and $0.1423$, respectively; Bottom row: absolute values of DCT coefficients with sparsity rates of $94,74\%$, $84.73\%$ and $0\%$, respectively.}
\end{figure}

In the Gamma hyperprior, when $0<\alpha<1$, $\boldsymbol{\gamma}^\star=0$ is always a global minimizer. Therefore, the Gamma hyperprior is a strongly sparsity-promoting prior in this case, thereby pushing all $\gamma_i$ values towards $0$. However, there are still quite a few non-zeros $\gamma_i$ in the low-frequency region in Figure \ref{fig3.gamma} because the algorithm converges to other KKT points instead of the global minimizer. When $\alpha>1$, many mid- and high-frequency coefficients that should have been suppressed remain relatively large magnitudes. It can be observed from both Figure \ref{fig3.gamma} and its $0\%$ sparsity rate that the Gamma hyperprior in this case completely destroys the sparse structure with a positive $\bgamma$, but it can still stabilize the problem. The case $\alpha = 1$, corresponding to the half-Laplace hyperprior, has been analyzed in previous tests.

\subsection{Influence of ill-posedness and noise}\label{sec:num2}

In \cref{sec:num1}, we have compared the influence of different hyperpriors applied in EBF on sparsity and restoration. In this subsection, we will use the half-Laplace hyperprior as an example to study the performance of EBF under different levels of ill-posedness and noise. In our tests, besides \textit{Cameraman} we include 256-by-256 gray image \textit{House} as an extra test image. To ensure the sparsity in its DCT coefficients, we pre-process the image similarly by truncating its DCT coefficients with the threshold $0.025$. It gives the sparsity rate as $75.18\%$, and the relative error between the compressed and original images is 0.0157. In Figure \ref{fig4.testhouse} we show the compressed \textit{House} image with its DCT coefficient map.

\begin{figure}[t]
	\label{fig4.testhouse}
	\centering
	%\subfigure
	%{
		%    \begin{minipage}[b]{.3\linewidth}
			%        \centering
			%        \includegraphics[width=\linewidth,trim=280 0 280 0,clip]{picture/illposedness/origin house x.png}
			%    \end{minipage}
		\begin{minipage}[b]{.3\linewidth}
			\centering
			\includegraphics[width=\linewidth,trim=280 0 280 0,clip]{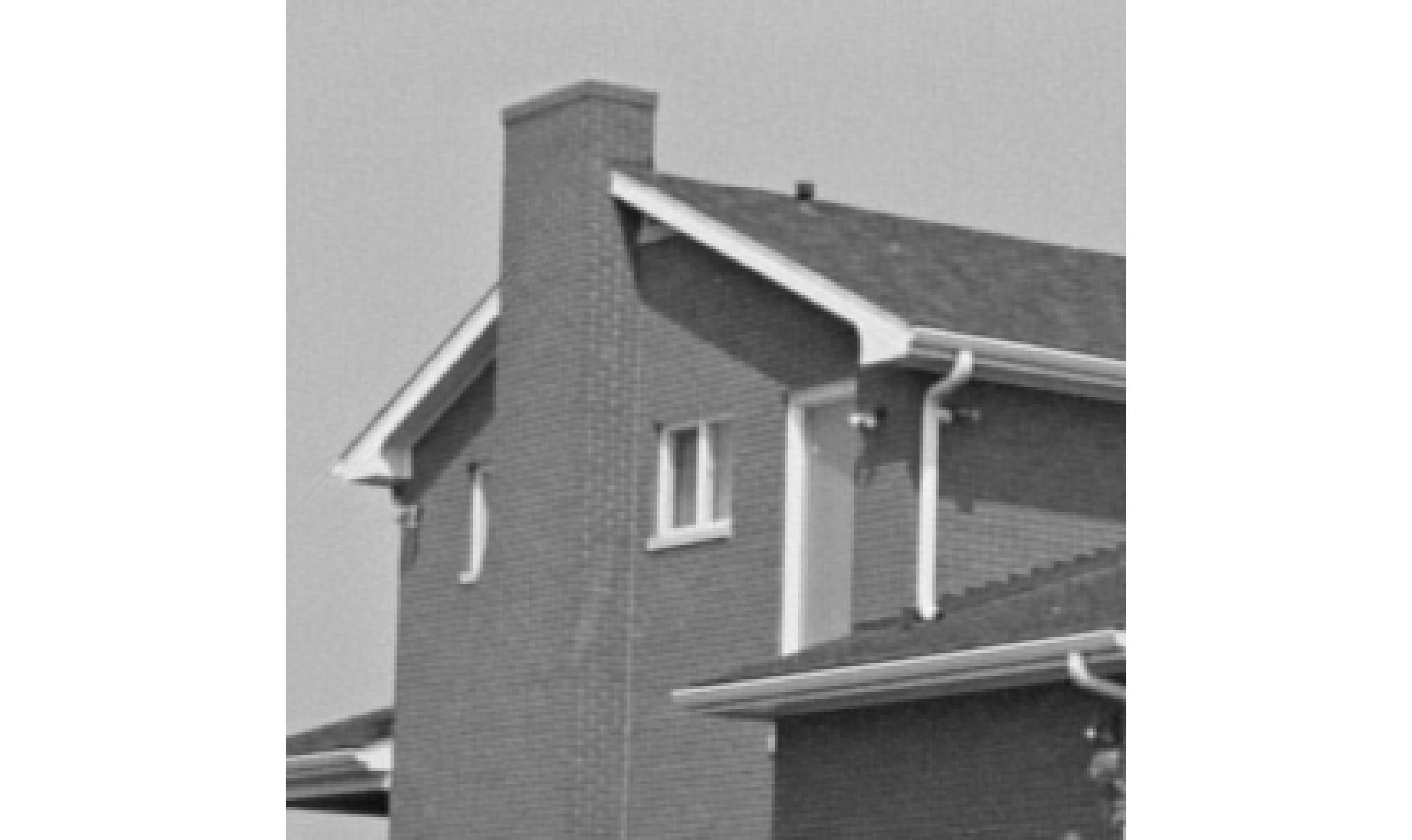}
			%\\Compressed image
		\end{minipage}
		\begin{minipage}[b]{.3\linewidth}
			\centering
			\includegraphics[width=\linewidth,trim=280 0 280 0,clip]{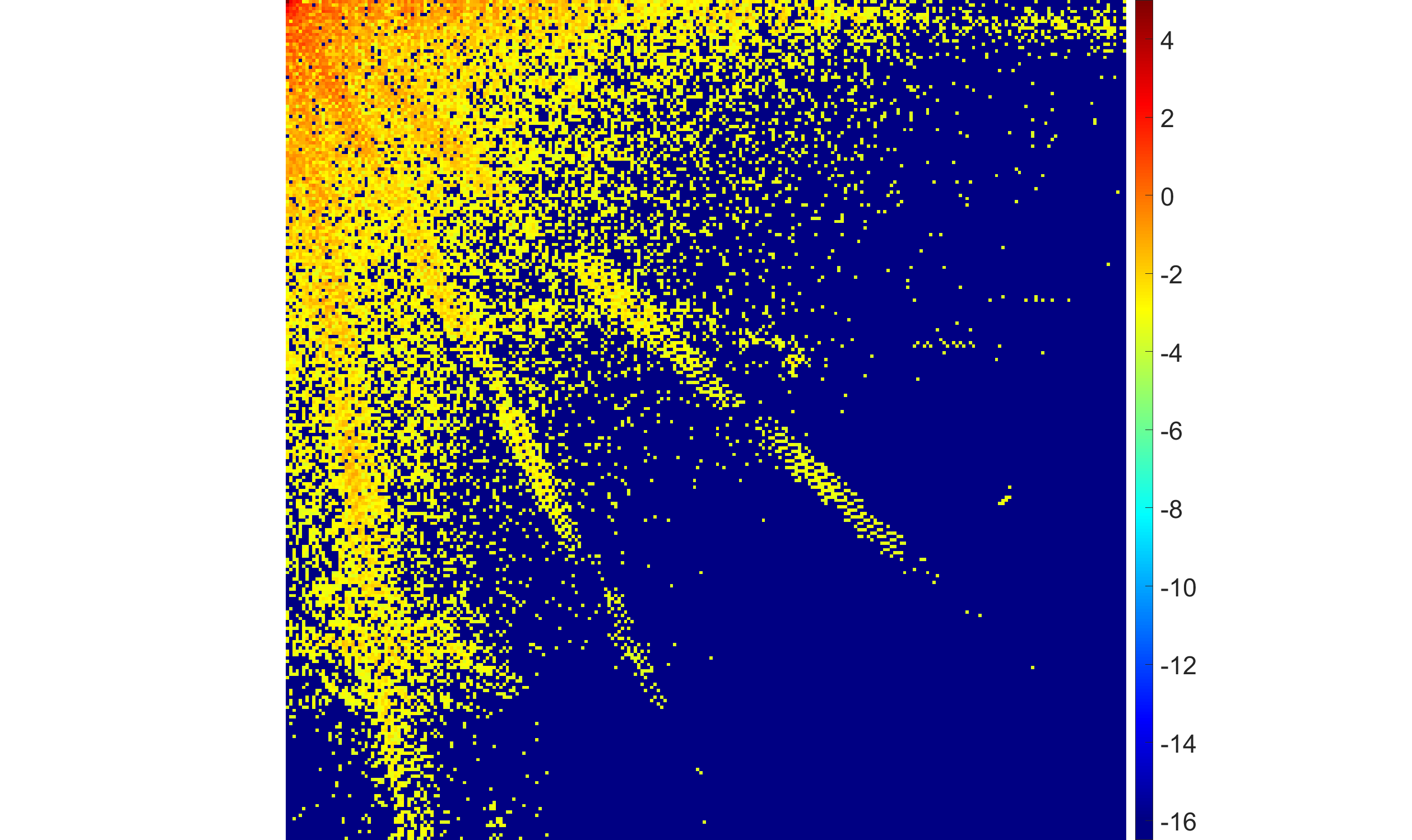}
			%\\DCT Coefficients
		\end{minipage}
		\hspace{0.01\linewidth}
		%}
	%\subfigure
	%{
		%    \begin{minipage}[b]{.3\linewidth}
			%        \centering
			%        \includegraphics[width=\linewidth,trim=280 0 280 0,clip]{picture/illposedness/origin house c.png}
			%    \end{minipage}
		%    \begin{minipage}[b]{.3\linewidth}
			%        \centering
			%        \includegraphics[width=\linewidth,trim=280 0 280 0,clip]{picture/illposedness/test house c.png}
			%    \end{minipage}
		%    \hspace{0.01\linewidth}
		%}
	\caption{The compressed \textit{House} image and its DCT coefficients.}
\end{figure}

\subsubsection{Ill-posedness}
In all tests, we consider the Gaussian blur. The standard deviation $\sigma_{ker}$ of the Gaussian blurring kernel controls the ill-posedness of the inverse problem. A larger $\sigma_{ker}$ results in stronger smoothing and faster decay of the singular values of the forward operator, thereby making the inverse problem more ill-posed. In this test, we vary the ill-posedness of the problem by adopting different values of $\sigma_{ker}$. We set $\sigma_{ker}=0.5$, $1$ and $1.5$. In the half-Laplace hyperprior we still set $\beta$ as $0.1$. To reduce the interference of noise and ensure that the differences in performance in the tests are mainly caused by the variation in ill-posedness, we set a rather low noise level as $5\%$.

\begin{figure}[t]
	\label{fig4.restore}
	\centering
	\subfigure
	{
		\begin{minipage}[b]{.3\linewidth}
			\centering
			\includegraphics[width=\linewidth,trim=280 0 280 0,clip]{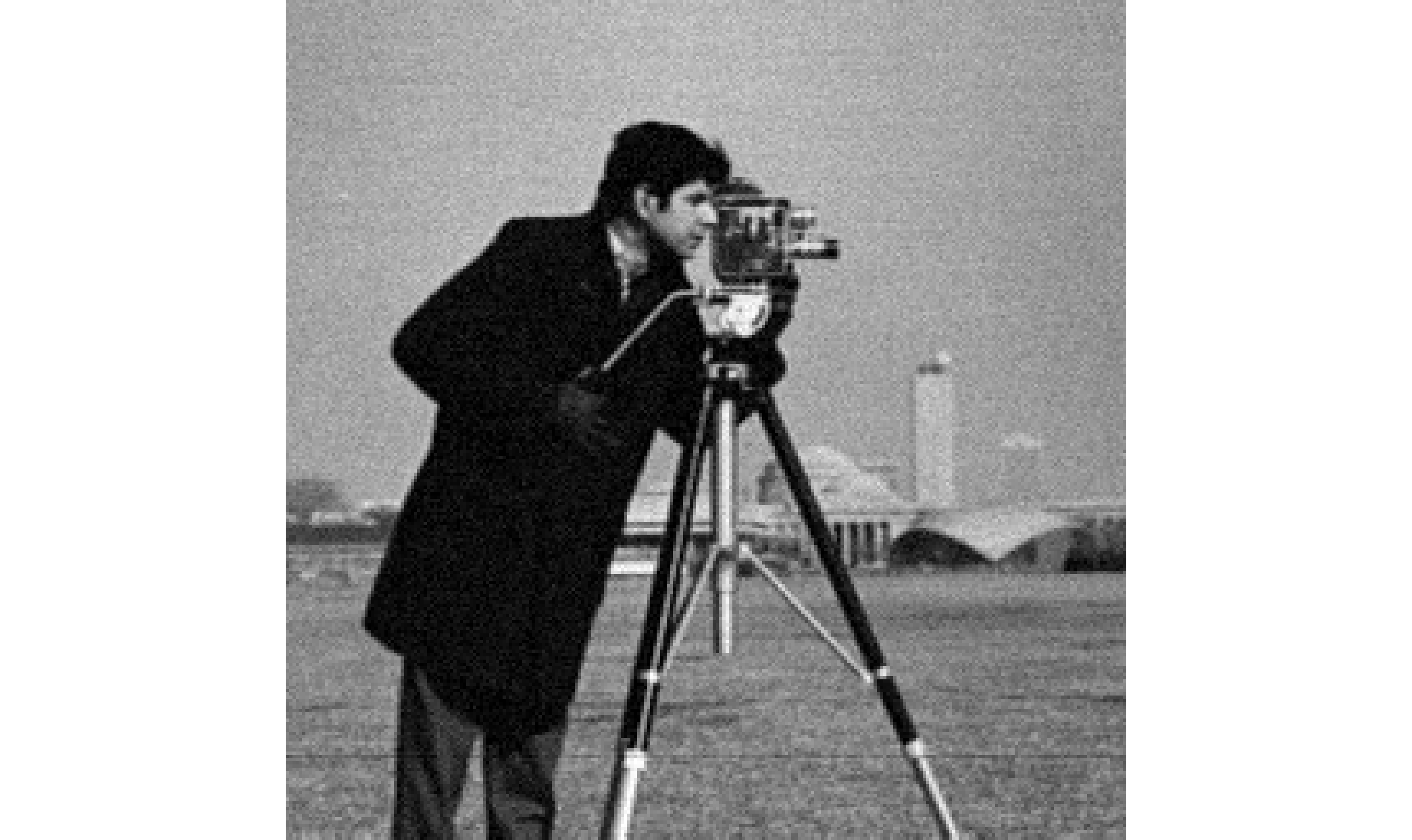}
		\end{minipage}
		\begin{minipage}[b]{.3\linewidth}
			\centering
			\includegraphics[width=\linewidth,trim=280 0 280 0,clip]{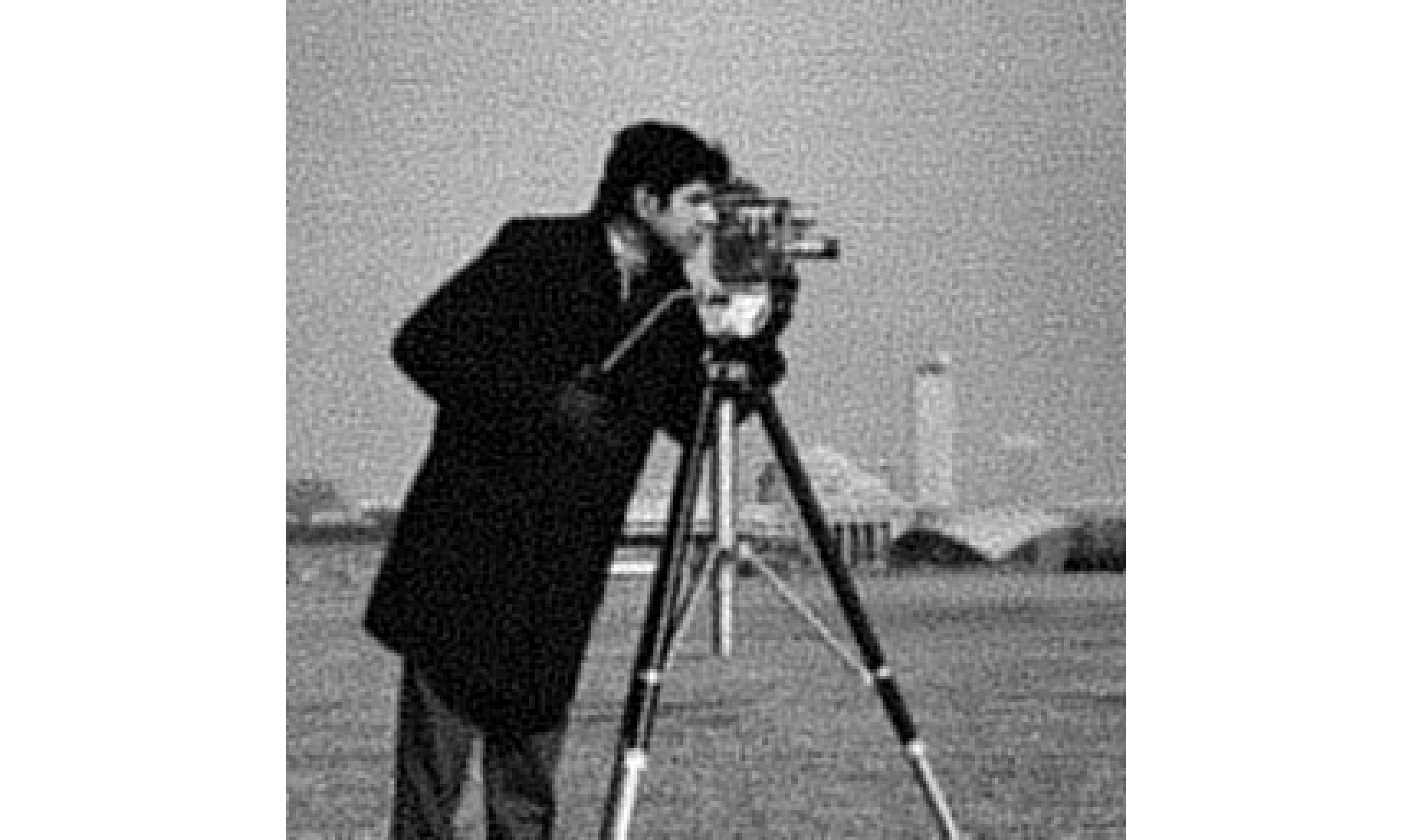}
		\end{minipage}
		\begin{minipage}[b]{.3\linewidth}
			\centering
			\includegraphics[width=\linewidth,trim=280 0 280 0,clip]{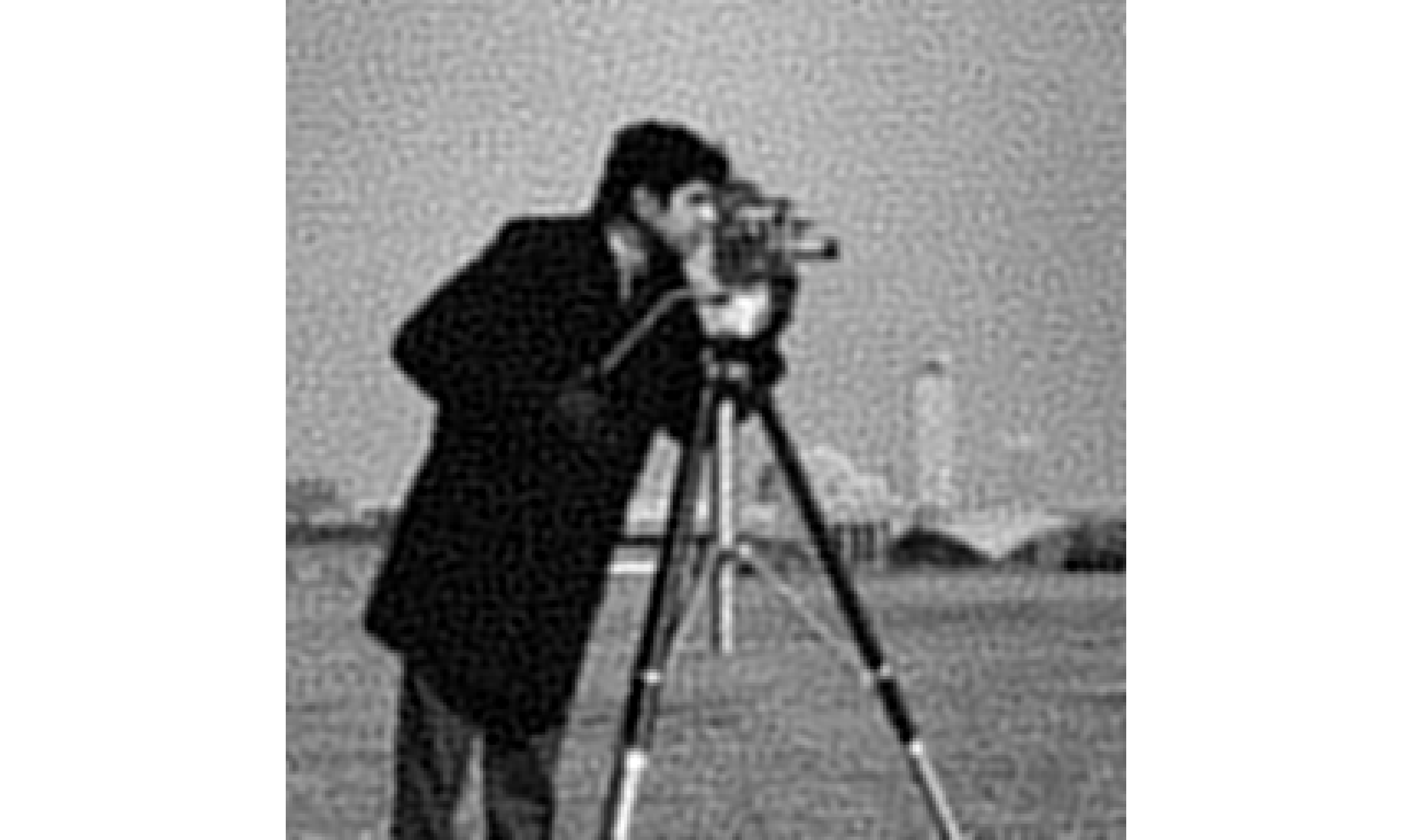}
		\end{minipage}
		\hspace{0.01\linewidth}
	}
	\subfigure
	{
		\begin{minipage}[b]{.3\linewidth}
			\centering
			\includegraphics[width=\linewidth,trim=280 0 280 0,clip]{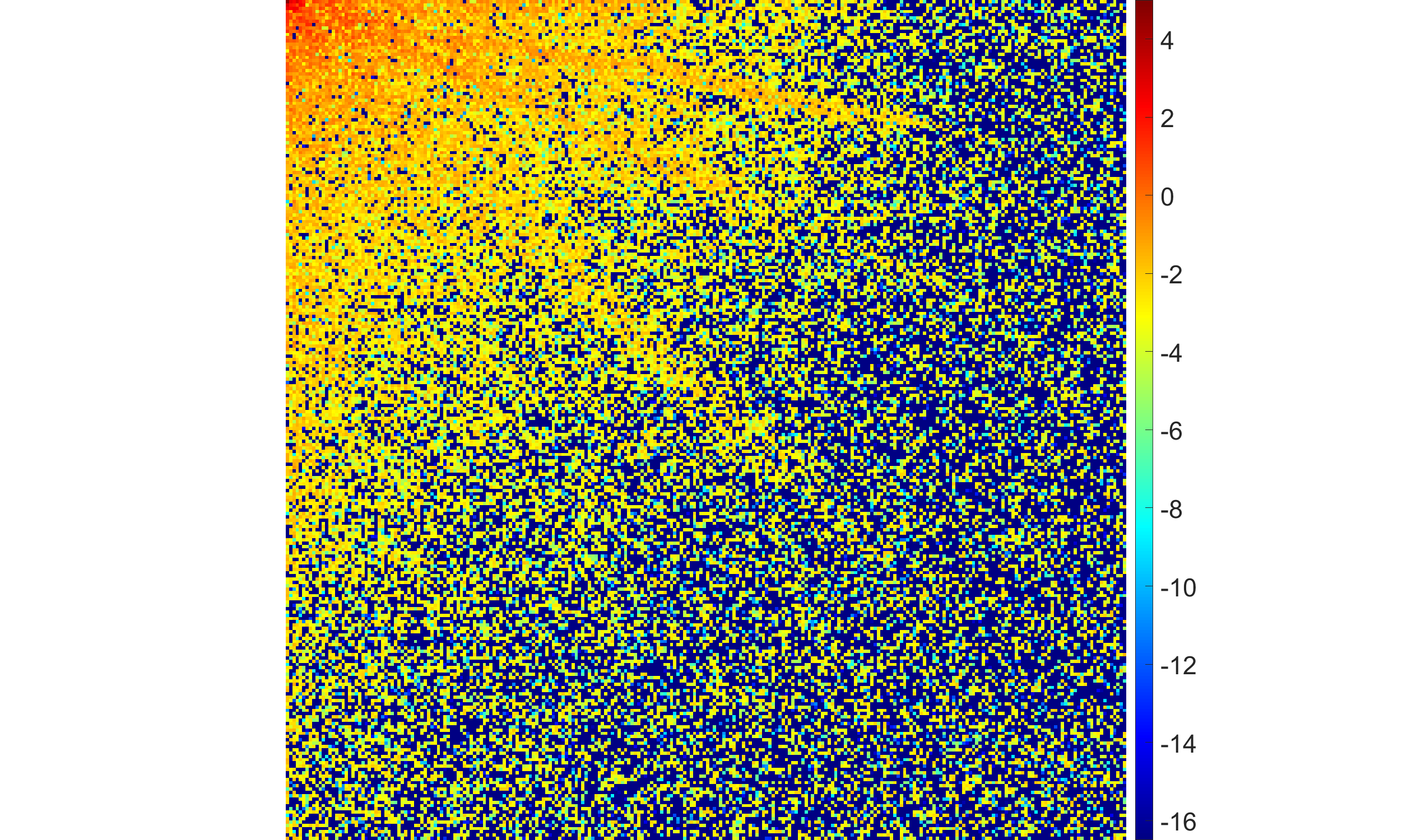}
		\end{minipage}
		\begin{minipage}[b]{.3\linewidth}
			\centering
			\includegraphics[width=\linewidth,trim=280 0 280 0,clip]{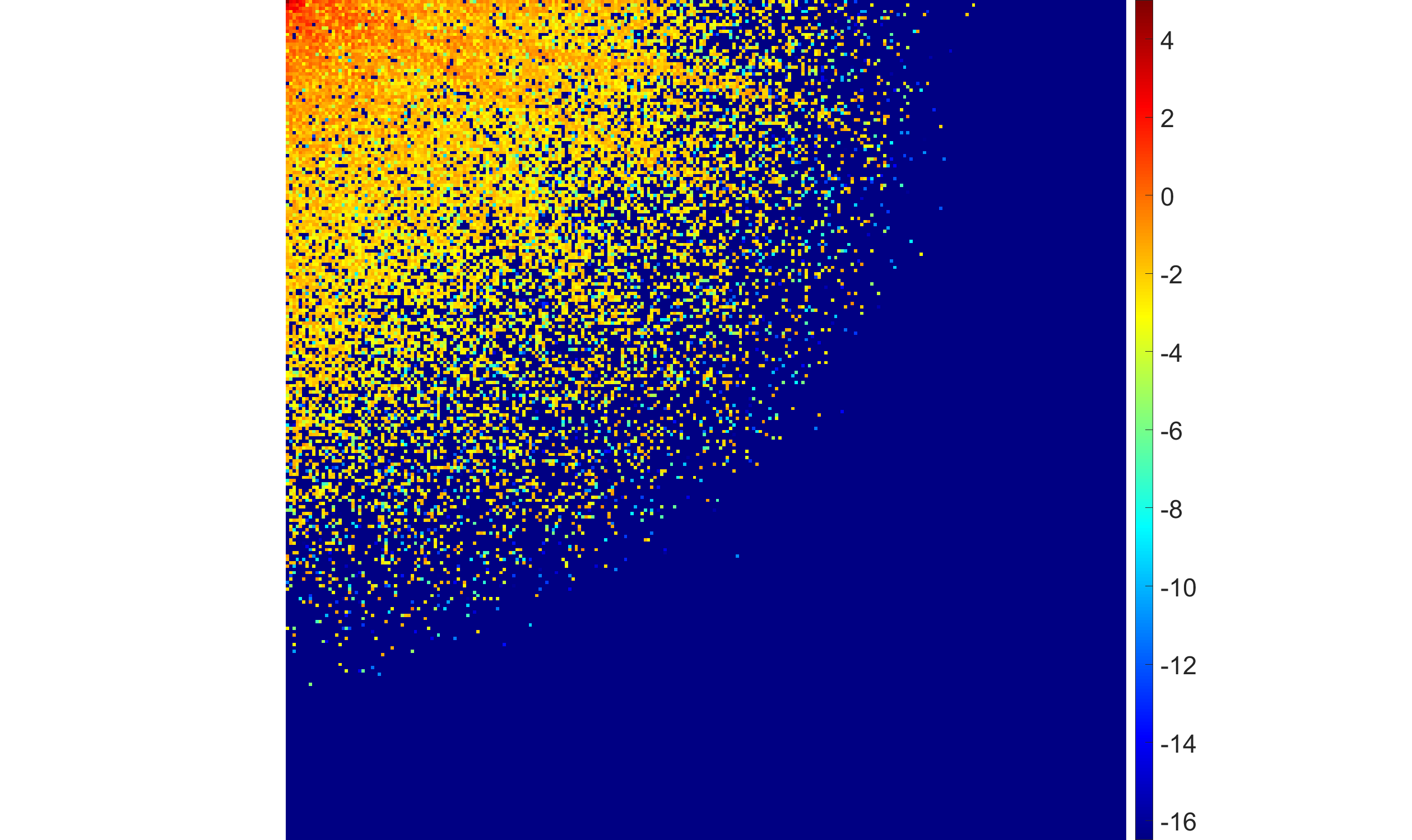}
		\end{minipage}
		\begin{minipage}[b]{.3\linewidth}
			\centering
			\includegraphics[width=\linewidth,trim=280 0 280 0,clip]{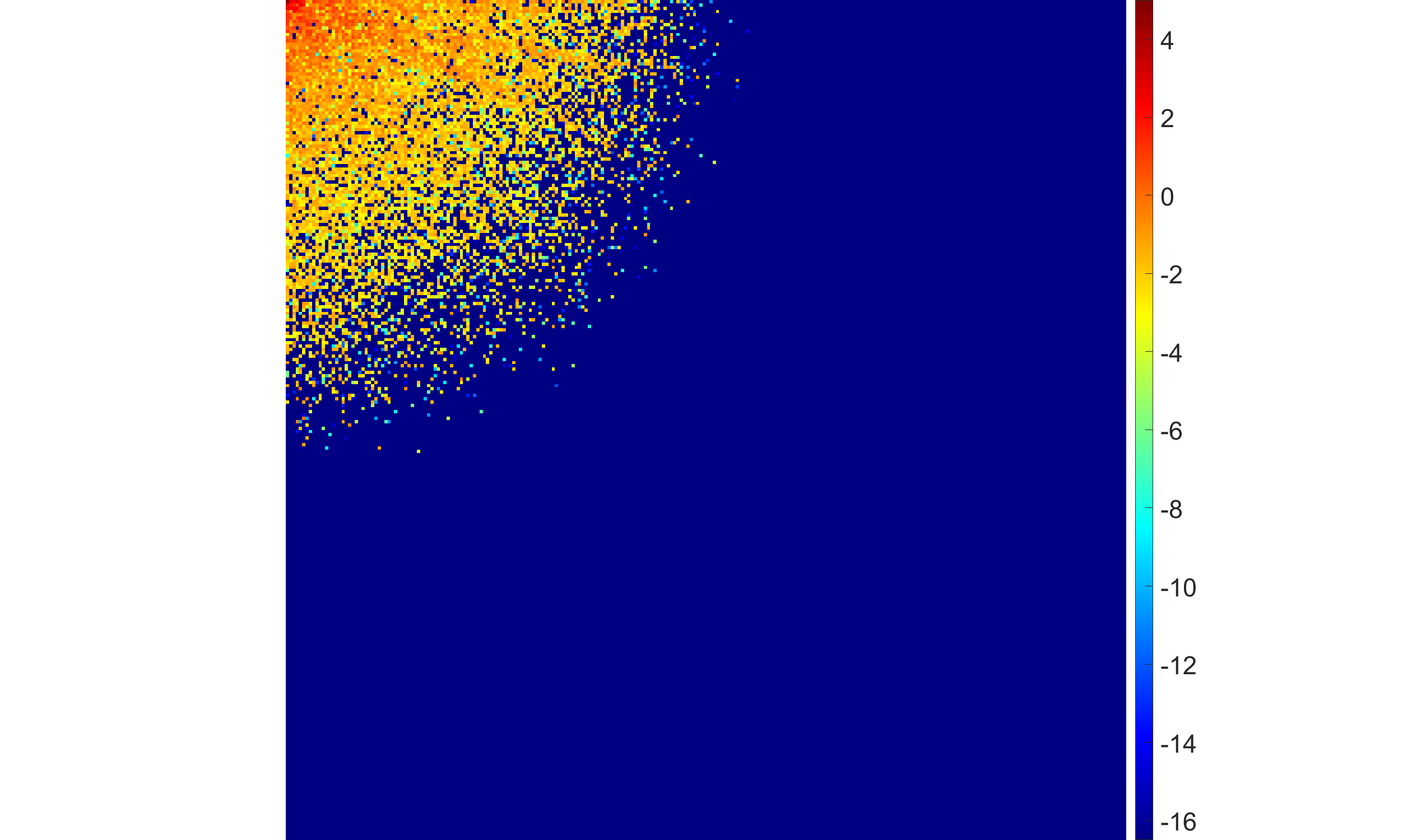}
		\end{minipage}
		\hspace{0.01\linewidth}
	}
	\subfigure
	{
		\begin{minipage}[b]{.3\linewidth}
			\centering
			\includegraphics[width=\linewidth,trim=280 0 280 0,clip]{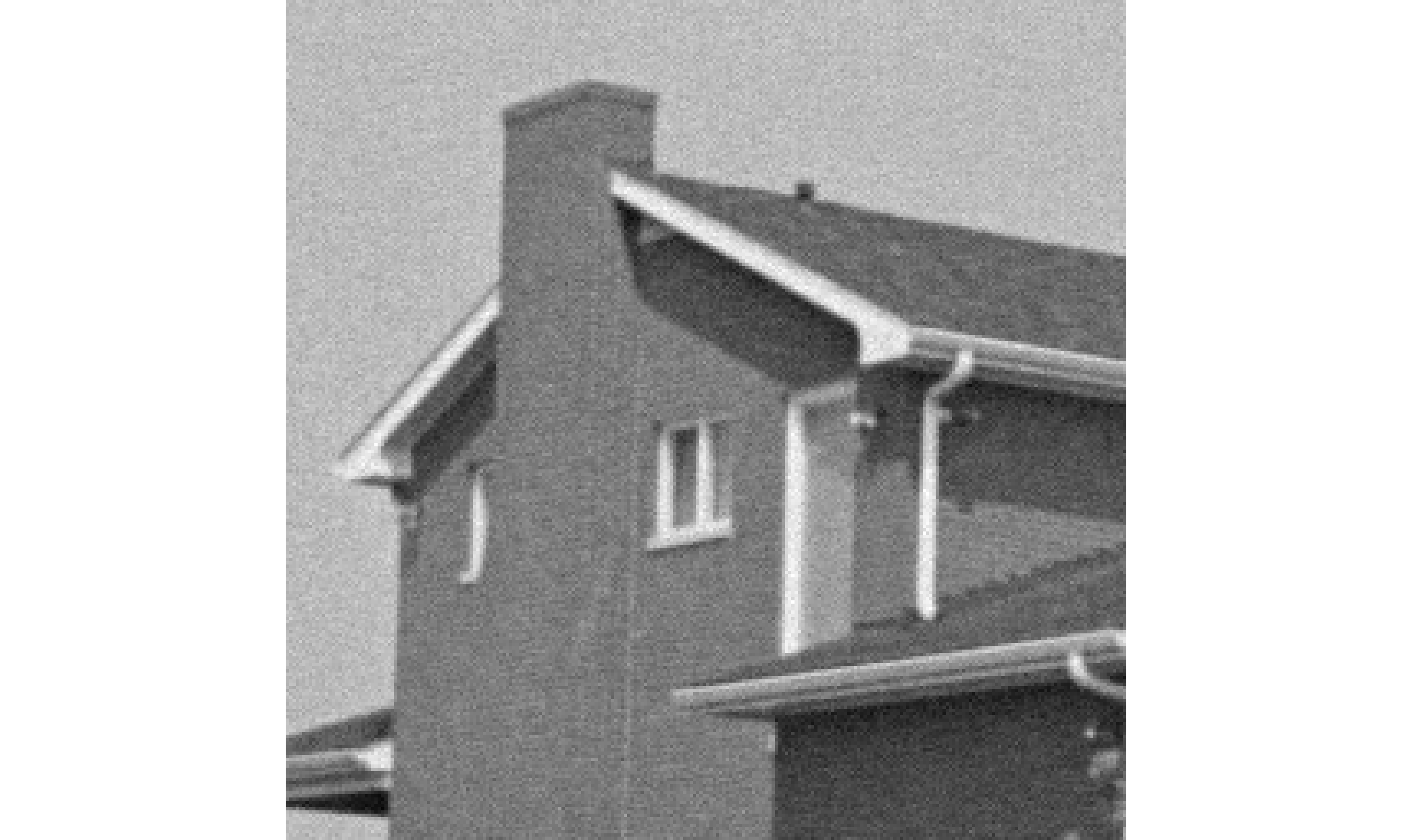}
		\end{minipage}
		\begin{minipage}[b]{.3\linewidth}
			\centering
			\includegraphics[width=\linewidth,trim=280 0 280 0,clip]{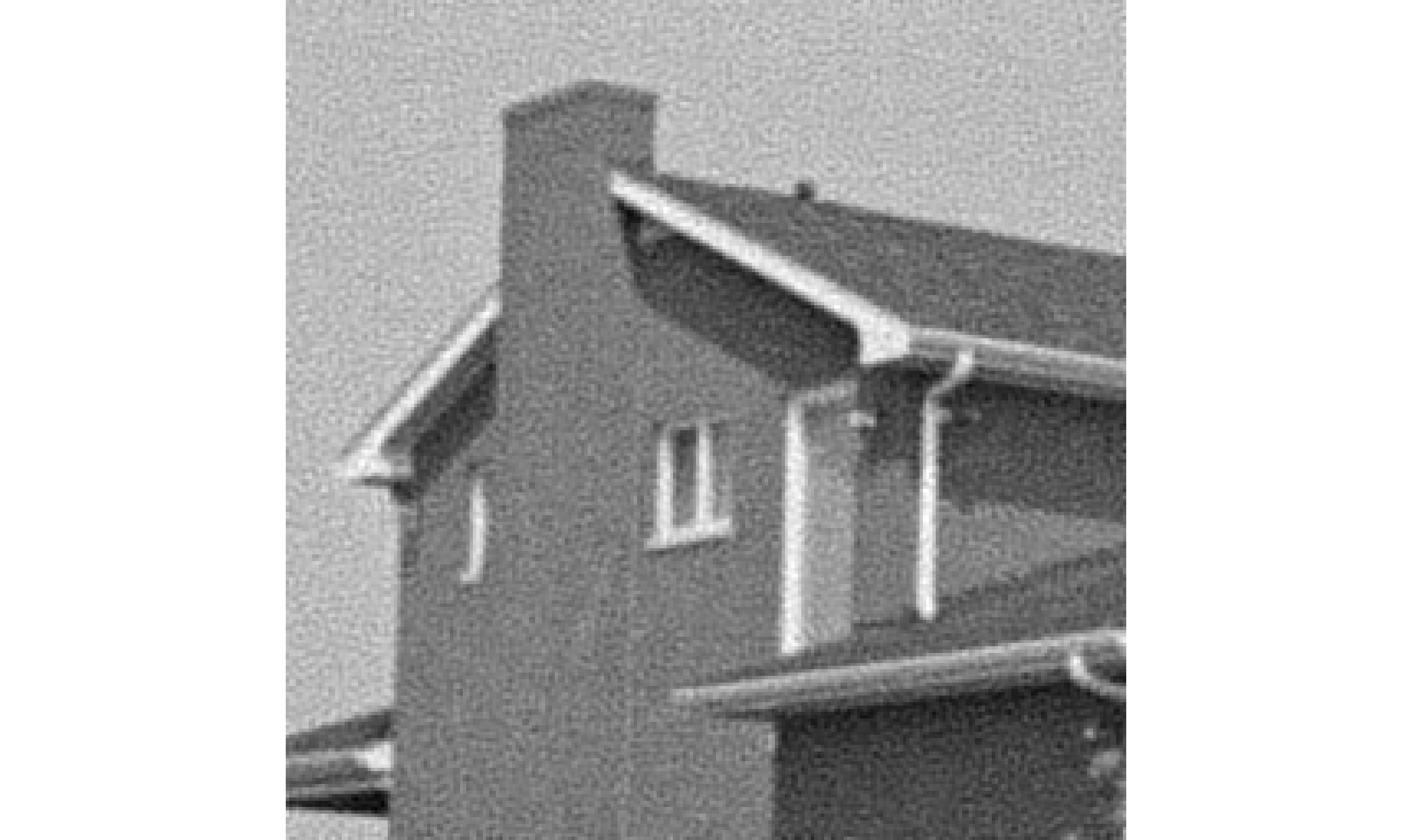}
		\end{minipage}
		\begin{minipage}[b]{.3\linewidth}
			\centering
			\includegraphics[width=\linewidth,trim=280 0 280 0,clip]{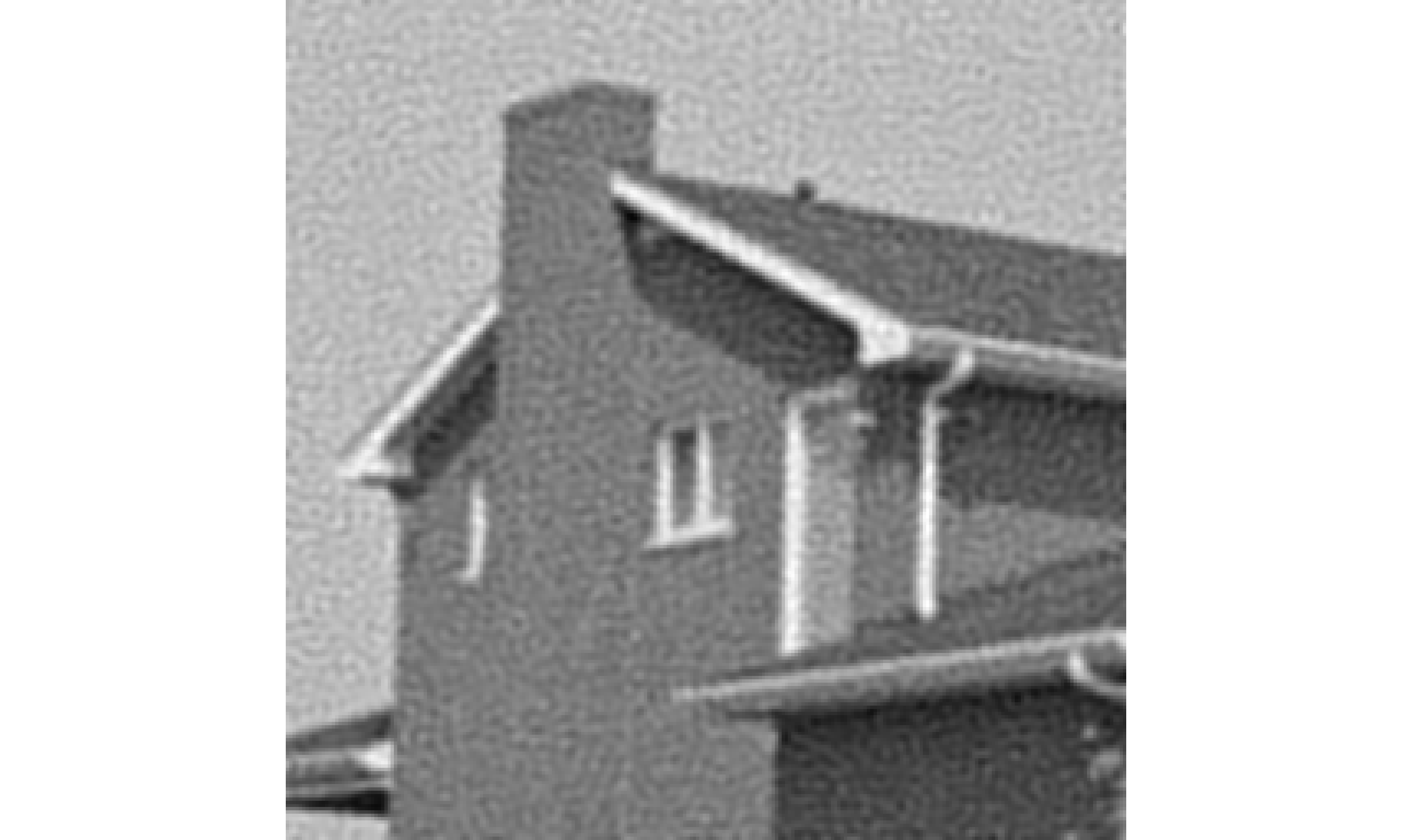}
		\end{minipage}
		\hspace{0.01\linewidth}
	}
	\subfigure
	{
		\begin{minipage}[b]{.3\linewidth}
			\centering
			\includegraphics[width=\linewidth,trim=280 0 280 0,clip]{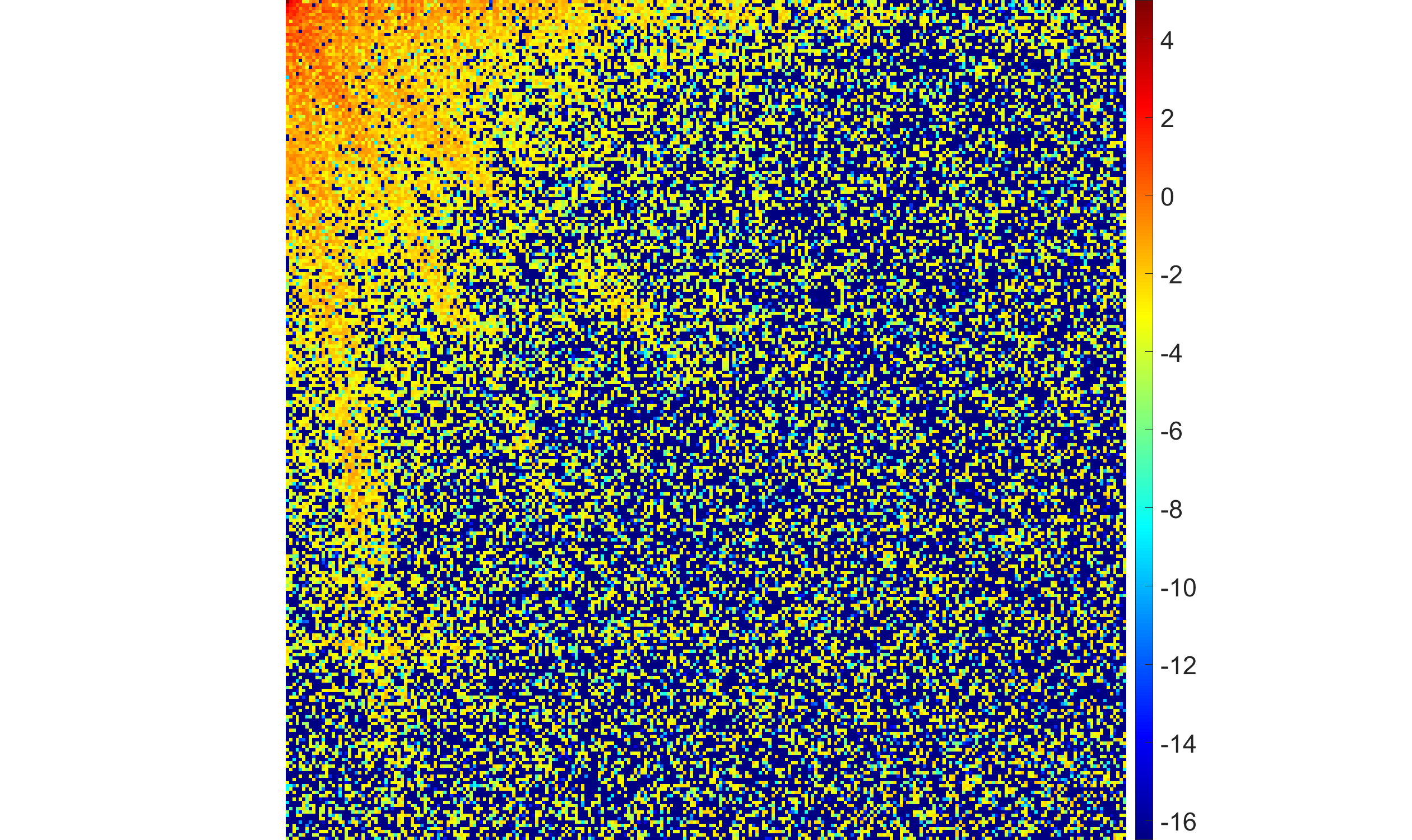}\\
			$\sigma_{ker}=0.5$
		\end{minipage}
		\begin{minipage}[b]{.3\linewidth}
			\centering
			\includegraphics[width=\linewidth,trim=280 0 280 0,clip]{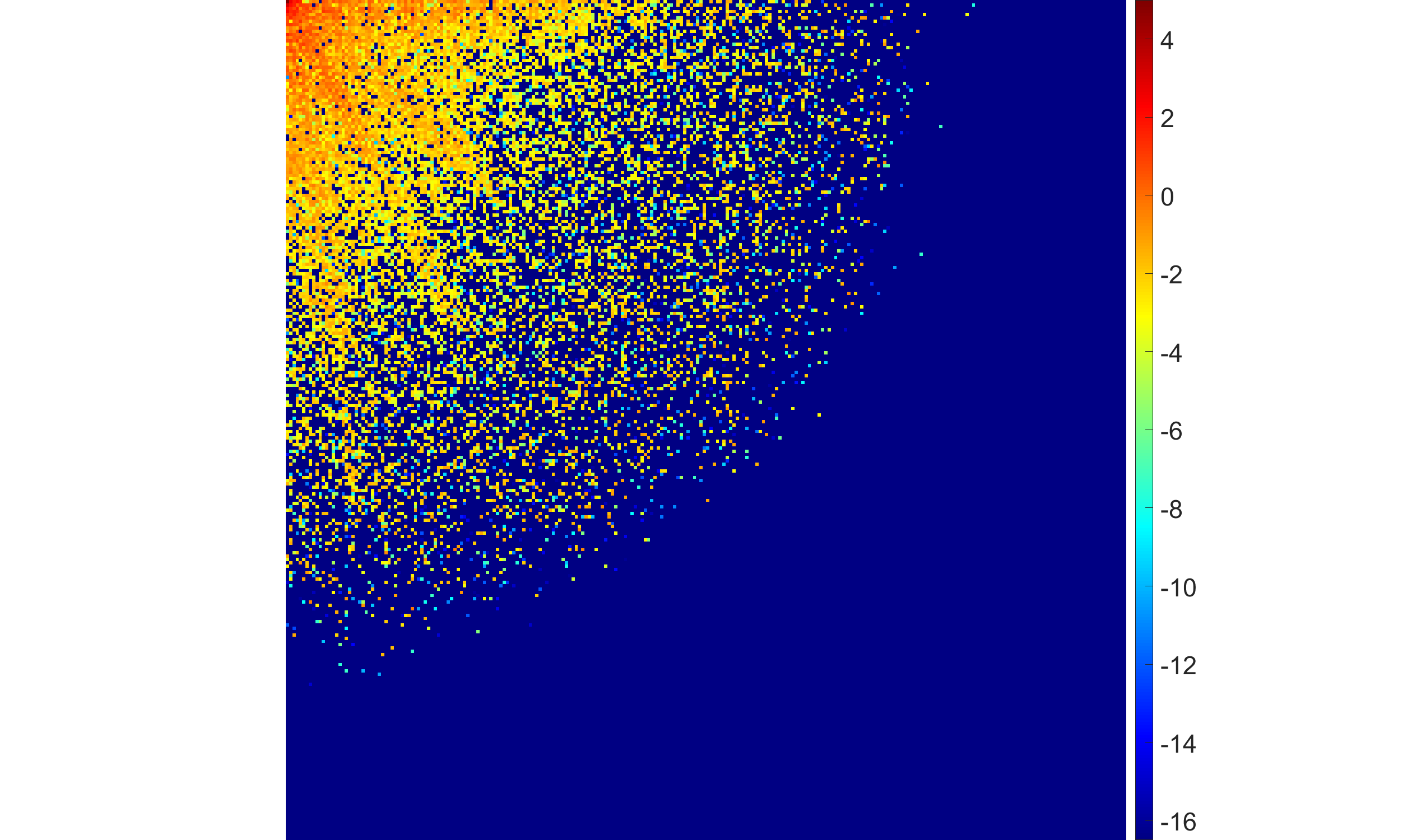}\\
			$\sigma_{ker}=1$
		\end{minipage}
		\begin{minipage}[b]{.3\linewidth}
			\centering
			\includegraphics[width=\linewidth,trim=280 0 280 0,clip]{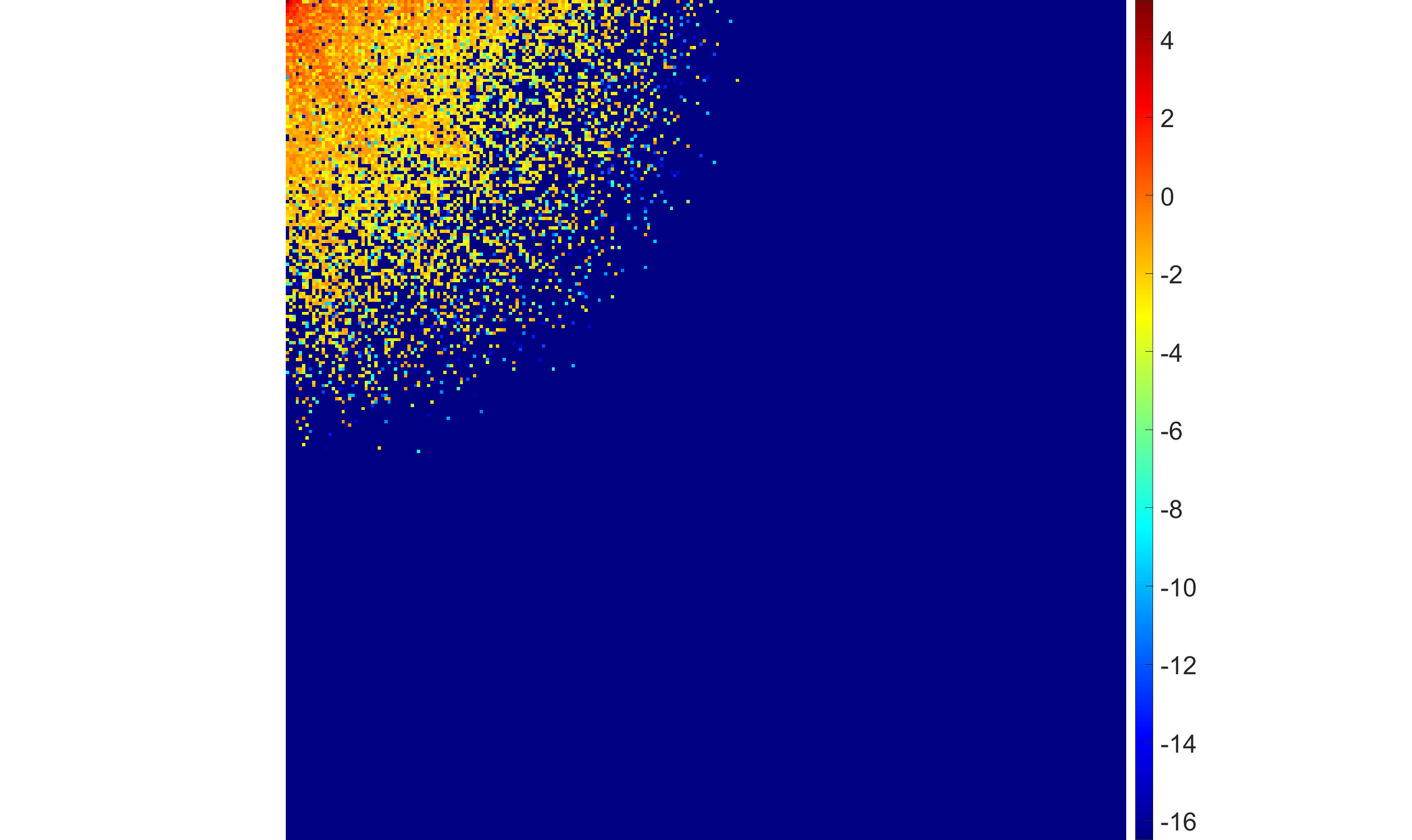}\\
			$\sigma_{ker}=1.5$
		\end{minipage}
		\hspace{0.01\linewidth}
	}
	\caption{The results of EBF with the Laplace hyperprior for restoring images degraded by different Gaussian blurring kernels. Top and third rows: restored images; Second and fourth rows: absolute values of DCT coefficients.}
\end{figure}

The restored results are shown in Figure \ref{fig4.restore}. We can see that in a rather well-posed case with $\sigma_{ker}=0.5$, EBF provides accurate restoration with coefficients concentrated in the low-frequency region forming a distinct fan-shaped pattern to preserve most details. As the level of ill-posedness increases, the sector-shaped areas of the DCT coefficient maps shrink.  That is because the stronger the ill-posedness of the problem, the more severely the high-frequency information is corrupted by noise. Therefore, the sparsity of the restored coefficients gradually increases. Furthermore, because more and more high-frequency coefficients are suppressed to zero, the restored image loses more details as $\sigma_{ker}$ increases. In \cref{tab:restoration}, we list the relative errors of the restored images and the sparsity rates of the DCT coefficients. It is clear that the sparsity rate of the restored result increases as the problem becomes more ill-posed due to amplified noise on the high-frequency. Moreover, more ill-posed problem is more challenging to solve, leading to larger relative error. We observe that for the \textit{House} test image the relative error in the case of $\sigma_{ker}=1.5$ is slightly smaller than the case of $\sigma_{ker}=1.0$. This is due to the trade-off between suppressing noise and preserving details. 

%These results indicate that the fan-shaped distribution of the DCT coefficients restored by the method with Laplace hyperprior shrinks as the ill-posedness of the problem increases, while maintaining the characteristic of low-frequency concentration, effectively improving the sparsity of the model and its robustness to ill-posed problems.

\subsubsection{Noise}
In this test, we fix $\sigma_{ker}=1$ and vary the noise level to assess the performance of EBF with a half-Laplace hyperprior $(\beta=0.1)$. Specifically, we test three levels of Gaussian noise: $5\%$, $10\%$, and $20\%$.

\begin{figure}[t]\label{fig5.noise}
	\centering
	\subfigure
	{
		\begin{minipage}[b]{.3\linewidth}
			\centering
			\includegraphics[width=\linewidth,trim=280 0 280 0,clip]{camera5lax.png}
		\end{minipage}
		\begin{minipage}[b]{.3\linewidth}
			\centering
			\includegraphics[width=\linewidth,trim=280 0 280 0,clip]{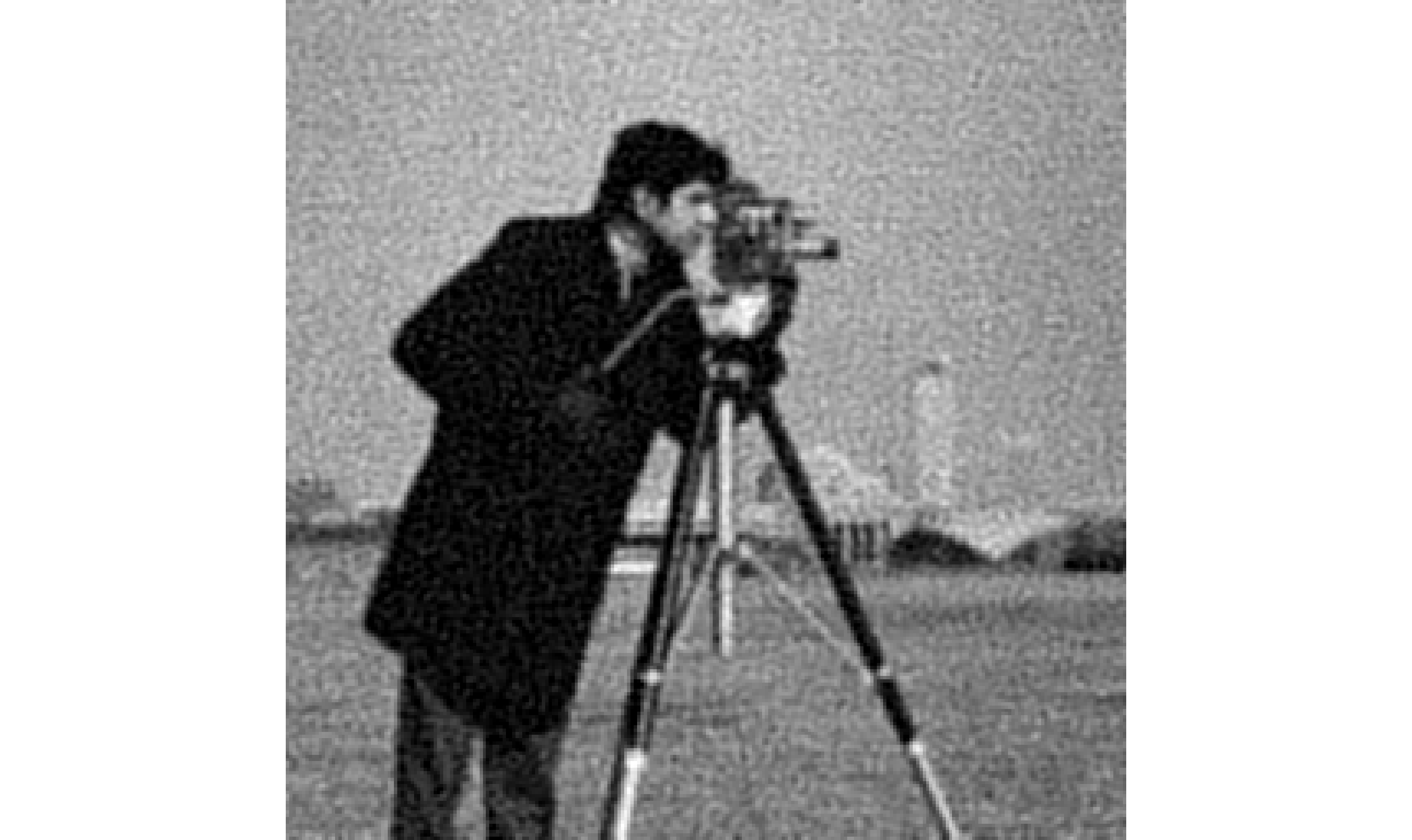}
		\end{minipage}
		\begin{minipage}[b]{.3\linewidth}
			\centering
			\includegraphics[width=\linewidth,trim=280 0 280 0,clip]{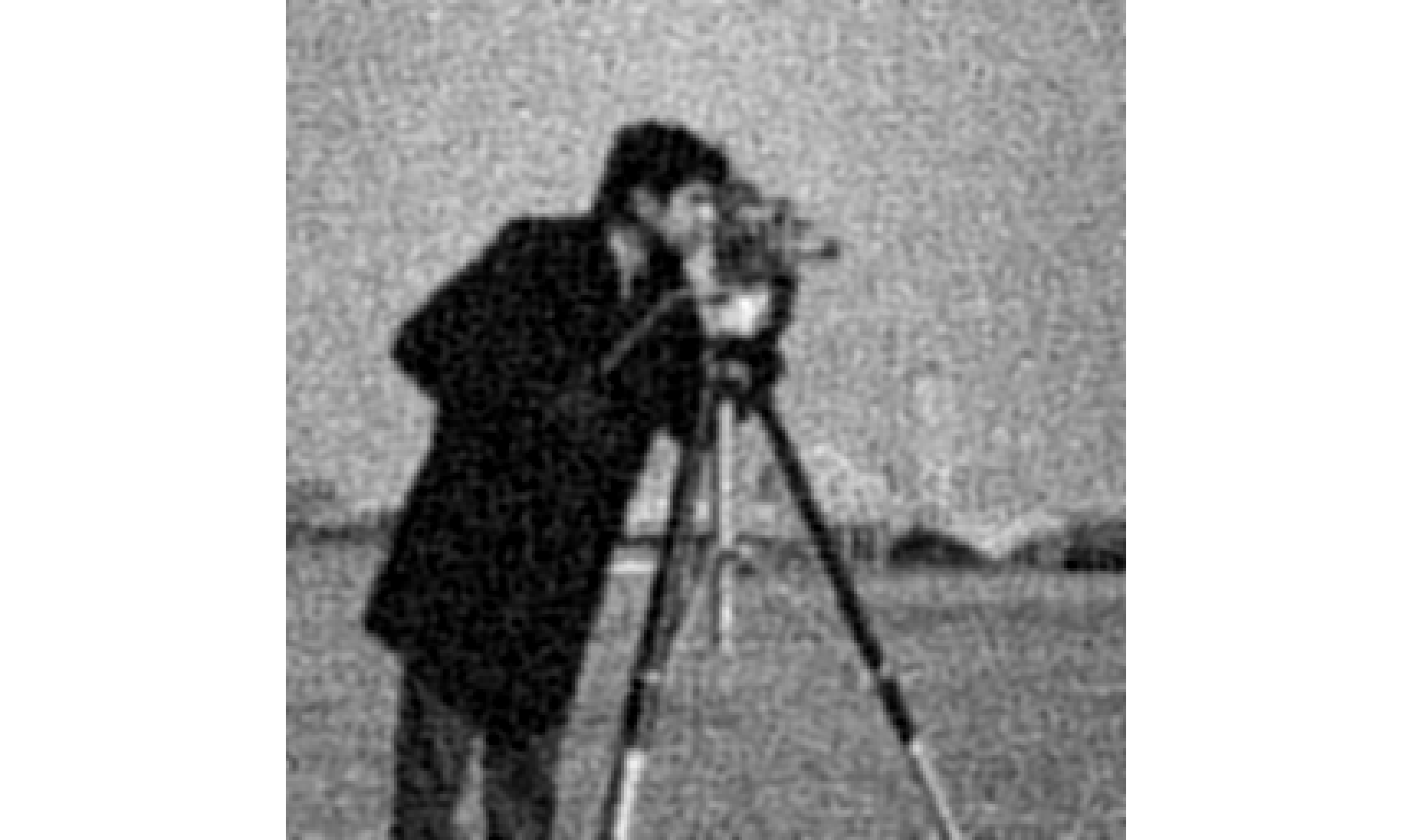}
		\end{minipage}
		\hspace{0.01\linewidth}
	}
	\subfigure
	{
		\begin{minipage}[b]{.3\linewidth}
			\centering
			\includegraphics[width=\linewidth,trim=280 0 280 0,clip]{camera5lac.png}
		\end{minipage}
		\begin{minipage}[b]{.3\linewidth}
			\centering
			\includegraphics[width=\linewidth,trim=280 0 280 0,clip]{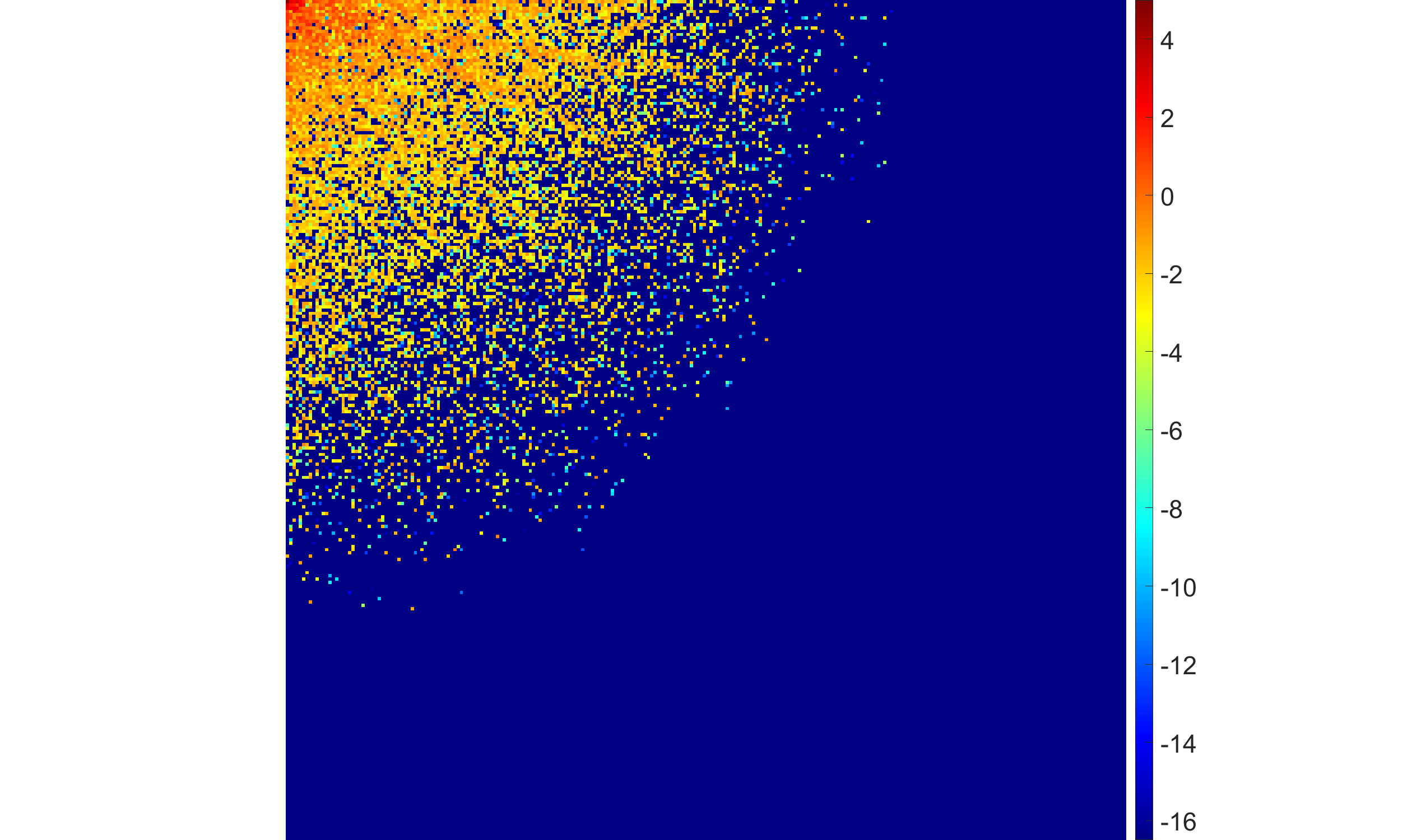}
		\end{minipage}
		\begin{minipage}[b]{.3\linewidth}
			\centering
			\includegraphics[width=\linewidth,trim=280 0 280 0,clip]{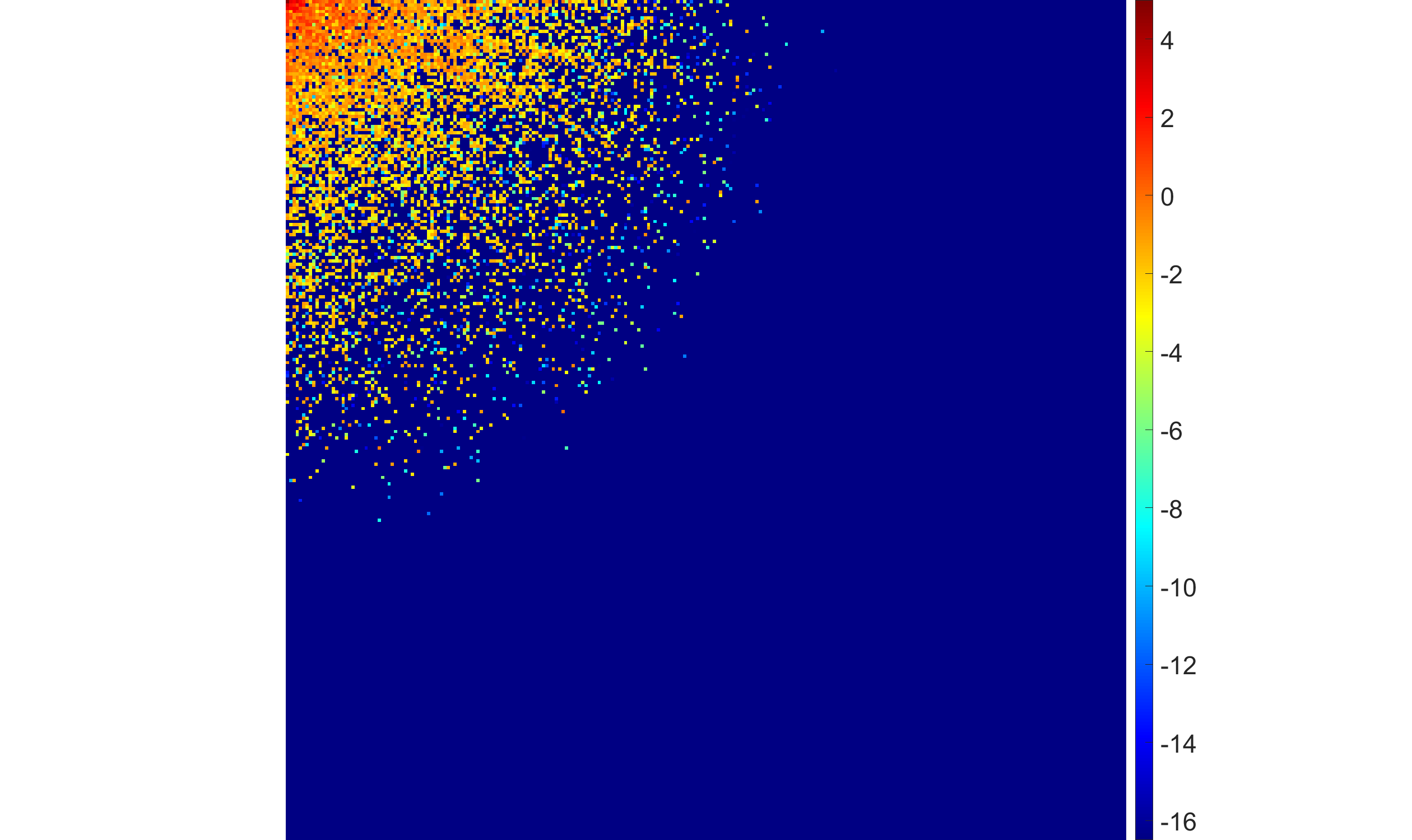}
		\end{minipage}
		\hspace{0.01\linewidth}
	}
	\subfigure
	{
		\begin{minipage}[b]{.3\linewidth}
			\centering
			\includegraphics[width=\linewidth,trim=280 0 280 0,clip]{house5lax.png}
		\end{minipage}
		\begin{minipage}[b]{.3\linewidth}
			\centering
			\includegraphics[width=\linewidth,trim=280 0 280 0,clip]{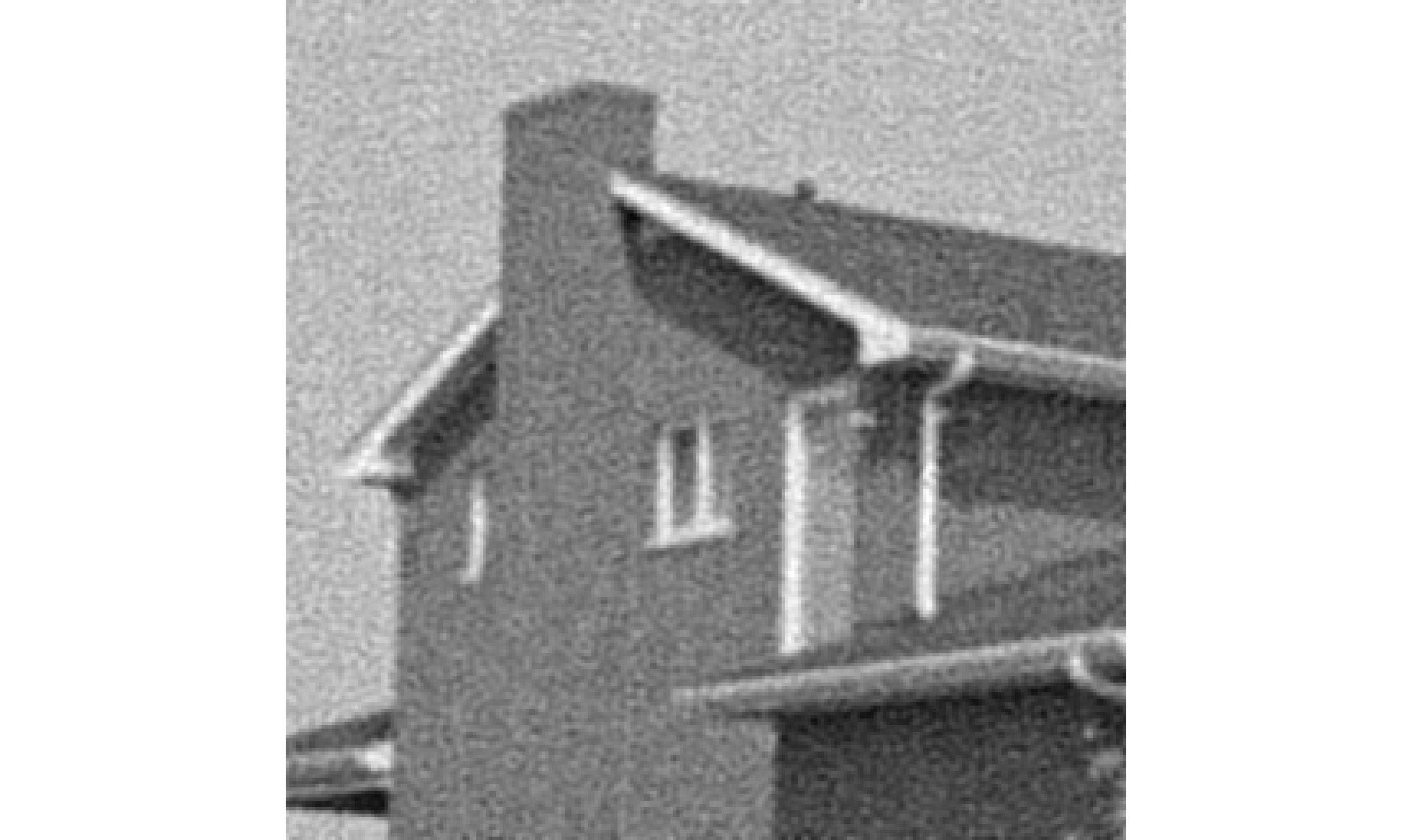}
		\end{minipage}
		\begin{minipage}[b]{.3\linewidth}
			\centering
			\includegraphics[width=\linewidth,trim=280 0 280 0,clip]{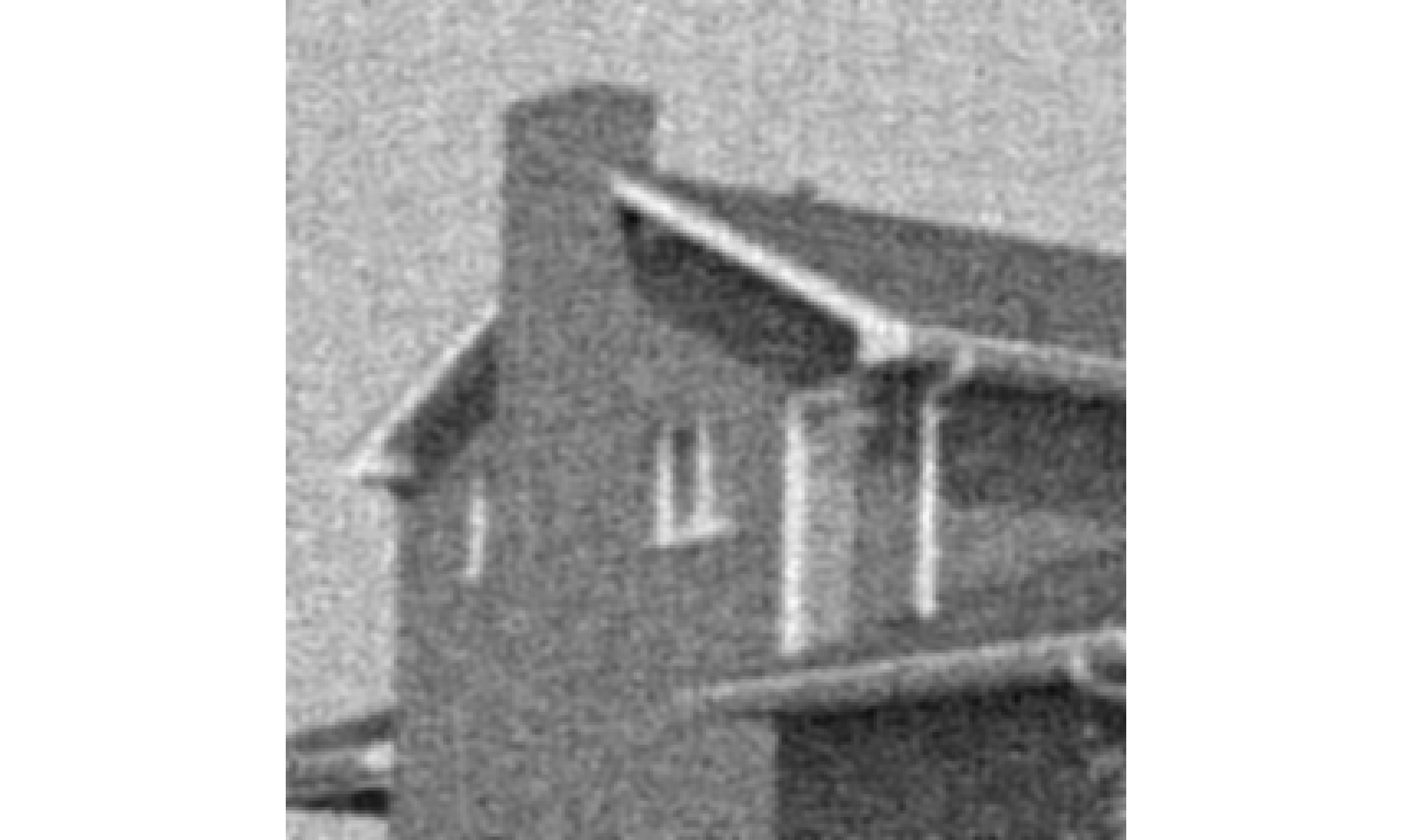}
		\end{minipage}
		\hspace{0.01\linewidth}
	}
	\subfigure
	{
		\begin{minipage}[b]{.3\linewidth}
			\centering
			\includegraphics[width=\linewidth,trim=280 0 280 0,clip]{house5lac.png}\\
			$5\%$
		\end{minipage}
		\begin{minipage}[b]{.3\linewidth}
			\centering
			\includegraphics[width=\linewidth,trim=280 0 280 0,clip]{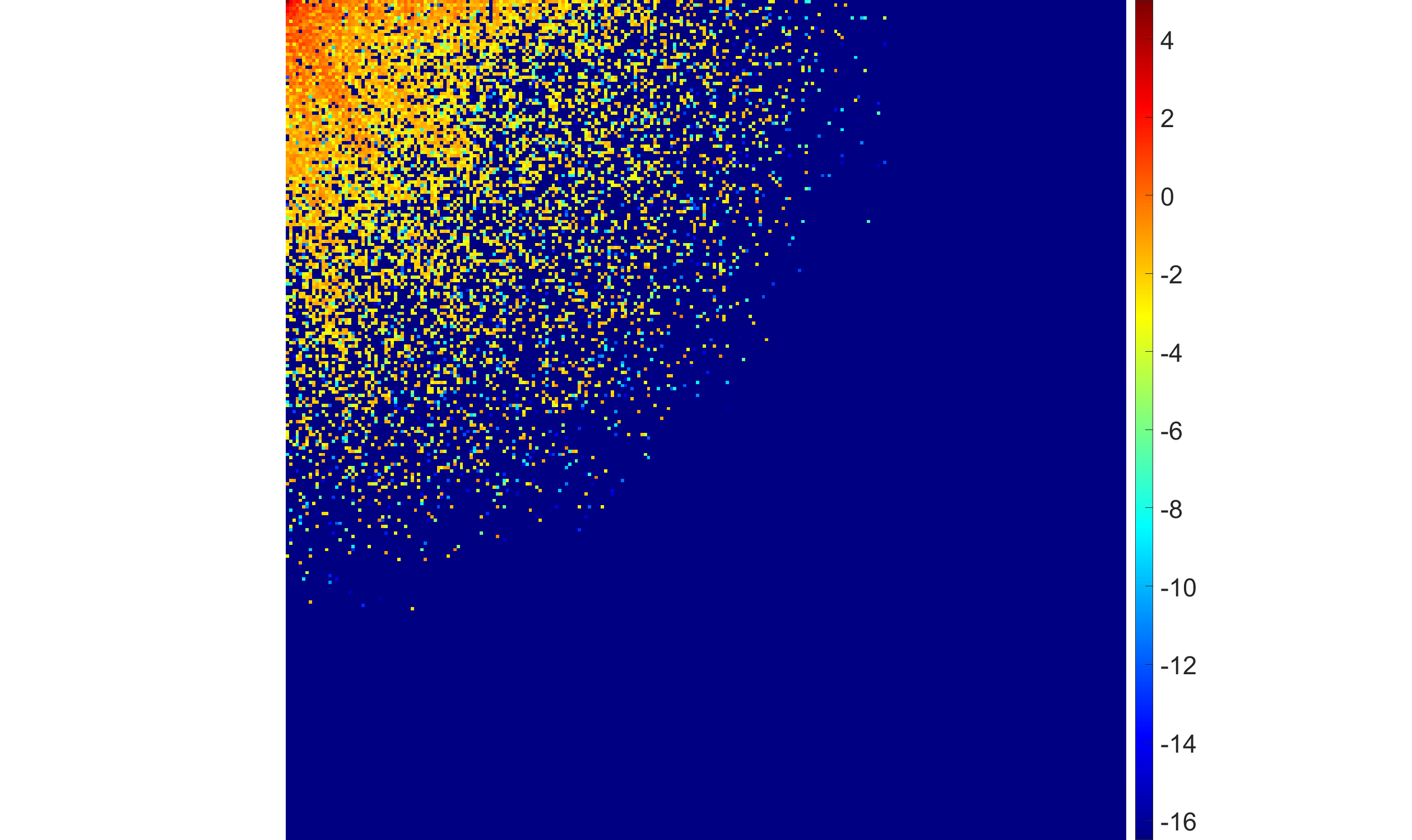}\\
			$10\%$
		\end{minipage}
		\begin{minipage}[b]{.3\linewidth}
			\centering
			\includegraphics[width=\linewidth,trim=280 0 280 0,clip]{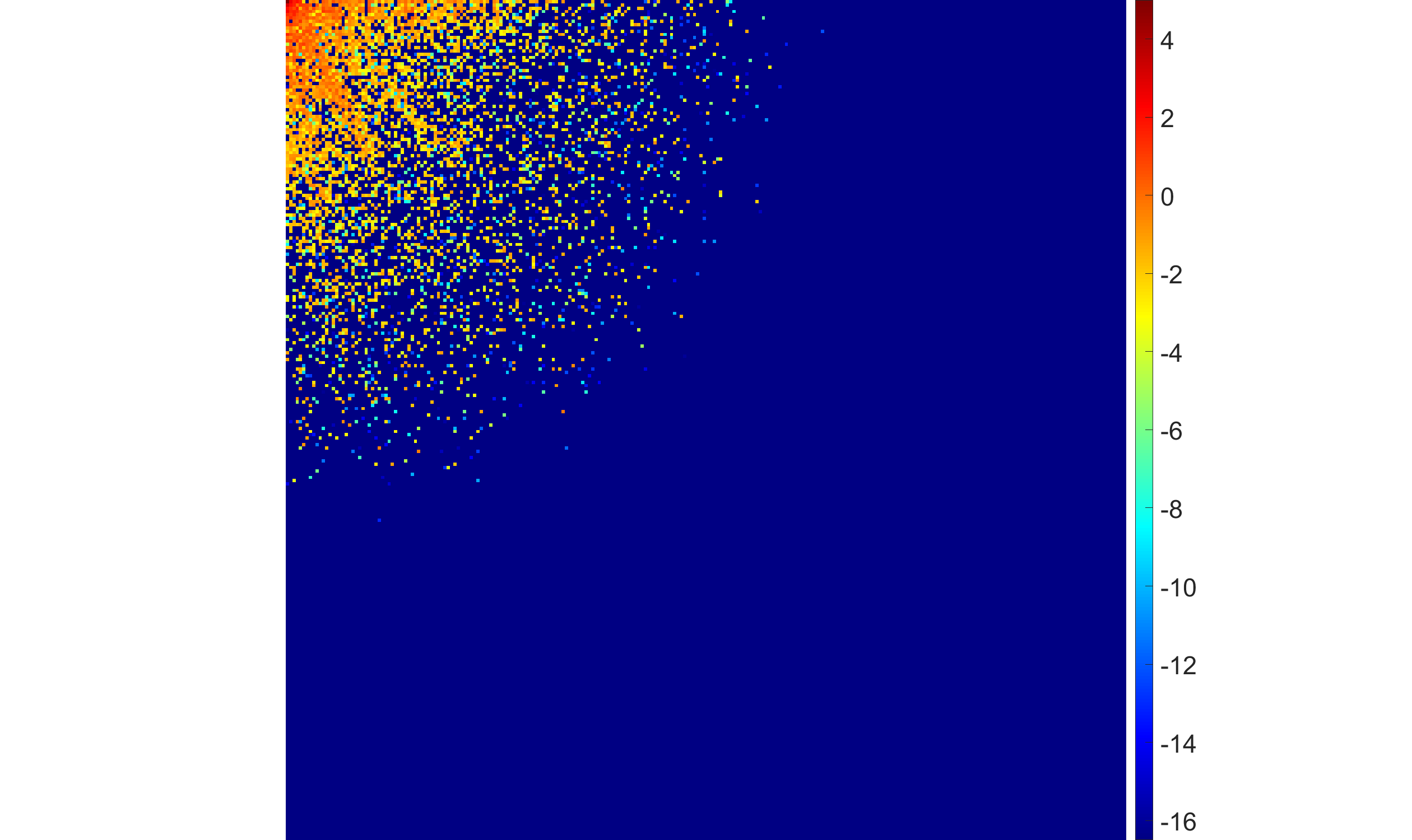}\\
			$20\%$
		\end{minipage}
		\hspace{0.01\linewidth}
	}
	\caption{The results of EBF with the half-Laplace hyperprior for restoring images degraded by different noise levels. Top and third rows: restored images; Second and fourth rows: absolute values of DCT coefficients.}
\end{figure}

The restored results are given in Figure \ref{fig5.noise}. 
It shows that as the noise level increases, the sparsity of the restored DCT coefficients gradually increases and more high-frequency coefficients are turned to zero. In low noise level cases, more the high-frequency DCT coefficients contribute to the restoration, therefore the restored images preserve most of the details. However, in high noise level cases, in order to suppress the influence of the noise most of high-frequence coefficients become zero, so some details in the restored image are lost and some compression artifacts appear. In \cref{tab:restoration} we list the relative errors of the restored images and the sparsity rates of the DCT coefficients. It is obvious that the relative error and the sparsity rate increase as the noise level increases. It confirms that EBF with a sparsity-promotting hyperprior can automatically adjust the sparsity with respect to the noise level to obtain a good balance between suppressing noise and preserving details.

\begin{table}[htbp]
	\centering
	\caption{Restoration performance under different levels of ill-posedness (with 5\% noise) and different noise levels (with $\sigma_{ker} = 0.1$).}
	\renewcommand{\arraystretch}{1.2}
	\resizebox{\textwidth}{!}{
	\begin{tabular}{lllccc}
		\toprule
		Test type & Image & Parameter&   & Relative error & Sparsity (\%) \\
		\midrule
		\multirow{6}{*}{Ill-posedness} 
		& \multirow{3}{*}{\textit{Cameraman}} & \multirow{3}{*}{$\sigma_{\mathrm{ker}}$} 
		& 0.5 & 0.0627 & 41.52 \\ 
		&  &  & 1.0 & 0.0909 & 77.62 \\
		&  &  & 1.5 & 0.1006  & 88.49   \\
		\cmidrule(lr){2-6}
		& \multirow{3}{*}{\textit{House}}     & \multirow{3}{*}{$\sigma_{\mathrm{ker}}$} & 0.5& 0.0550   & 49.32   \\
		&  &  & 1.0 & 0.0684   & 81.20   \\
		&  &  & 1.5 & 0.0627   & 90.02   \\
		\midrule
		\multirow{6}{*}{Noise level} 
		& \multirow{3}{*}{\textit{Cameraman}} & \multirow{3}{*}{Noise level (\%)} & 5  & 0.0909 & 77.62 \\
		&  & & 10 & 0.1055 & 84.73 \\
		&  & & 20 & 0.1273 & 90.99 \\
		\cmidrule(lr){2-6}
		& \multirow{3}{*}{\textit{House}}     & \multirow{3}{*}{Noise level (\%)} & 5  & 0.0684 & 81.20 \\
		&  &       & 10 & 0.0751 & 87.28 \\
		&  &       & 20 & 0.0870 & 92.58 \\
		\bottomrule
	\end{tabular}}
	\label{tab:restoration}
\end{table}

\subsection{Convergence of PALM}
In the last test, we study the convergence of \cref{alg} numerically. Here, we use EBF equippied with a half-Laplace hyperprior to solve the image deblurring problem with $10\%$ noise and $\sigma_{ker}=1$. We test our PALM algorithm with different choices of $\beta$ in the half-Laplace hyperprior. The convergence plots are shown in Figure \ref{fig.converge}, where we plot the values of the objective function defined in \eqref{eq:eqxgamma} in terms of iterations. It is obvious that the objective function values decrease monotonically, and the algorithm converges. Furthermore, the algorithm converges faster with smaller $\beta$, since in this case the half-Laplace hyperprior is closer to the $\delta$ distribution and dominates the posterior.         

\begin{figure}[t]\label{fig.converge}
	\centering
	\subfigure
	{
		\begin{minipage}[b]{.45\linewidth}
			\centering
			\includegraphics[width=\linewidth]{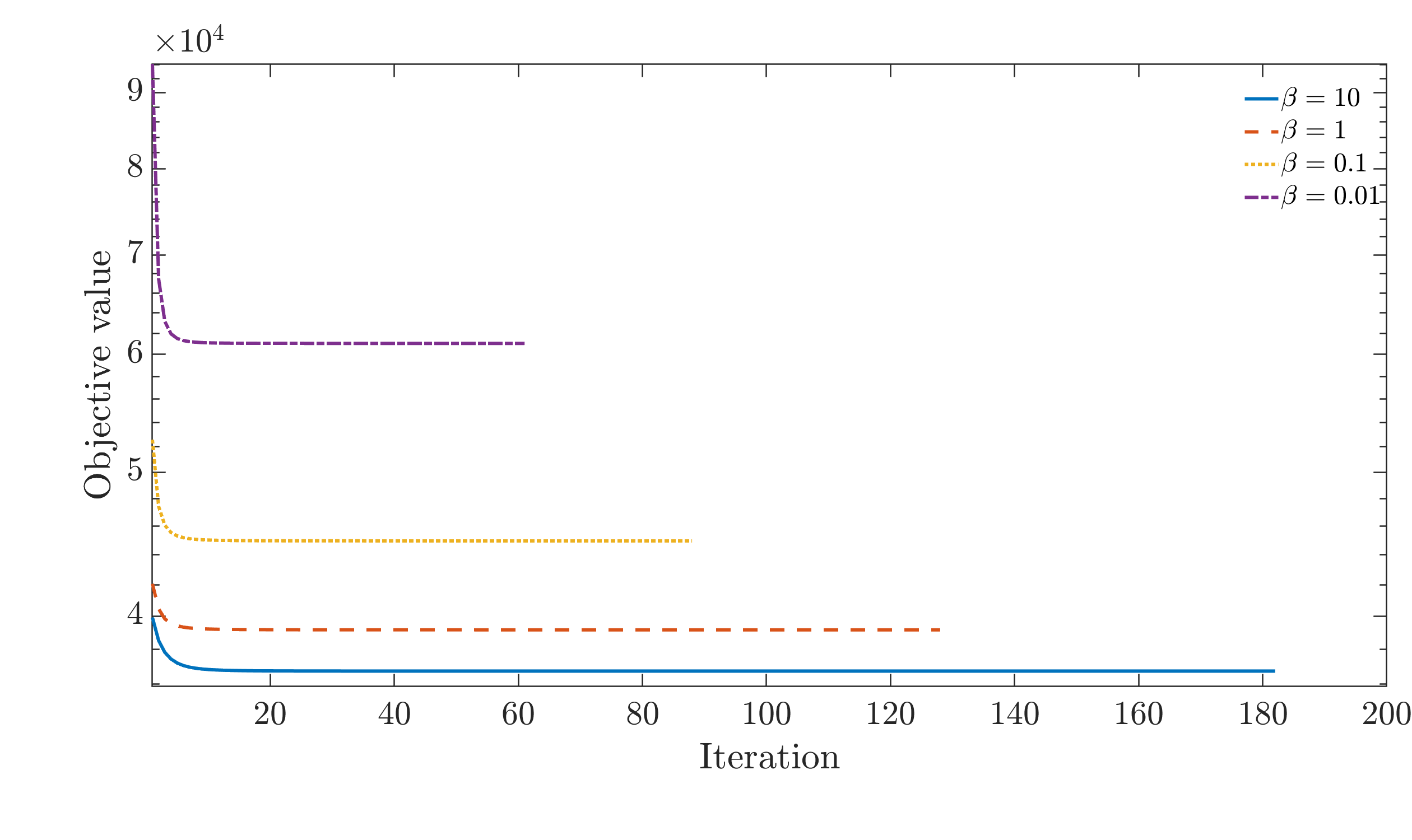}
		\end{minipage}
	}
	\subfigure
	{
		\begin{minipage}[b]{.45\linewidth}
			\centering
			\includegraphics[width=\linewidth]{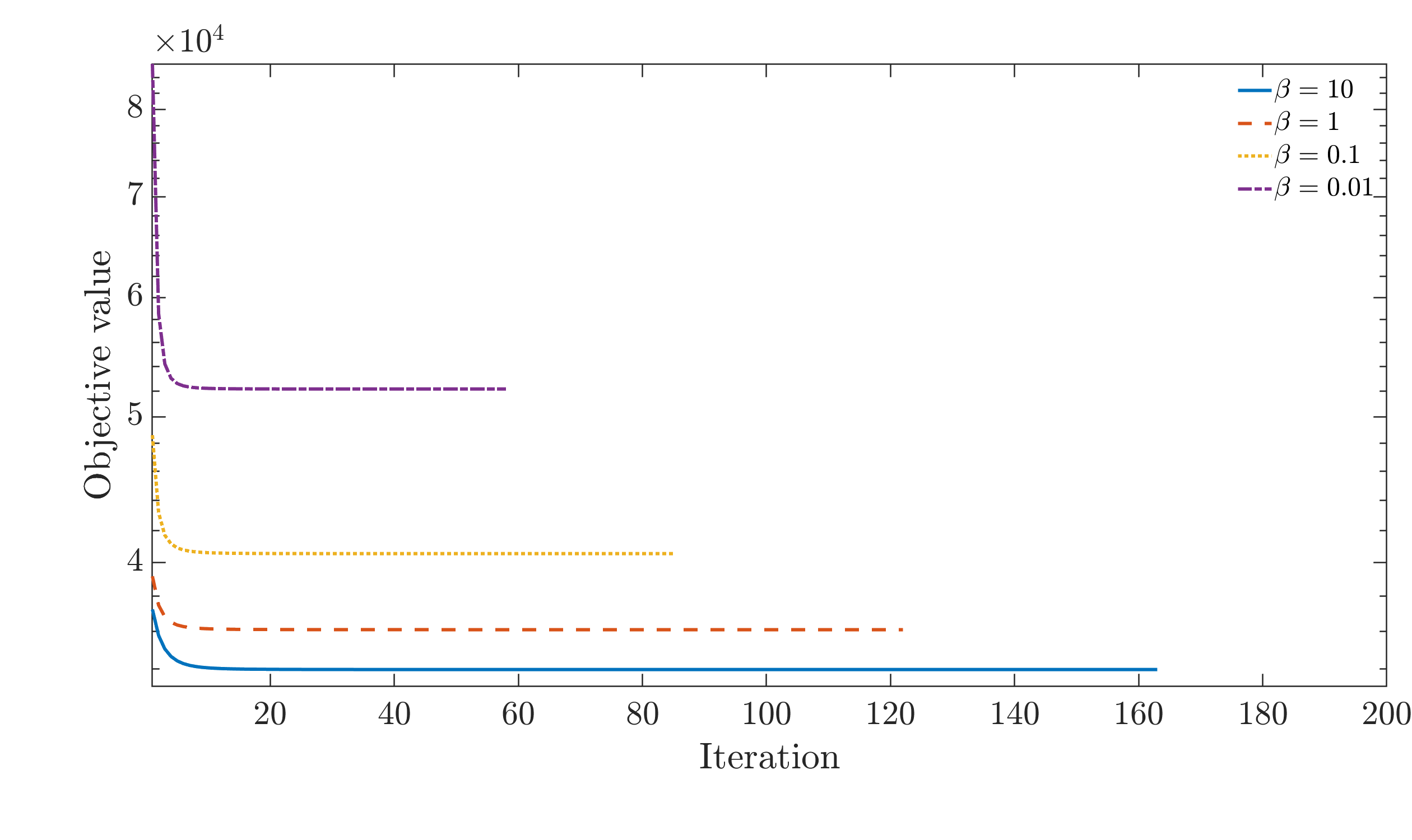}
		\end{minipage}
	}
	\caption{The convergence results of PALM in \cref{alg}. Left: Cameraman; Right: House.}
\end{figure}

\section{Conclusions}
\label{sec:conclusions}
This work investigated the fundamental question of how different hyperpriors influence sparsity and stability within the empirical Bayes framework. We provided a comprehensive theoretical analysis that study the choice of the hyperprior to both the sparsity and the local optimality of the resulting solutions. In particular, we demonstrated that certain hyperpriors, such as the half-Laplace prior and half-generalized Gaussian prior with shape parameter $0<p<1$, can not only promote sparsity in the estimated hyperparameters, but also stabilize the solution with respect to measurement noise. To address the resulting nonconvex optimization problem, we proposed a PALM algorithm and established its convergence under both convex and concave hyperpriors. Extensive numerical experiments on 2D image deblurring tasks were conducted to validate our theoretical findings. The results confirm that the choice of hyperprior plays a decisive role in balancing sparsity, stability, and restoration accuracy, especially under increasing levels of ill-posedness and noise.

Several directions remain open for future work. First, it is important to extend the convergence theory of PALM beyond the convex or concave setting, to accommodate more general nonconvex hyperpriors. Second, we plan to investigate the sparsity-promoting properties of broader classes of group or heavy-tailed hyperpriors, such as the Student’s t prior. Finally, to further improve image restoration quality, we will explore alternative sparsifying transforms beyond DCT, and develop fast solvers tailored to these bases within EBF.
\appendix

\bibliographystyle{siamplain}
\bibliography{references}
\end{document}